\documentclass[nohyperref]{article}

\usepackage[utf8]{inputenc} 
\usepackage[T1]{fontenc}    
\usepackage{hyperref}       
\usepackage{url}            
\usepackage{booktabs}       
\usepackage{amsfonts}       
\usepackage{nicefrac}       
\usepackage{microtype}      
\usepackage{xcolor}         

\usepackage{graphicx}
\usepackage{amsmath,amsthm,amsfonts,amssymb}
\usepackage[utf8]{inputenc} 
\usepackage{booktabs}       
\usepackage{nicefrac}       
\usepackage{microtype}      
\usepackage{soul}           
\usepackage{tabularx}
\usepackage{xspace}
\usepackage{multirow}
\usepackage{bm}
\usepackage{bbm}
\usepackage{mwe}
\usepackage{cleveref}
\usepackage{array}
\usepackage{rotating}       
\usepackage{verbatim}
\usepackage{adjustbox}
\usepackage{subcaption}    
\newtheorem{theorem}{Theorem}
\newtheorem{proposition}{Proposition}
\newtheorem{lemma}{Lemma}
\newtheorem{definition}{Definition}

\newcommand{\inputs}{\mathcal{X}} 
\newcommand{\outputs}{\overline{\mathcal{Y}}} 
\newcommand{\xoutputs}{\mathcal{Y}} 
\newcommand{\rejsym}{\bot} 

\newcommand{\xb}{\mathbf{x}} 
\newcommand{\simplex}{\Delta} 
\newcommand{\ind}{\mathbf{1}} 
\newcommand{\nei}{\mathcal{N}} 

\newcommand{\rerej}{R^\textrm{rej}} 
\newcommand{\rarej}{A^\textrm{rej}} 
\newcommand{\ore}{R} 
\newcommand{\orerej}{R^\textrm{rej}} 
\newcommand{\nerr}{R} 
\newcommand{\lrej}{\ell^\textrm{rej}} 
\newcommand{\tlrej}{L} 

\def \ie {{\em i.e.},~}

\def \eg {{\em e.g.},~}

\newcommand{\argmax}{\operatornamewithlimits{arg\!\max}}
\newcommand{\ben}{\begin{enumerate}}
\newcommand{\een}{\end{enumerate}}
\newcommand{\beq}{\begin{equation}}
\newcommand{\eeq}{\end{equation}}
\newcommand{\beqa}{\begin{eqnarray}}
\newcommand{\eeqa}{\end{eqnarray}}
\newcommand{\bit}{\begin{itemize}}
\newcommand{\eit}{\end{itemize}}
\newcommand{\btab}{\begin{tabular}}
\newcommand{\etab}{\end{tabular}}

\newcommand{\semic}{{\,;\,}}

\newcommand{\noprint}[1]{}

\newcommand{\mypara}[1]{\noindent\textbf{#1}}

\def \bfzero {\mathbf{0}}
\def \expec {\mathop{\mathbb{E}}}

\def \indicator {\mathbf{1}}

\def \mymax {\displaystyle\max\limits}

\def \bftheta {\bm{\theta}}

\def \bfdelta {\bm{\delta}}

\def \bfg {\mathbf{g}}
\def \bfh {\mathbf{h}}

\def \bfu {\mathbf{u}}

\def \bfx {\mathbf{x}}

\def \bfz {\mathbf{z}}

\def \bfJ {\mathbf{J}}

\def \calC {\mathcal{C}}
\def \calD {\mathcal{D}}

\def \calR {\mathcal{R}}

\def \calW {\mathcal{W}}

\usepackage{algorithm}
\usepackage{algorithmic}

\usepackage{newfloat}
\usepackage{listings}
\usepackage{wrapfig}
\usepackage[inline, shortlabels]{enumitem}
\setlist[enumerate]{nosep}

\usepackage[accepted]{icml2023}

\usepackage{amsmath}
\usepackage{amssymb}
\usepackage{mathtools}
\usepackage{amsthm}

\usepackage[textsize=tiny]{todonotes}

\icmltitlerunning{Stratified Adversarial Robustness with Rejection}

\begin{document}

\twocolumn[
\icmltitle{Stratified Adversarial Robustness with Rejection}



\icmlsetsymbol{equal}{*}

\begin{icmlauthorlist}
\icmlauthor{Jiefeng Chen}{equal,wisc}
\icmlauthor{Jayaram Raghuram}{equal,wisc}
\icmlauthor{Jihye Choi}{wisc}
\icmlauthor{Xi Wu}{google}
\icmlauthor{Yingyu Liang}{wisc}
\icmlauthor{Somesh Jha}{wisc}
\end{icmlauthorlist}

\icmlaffiliation{wisc}{Department of Computer Sciences, University of Wisconsin at Madison}
\icmlaffiliation{google}{Google}

\icmlcorrespondingauthor{Jiefeng Chen}{jiefeng@cs.wisc.edu}

\icmlkeywords{Machine Learning, ICML}

\vskip 0.3in
]



\printAffiliationsAndNotice{\icmlEqualContribution} 

\begin{abstract}
Recently, there is an emerging interest in adversarially training a classifier with a rejection option (also known as a selective classifier) for boosting adversarial robustness.
While rejection can incur a cost in many applications, existing studies typically associate zero cost with rejecting perturbed inputs, which can result in the rejection of numerous slightly-perturbed inputs that could be correctly classified. 
In this work, we study adversarially-robust classification with rejection in the stratified rejection setting, where the rejection cost is modeled by rejection loss functions monotonically non-increasing in the perturbation magnitude.   
We theoretically analyze the stratified rejection setting and propose a novel defense method -- \textit{Adversarial Training with Consistent Prediction-based Rejection} (CPR) -- for building a robust selective classifier. 
Experiments on image datasets demonstrate that the proposed method significantly outperforms existing methods under strong adaptive attacks.
For instance, on CIFAR-10, CPR reduces the total robust loss (for different rejection losses) by at least 7.3\% under both seen and unseen attacks.

\end{abstract}

\section{Introduction}
\label{sec:intro}


Building robust models against adversarial attacks is critical for designing secure and reliable machine learning systems~\cite{BiggioCMNSLGR13,SzegedyZSBEGF13,biggio2018wild,madry2018towards, zhang2019theoretically}.
However, the robust error of existing methods on complex datasets is still not satisfactory (\eg~\cite{croce2020robustbench}). Also, the robust models usually have poor generalization to threat models that are not utilized during training~\cite{stutz2020ccat,Laidlaw0F21}. Given these limitations, it is important to design selective classifiers that know when to reject or abstain from predicting on adversarial examples. This can be especially crucial when it comes to real-world, safety-critical systems such as self-driving cars, where abstaining from prediction is often a much safer alternative than making an incorrect decision. Along this line of adversarial robustness with rejection, several recent studies~\cite{laidlaw2019playing, stutz2020ccat, sheikholeslami2021provably, pang2022two, tramer2021detecting, kato2020atro} have extended the standard definition of robust error to the setting where the classifier can also reject inputs, and they consider rejecting any perturbed input to be a valid decision that does not count towards the robust error.


A key limitation with these studies is that they associate \emph{zero cost} with the rejection decision on perturbed inputs~\footnote{Note that the rejection of clean inputs does incur a cost.}, whereas rejection can often have a high cost in many practical applications. For example, consider a selective classifier for traffic signs in a self-driving system. If it rejects an input (\eg a perturbed ``speed limit 60'' sign), then the system may not know how to react and thus need human intervention (\eg adjust the speed). In such cases, rejection has the cost of service-denial and manual intervention~\cite{John16,cunningham2015autonomous,mozannar2020consistent}.
In contrast to the practical consideration, existing studies on adversarial robustness with rejection typically do not explicitly consider a cost for rejecting perturbed inputs.
The learned models thus may not satisfy the need of these applications. Indeed, the models from existing methods may end up rejecting too many slightly-perturbed inputs that could be correctly classified. 
As a concrete example, on MNIST with $\ell_\infty$-perturbation magnitude $0.4$, the method CCAT~\cite{stutz2020ccat} achieves very good performance on the 
existing metrics such as 1.82\% rejection rate on clean test inputs and 75.50\% robust accuracy with detection (a metric introduced in~\cite{tramer2021detecting}; see Eq. \ref{eq:robust_error_rejection_tramer}). 
However, for 99.30\% of the test points, CCAT will reject some small perturbations within magnitude as small as $0.02$. More results can be found in our experiments in Section~\ref{sec:experiment}. In summary, while rejecting such small perturbations has a cost, existing studies have not adequately measured the quality of the selective classifiers, and the training methods may not learn desired models for such applications.

To address this limitation, we revisit adversarially-robust classification with rejection by introducing rejection loss functions to model the potential cost of rejection. This offers a more flexible framework than the traditional adversarial robustness studies that do not consider rejection (roughly equivalent to associating with rejection the same loss as mis-classification), and recent studies on adversarial robustness with rejection (e.g. \cite{tramer2021detecting}) that associate zero loss with rejecting any perturbed input.
We focus on \emph{the stratified rejection setting where the rejection loss functions are monotonically non-increasing in the perturbation magnitude}. 
This is motivated by the consideration that small input perturbations should not be rejected when correct classification is possible, and thus their rejection should incur a large cost. 
However, large input perturbations can often be harder to classify, and rejection may be the best option when correct classification is not possible; so they should incur a lower rejection cost compared to smaller perturbations.
Furthermore, we consider the challenging scenario where  \emph{the rejection loss function used for testing is unknown at training time}. That is, the learning method is not designed for one specific rejection loss; the learned selective classifier should work under a range of reasonable rejection loss functions. Our goal is then to design such a method that can learn a robust selective classifier with small loss due to \textit{both} mis-classification and rejection. 

In summary, we make the following contributions:
\begin{itemize}[leftmargin=*, topsep=1pt, noitemsep]
\item We propose to introduce rejection loss functions to model the potential cost of rejection in applications, and study the stratified rejection setting with monotonically non-increasing rejection loss functions (Section \ref{sec:problem_setup}).
\item We provide a theoretical analysis of the stratified rejection setting. We analyze the existence of a robust selective classifier and discuss conditions when it can improve over classifiers without rejection (Section \ref{sec:theory}).
\item We propose a novel defense method CPR inspired by our theoretical analysis. Our experiments demonstrate that CPR significantly outperforms previous methods under strong adaptive attacks (Sections \ref{sec:method} and \ref{sec:experiment}). Its performance is strong for different rejection losses, on both traditional and our new metrics, and under seen and unseen attacks. CPR can be combined with different kinds of adversarial training methods (e.g., TRADES~\cite{zhang2019theoretically}) to enhance their robustness. 
\end{itemize}

\subsection{Related Work} 
\label{sec:related}

Adversarial robustness of deep learning models has received significant attention in recent years. Many defenses have been proposed and most of them have been broken by strong adaptive attacks~\cite{athalye2018obfuscated, tramer2020adaptive}. The most effective approach for improving adversarial robustness is adversarial training~\cite{madry2018towards, zhang2019theoretically}. However, adversarial training still does not achieve very high robust accuracy on complex datasets. For example, as reported in RobustBench~\cite{croce2020robustbench}, even state-of-the-art adversarially trained models struggle to exceed 67\% robust test accuracy on CIFAR-10.

One approach to break this adversarial robustness bottleneck is to allow the classifier to reject inputs, instead of trying to correctly classify all of them. Standard (non-adversarial) classification with a reject option (or selective classification) has been extensively studied in the literature~\cite{tax2008growing,geifman2019selectivenet,charoenphakdee2021classification,cortes2016learning}. 
Selective classification in the transductive setting with provable guarantees has been studied by \citet{goldwasser2020beyond}.
Recently, there has been a great interest in exploring adversarially robust classification with a reject option~\cite{laidlaw2019playing, stutz2020ccat, kato2020atro, yin2020generative, sheikholeslami2021provably, tramer2021detecting, pang2022two, balcan2023analysis}.
We next discuss some of these closely related works~\nocite{hosseini2017blocking, balcan2022robustly, chen2023aspest}.

\citet{stutz2020ccat} proposed to adversarially train confidence-calibrated models using label smoothing and confidence thresholding so that they can generalize to unseen adversarial attacks. 
\citet{sheikholeslami2021provably} modified existing certified-defense mechanisms to allow the classifier to either robustly classify or detect adversarial attacks, and showed that it can lead to better certified robustness, especially for large perturbation sizes.
\citet{pang2022two} observed that two coupled metrics, the prediction confidence and the true confidence (T-Con), can be combined to provably distinguish correctly-classified inputs from mis-classified inputs. Based on this, they propose to learn a rectified confidence (R-Con) that models T-Con, which is then used to adversarially train a selective classifier.
\citet{laidlaw2019playing} proposed a method called Combined Abstention Robustness Learning (CARL) for jointly learning a classifier and the region of the input space on which it should abstain, and showed that training with CARL can result in a more accurate and robust classifier.
In \cite{balcan2023analysis}, the authors introduced a random feature subspace threat model and showed that classifiers without the ability to abstain (reject) are provably vulnerable to this adversary; but allowing the classifier to abstain (e.g., via a thresholded nearest-neighbor algorithm) can overcome such attacks.


An important aspect that has not been explored in prior works is the cost or loss of rejecting perturbed inputs. This is important for designing robust classifiers that do not reject many slightly-perturbed inputs which could be correctly classified.
To the best of our knowledge, we are the first to study adversarially-robust classification with rejection in the stratified-rejection setting.

\section{Stratified Adversarial Robustness with Rejection}
\label{sec:problem_setup}

\mypara{Notations.}
Let $\inputs \subseteq \mathbb{R}^d$ denote the space of inputs $\xb$ and $\outputs := \{1, \cdots, k\}$ denote the space of outputs $y$. Let $\xoutputs := \outputs \cup \{\rejsym\}$ be the extended output space where $\rejsym$ denotes the abstain or rejection option. 
Let $\simplex_k$ denote the set of output probabilities over $\outputs$ (\ie the simplex in $k$-dimensions). 
Let $d(\xb, \xb')$ be a norm-induced distance on $\inputs$ (\eg the $\ell_p$-distance for some $p \geq 1$), and let $\,\nei(\xb, r) := \{\xb' \in \inputs: d(\xb', \xb) \le r\}\,$ denote the neighborhood of $\xb$ with distance $r$. 
Let $\wedge$ and $\vee$ denote the min and max operations respectively (reducing to AND and OR operations when applied on boolean values).
Let $\indicator\{c\}$ define the binary indicator function which takes value $1$ ($0$) when the condition $c$ is true (false). 

\mypara{Review.} We first review the standard setting of adversarial robustness without rejection and the setting with rejection in a recent line of work.
Given samples from a distribution $\calD$ over $\inputs \times  \outputs$, the goal is to learn a classifier with a reject option (a selective classifier), $f: \inputs \rightarrow \xoutputs$, that has a small error.   
The {\em standard robust error} at adversarial budget $\epsilon > 0$ is defined as~\cite{madry2018towards}:
\begin{align} \label{eq:robust_error}
    \ore_\epsilon(f) := \expec_{(\xb,y) \sim \calD}\left[\max_{\xb' \in \nei(\xb, \epsilon)} \ind\{ f(\xb') \neq y\}\right],
\end{align}
which does not allow rejection (\ie rejection is an error). 
A few recent studies (e.g.~\cite{tramer2021detecting}) have proposed the {\em robust error with detection} at adversarial budget $\epsilon$ as
\begin{align}
\label{eq:robust_error_rejection_tramer}
  \orerej_\epsilon(f) ~:=&~ \expec_{(\xb,y) \sim \calD} \bigg[
    \ind\{ f(\xb) \neq y \} \\ \nonumber
    ~\vee & \max_{\xb^\prime \in \nei(\xb, \epsilon)} \ind\big\{ f(\xb^\prime) \not \in \{y, \rejsym\} \big\} \bigg],
\end{align}
which allows the rejection of small (even infinitesimal) perturbations without incurring any error. 

\mypara{Rejection Loss Functions.}
The above studies are not well-suited for the needs of certain applications where rejection can have a cost. We would like to associate with the reject decision a loss that is a function of the perturbation magnitude.
Intuitively, rejecting a clean input 
should incur the maximum loss, and the loss of rejecting a perturbed input should decrease (or at least not increase) as the perturbation magnitude increases. 
Formally, let $\,\lrej(r): [0, \infty) \rightarrow [0,1]$ be a function specifying the loss of rejecting a perturbation $\bfx'$ of a clean input $\bfx$ with perturbation magnitude $r = d(\xb, \xb')$. We consider two concrete cases of such losses.
The {\em step rejection loss} is defined as
\begin{align}
\label{eq:step_rejection_loss}
\lrej(r) ~=~ \ind\{r \,\leq\, \alpha_0\epsilon\} 
\end{align}
for some $\alpha_0 \in [0, 1]$. That is, rejecting a perturbed input of magnitude smaller than $\alpha_0\epsilon$ has a loss 1 but rejecting larger perturbations has no loss.
The {\em ramp rejection loss} is defined as follows for some $t \geq 0$
\begin{align}
\label{eq:ramp_rejection_loss}
\lrej(r) ~=~ \left(1 \,-\, \frac{r}{\epsilon}\right)^t. 
\end{align}
For instance, for $t = 1$, the ramp rejection loss decreases linearly with the perturbation magnitude. 

\mypara{Total Robust Loss.}
With the rejection loss modeling the potential cost of rejection, the adversary can make the selective classifier suffer a mis-classification loss when there exists a perturbation $\xb'$ with $f(\xb^\prime) \not \in \{y, \rejsym\}$, or suffer a rejection loss $\lrej(d(\xb, \xb'))$ when there exists a perturbation $\xb'$ that is rejected (\ie $f(\xb') = \rejsym$). Then our goal is to learn a selective classifier that has a small total loss due to both mis-classification and rejection induced by the adversary, which is formalized as follows. 
\begin{definition} \label{def:total-loss}
{\em The total robust loss} of a selective classifier $f$ at adversarial budget $\epsilon>0$ with respect to a given rejection loss function $\lrej$ is:
\begin{align*}
    \tlrej_\epsilon(f; \lrej) ~:=~ &\expec_{(\xb,y) \sim \calD} \bigg[ \max_{\xb^\prime \in \nei(\xb, \epsilon)} \Big( \ind\big\{ f(\xb^\prime) \not \in \{y, \rejsym\} \big\} \\ \nonumber
   ~&\vee~ \ind\{f(\xb') = \rejsym \} \,\lrej(d(\xb, \xb')) \Big) \bigg].
\end{align*}
\end{definition} 
Here $\vee$ denotes the maximum of the mis-classification loss and the rejection loss. 
While this definition is compatible with any rejection loss, we will focus on the monotonically non-increasing ones. 
This definition also applies to classifiers without a rejection option, in which case the total robust loss reduces to the standard robust error $\ore_\epsilon(f)$.

\mypara{The Curve of Robust Error.}
The definition of the total robust loss however depends on the specific instantiation of $\lrej$, which may vary for different applications. In practice, we would like to have a single evaluation of $f$ which can be combined with different definitions of $\lrej$ to compute the total robust loss.
Towards this end, we propose the notion of the \textit{curve of robust error (or accuracy)}, which generalizes the existing metrics. More importantly, Lemma~\ref{corr_total_robust_loss} shows that the curve can be used to compute the total robust loss for different rejection losses (proof given in Appendix~\ref{sec:proof_lemma1}). 

\begin{definition} \label{def:robust-curve}
{\em The curve of robust error} of a selective classifier $f$ at adversarial budget $\epsilon \geq 0$ is $\{\rerej_{ \epsilon}(f, \alpha): \alpha \in [0,1]\}$, where 
\begin{align}
\label{eq:robust_error_rejection}
 \rerej_{ \epsilon}(f, \alpha) ~:=&~ \expec_{(\xb,y) \sim \calD} \bigg[ \max_{\xb' \in \nei(\xb, \,\alpha\epsilon)} \ind\{ f(\xb') \neq y\} \\ \nonumber
 \vee & \max_{\xb'' \in \nei(\xb, \epsilon) } \ind\big\{ f(\xb'') \not \in \{y, \rejsym\} \big\} \bigg].
\end{align} 
{\em The curve of robust accuracy} or simply {\em the robustness curve} of $f$ at adversarial budget $\epsilon$ is defined as
\begin{align}
\label{eq:robust_accuracy_rejection}
    \{\rarej_{\epsilon}(f, \alpha) \,:=\, 1 - \rerej_{\epsilon}(f, \alpha) : \alpha \in [0,1]\}.
\end{align}
\end{definition}
We name $\rerej_{ \epsilon}(f, \alpha)$ the {\em robust error with rejection} at $\alpha$, and  $\rarej_{\epsilon}(f, \alpha)$ the {\em robustness with rejection} at $\alpha$. The intuition behind  $\rerej_{\epsilon}(f, \alpha)$ for a fixed $\alpha$ is clear. For small perturbations within $\nei(\xb, \,\alpha\epsilon)$, both an incorrect prediction and rejection are considered an error.
For larger perturbations outside $\nei(\xb, \,\alpha\epsilon)$, rejection is \textit{not} considered to be an error (\ie the classifier can either classify correctly or reject larger perturbations).
Moreover, $\rerej_{\epsilon}(f, \alpha)$ actually includes several existing metrics as special cases:
\vspace{-2mm}
\begin{itemize}
    \item When $\alpha=1$, the metric reduces to the standard robust error at budget $\epsilon$ in Eq. (\ref{eq:robust_error}), \ie $\rerej_{\epsilon}(f, 1) = \ore_{ \epsilon}(f)$.
    \item When $\alpha = 0$, it reduces to the robust error with detection at budget $\epsilon$ in Eq. (\ref{eq:robust_error_rejection_tramer}), \ie $\rerej_{\epsilon}(f, 0) = \orerej_\epsilon(f)$. Here rejection incurs an error \textit{only} for clean inputs.
    \item For any classifier $f$ \emph{without rejection} and any $\alpha \in [0, 1]$, it reduces to the standard robust error at budget $\epsilon$ defined in Eq. (\ref{eq:robust_error}), \ie $\,\rerej_{\epsilon}(f, \alpha) = \ore_{\epsilon}(f)$. 
\end{itemize}
 
\begin{lemma}
\label{corr_total_robust_loss}
Let $s(\alpha) :=  \rerej_{\epsilon}(f, \alpha)$. Suppose the rejection loss $\lrej : [0, \infty) \mapsto [0, 1]$ is monotonically non-increasing, differentiable almost everywhere, and $\lrej(0) = 1$ and $\lrej(\epsilon) = 0$. Then the total robust loss simplifies to 
$ 
    \tlrej_\epsilon(f; \lrej) 
    ~=~ -\int_0^1 s(\alpha) ~\mathrm{d} \lrej(\alpha \epsilon).
$ 
\end{lemma}

Given this nice connection between the total robust loss and the curve of robust error, we advocate for evaluating a selective classifier $f$ using the curve $\{\rerej_{\epsilon}(f, \alpha) : \alpha \in [0,1]\}$. 
Without knowing the concrete definition of the rejection loss $\lrej$, this curve can give a holographic evaluation of the model. 
Even when $f$ is trained with a specific rejection loss $\lrej$ in mind, the curve is helpful in providing a more complete evaluation w.r.t.\ {\em any other} definition of the rejection loss at deployment time. 

\section{Theoretical Analysis}
\label{sec:theory}

Our goal is to learn a robust selective classifier, even when the training does not know the precise specification of the rejection loss for testing. 
Some fundamental questions arise:
\begin{enumerate}
    \item[\textbf{Q1}.] \emph{Whether and under what conditions does there exist a selective classifier with a small total robust loss?} 
    Many applications could have a cost for certain rejections, e.g., rejecting very small perturbations is undesirable. To design algorithms for such applications, we would like to first investigate the existence of a solution with a small total robust loss.
    \item[\textbf{Q2}.] \emph{Can allowing rejection lead to a smaller total robust loss?} 
    The question is essentially about the benefit of selective classification over traditional adversarial robustness (without rejection) that tries to correctly classify all perturbations, typically by adversarial training or its variants. 
\end{enumerate}

The following theorem helps address these questions: by allowing rejection, there exists a selective classifier with a small total robust loss under proper conditions.
\begin{theorem}
\label{thm:main}
Consider binary classification. Let $f^*(\xb)$ be any classifier without a rejection option. For any $\delta \in [0,1]$ and $\epsilon \geq 0$, there is a selective classifier $f_\delta$ whose robust error curve is bounded by: 
\begin{align}
    \rerej_{ \epsilon}(f_\delta, \alpha) \,\le\, \ore_{\epsilon'}(f^*), ~\forall \alpha \in [0,1]
\end{align}
where $\epsilon' = \max\{(\alpha+\delta)\epsilon, (1-\delta)\epsilon\}$.
Moreover, the bound is tight: for any $\alpha \in [0,1]$, there exist simple data distributions and $f^*$ such that any $f$ must have $\rerej_{ \epsilon}(f, \alpha) \ge  \ore_{\epsilon'}(f^*)$.
\end{theorem}

\mypara{Proof Sketch:}
For any $r \ge 0$, let $\nei(f^*, r)$ denote the region within distance $r$ to the decision boundary of $f^*$:
$
    \nei(f^*, r) := \{\xb \in \inputs: \exists \xb' \in \nei(\bfx, r), f^*(\xb') \neq f^*(\xb)\}.
$
Consider a parameter $\delta \in [0, 1]$ and construct $f_\delta$ as follows: 
\begin{align} \label{eq:selective-classifier}
    f_\delta(\xb) ~:=~ 
    \begin{cases}
    \rejsym & \textrm{if~} \xb \in \nei(f^*, \delta\epsilon), 
    \\
    f^*(\xb) & \textrm{otherwise.}  
    \end{cases}
\end{align}
This is illustrated in Fig.~\ref{fig:satr_illustration}.
We will show that any clean data $(\xb, y)$ contributing error in $\rerej_{\epsilon}(f_\delta, \alpha)$ must contribute error in $\ore_{\epsilon'}(f^*)$, so $\rerej_{ \epsilon}(f_\delta, \alpha) \le \ore_{\epsilon'}(f^*)$. 
Intuitively, $f_\delta$ and $f^*$ differ on the region $\nei(f^*, \delta\epsilon)$.
If $f_\delta$ gets a loss on a (possibly perturbed) input $\xb' \in \nei(f^*, \delta\epsilon)$, then the original input $\xb$ is close to $\xb'$ and thus close to the boundary of $f^*$. Therefore, $\xb$ can be perturbed to cross the boundary and contribute error in $\ore_{\epsilon'}(f^*)$. 
The complete proof is given in Appendix~\ref{sec:proof_thm1}. 
$\hfill \qedsymbol$

%
\begin{figure}[!t]
\centering
\includegraphics[width=0.45\textwidth]{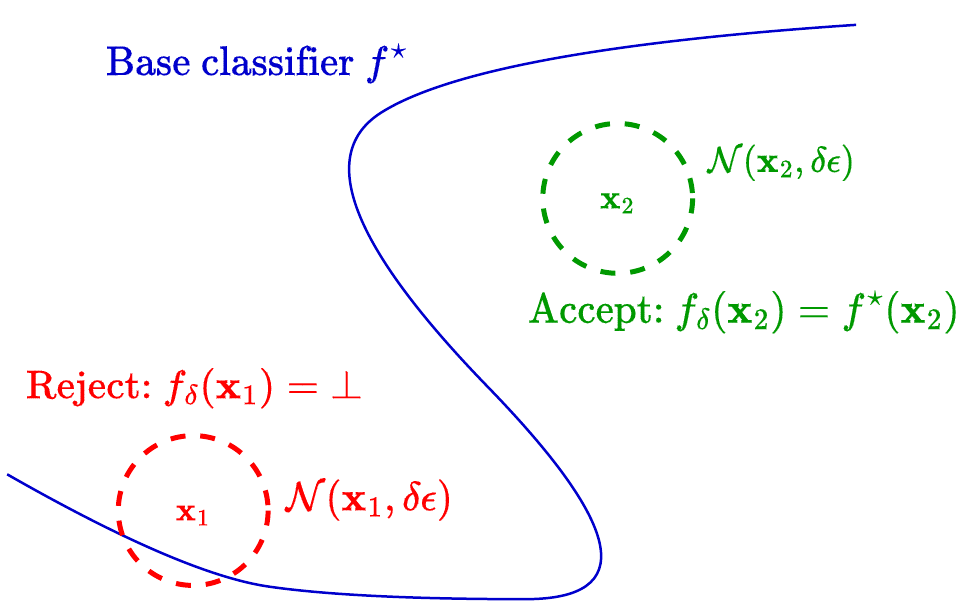}
\caption{
\small Proposed construction of a selective classifier $f_\delta$ from the base classifier $f^*$ for Theorem~\ref{thm:main}.
Input $\bfx_1$ is close to the boundary of $f^*$ and is rejected, while input $\bfx_2$ is accepted.
}
\vspace{-3mm}
\label{fig:satr_illustration}
\end{figure}

\mypara{Condition for Successful Selective Classification.}
The theorem shows that when the data allows a small robust error at budget $\epsilon' = \max\{(\alpha+\delta)\epsilon, (1-\delta)\epsilon\}$, there is a selective classifier $f_\delta$ with a small robust error with rejection $\rerej_{\epsilon}(f_\delta, \alpha)$, which is bounded by the small robust error $\ore_{\epsilon'}(f^*)$.
This shows a trade-off between the performance for small and large $\alpha$. For $\alpha < 1-\delta$, we have $\epsilon' < \epsilon$ and $\rerej_{\epsilon}(f_\delta, \alpha) \le \ore_{\epsilon'}(f^*) \le \ore_{\epsilon}(f^*)$. Note that $\ore_{\epsilon}(f^*) = \rerej_{\epsilon}(f^*, \alpha)$, so $\rerej_{\epsilon}(f_\delta, \alpha) \le \rerej_{\epsilon}(f^*, \alpha)$, i.e., $f_\delta$ can be better than $f^*$ for such an $\alpha$.
However, for $\alpha \ge 1-\delta$, the theorem does not guarantee $\rerej_{\epsilon}(f_\delta, \alpha) \le \ore_{\epsilon}(f^*)$, \ie $f_\delta$ may be worse than $f^*$ for such an $\alpha$.
This trade-off is necessary in the worst case since the bound is tight, and is also observed in our experiments in Section~\ref{sec:experiment}.
It is favorable towards a lower total robust loss when the rejection loss is monotonically non-increasing, as discussed in detail below.

\mypara{Benefit of Selective Classification.}
The trade-off discussed above shows that by allowing rejection, the selective classifier $f_\delta$ can potentially have a smaller total robust loss than the classifier $f^*$ without rejection. 
When $\alpha < 1 - \delta$, 
$
\rerej_{\epsilon}(f_\delta, \alpha) ~\le~  \ore_{\epsilon'}(f^*) ~\le~ \ore_{\epsilon}(f^*) ~=~ \rerej_{ \epsilon}(f^*, \alpha).
$
There can be a big improvement: when a large fraction of correctly-classified clean inputs have distances in $(\epsilon', \epsilon)$ to the boundary of $f^*$, we have $\ore_{\epsilon'}(f^*) \ll \ore_{\epsilon}(f^*)$ and thus $\rerej_{\epsilon}(f_\delta, \alpha) \ll \rerej_{ \epsilon}(f^*, \alpha)$. This can lead to a significant improvement in the total robust loss. When $\alpha \ge 1 - \delta$, $\rerej_{\epsilon}(f, \alpha)$ may be worse than $\rerej_{ \epsilon}(f^*, \alpha)$, but that only leads to milder rejection loss when $\lrej$ is monotonically decreasing. 
Then overall, the total robust loss can be improved, while the decreased amount would depend on the concrete data distribution and rejection loss. 

Theorem~\ref{thm:closed-form-solution} in Appendix \ref{sec:closed-form-solution} provides a rigorous and fine-grained analysis of the robustness curve and the total robust loss of $f_\delta$.
In general, the total robust losses of $f_\delta$ and $f^*$ only differ on  points with distances larger than $(1-\delta)\epsilon$ from the boundary of $f^*$. For such points, $f_\delta$ can correctly classify all their small perturbations of magnitude at most $\epsilon_0 := (1-2\delta)\epsilon$, and reject or correctly classify their larger perturbations. Intuitively, it correctly classifies small perturbations and rejects or correctly classifies large perturbations. 
Then $f_\delta$ only gets rejection loss for large magnitudes, and thus potentially gets smaller total robust losses than $f^*$ for monotonically non-increasing rejection losses. 

For a concrete example, suppose $f^*$ has $0$ standard error on the data distribution and each data point has distance at least $\frac{3\epsilon}{4}$ to the decision boundary of $f^*$.
Then when $\delta=\frac{1}{4}$, for any data point, $f_\delta$ can correctly classify all the small perturbations with magnitude bounded by $\frac{\epsilon}{2}$, and can reject or correctly classify the larger perturbations. Then $\rerej_{\epsilon}(f_\delta, \alpha)=0$ for $\alpha\leq \frac{1}{2}$.
For any step rejection loss with parameter $\alpha_0 \leq \frac{1}{2}$, the total robust loss of $f_\delta$ is 0, \ie this fixed $f_\delta$ can work well for a wide range of step rejection losses. In contrast, the total robust loss of $f^*$ is as large as $\ore_{\epsilon}(f^*)$, which is the probability mass of points within $\epsilon$ distance to the boundary of $f^*$. If there is a large mass of points with distance in $[\frac{3\epsilon}{4}, \epsilon]$, then $f_\delta$ has a significant advantage over $f^*$.


\section{Proposed Defense}
\label{sec:method}
 
We propose a defense method called adversarial training with consistent prediction-based rejection (CPR), following our theoretical analysis. 
CPR aims to learn the $f_\delta$ in our  analysis, where PGD is used to check the  condition for rejection efficiently at test time.
Our defense is quite general in that it can take any base classifier $f^*$ to construct the selective classifier $f_\delta$.
Finally, to evaluate the robustness of the defense, we also discuss the design of adaptive attacks.

\subsection{Consistent Prediction-Based Rejection}

The CPR defense is essentially the selective classifier $f_\delta$ in Eq.~(\ref{eq:selective-classifier}) in Theorem~\ref{thm:main}: given a base classifier and an input $\xb$, we define a selective classifier that rejects the input whenever the predictions in a small neighborhood of $\xb$ are not consistent;  otherwise it returns the class prediction of $\xb$. 
Equivalently, it rejects inputs within a small distance to the decision boundary of the given base classifier.

To formalize the details, we introduce some notations. 
Consider a base classifier without rejection (i.e., $f^*$) realized by a network with parameters $\bftheta$, whose predicted class probabilities are $\,\mathbf{h}(\xb \semic \bftheta) \!=\! [h_1(\xb \semic \bftheta), \cdots, h_k(\xb \semic \bftheta)] \!\in\! \simplex_k$. 
The class prediction is $\widehat{y}(\xb) := \arg\!\max_{y \in \outputs} \,h_y(\xb \semic \bftheta)$, and the cross-entropy loss is $\,\ell_{CE}(\xb, y) = -\log h_y(\xb \semic \bftheta)$. CPR aims to learn a selective classifier (corresponding to $f_\delta$): 
\begin{align*}
    \bfg_{\widetilde{\epsilon}}(\bfx \semic \bftheta) = 
    \begin{cases}
    \bfh(\bfx \semic \bftheta) & \textrm{if}~~ \mymax_{\widetilde{\xb} \in \nei(\xb, \widetilde{\epsilon})} \indicator\{\widehat{y}(\widetilde{\xb}) \neq \widehat{y}(\xb) \} \,=\, 0
    \\
    \rejsym    &   \textrm{otherwise},
    \end{cases}
\end{align*}
where $\widetilde{\epsilon} > 0$ is a hyper-parameter we call the {\em consistency radius}.
To check the consistent-prediction condition efficiently, we use the projected gradient descent (PGD) method~\cite{madry2018towards} to find the worst-case perturbed input $\widetilde{\xb}$. The selective classifier is then redefined as:
\begin{align}
\label{eq:consistent_prediction_selective2}
    \bfg_{\widetilde{\epsilon}}(\bfx \semic \bftheta) & ~=~ 
    \begin{cases}
    \bfh(\bfx \semic \bftheta) & \textrm{if}~~ \widehat{y}(\widetilde{\xb}) = \widehat{y}(\xb)
    \\
    \bot    &   \textrm{otherwise},
    \end{cases}
    \\
    \text{where~} \,\widetilde{\xb} & \,=\, \argmax_{\widetilde{\xb} \in \nei(\xb, \widetilde{\epsilon})} \,\ell_{CE}(\widetilde{\xb}, \widehat{y}(\xb)).
\end{align}
Then $\widehat{y}(\xb)$ is the predicted class when the input is accepted.
%
The details of our CPR defense are given in Algorithm~\ref{alg:consistent-prediction-rej}.
Note that 
the PGD method in the algorithm is deterministic. 
Thus, the mapping $\,T(\xb) = \xb_m$ is deterministic and we can summarize the defense as: if $\,\widehat{y}(T(\xb)) \neq \widehat{y}(\xb)$, then reject $\xb$\,; otherwise accept $\xb$ and output the predicted class $\widehat{y}(\xb)$.
\begin{algorithm}[!t]
  \small
  \caption{\textsc{\small Consistent prediction-based Rejection}}
  \begin{algorithmic}[1]
    \REQUIRE A base model $\bfh$, a test input $\xb$ (potentially with adversarial perturbations), the consistency radius $\widetilde{\epsilon} > 0$, the number of PGD steps $m \geq 1$, and the PGD step size $\eta > 0$.
    \STATE Class prediction: $y = \widehat{y}(\xb) = \argmax_{y' \in \outputs} h_{y'}(\xb \semic \bftheta)$
    \STATE Initialize the adversarial input: $\xb_0 = \xb$
    \FOR{$i = 1, 2, \dots, m$ }
    \STATE $\widehat{\xb}_i \,=\, \xb_{i-1} \,+\, \eta\, \textrm{sign}\big( \nabla_{\xb_{i-1}} \ell_{CE}(\xb_{i-1}, y) \big)$
    \STATE $\xb_i \,=\, \text{Proj}\big( \widehat{\xb}_i, \nei(\xb, \widetilde{\epsilon}) \big)$
    ~~\COMMENT{Project $\widehat{\xb}_i$ to $\nei(\xb, \widetilde{\epsilon})$}
    \ENDFOR
    \STATE Define $T(\xb) := \xb_m$
    \ENSURE If $\,\argmax_{y' \in \outputs} h_{y'}(T(\xb) \semic \bftheta) \neq y$, reject $\xb\,$; otherwise, accept $\xb$ and output the predicted class $y$.
  \end{algorithmic}
  \label{alg:consistent-prediction-rej}
\end{algorithm}
%

The robust base model for our defense $\bfh$ can be trained using methods such as standard adversarial training (AT)~\cite{madry2018towards}, TRADES~\cite{zhang2019theoretically}, and MART~\cite{wang2020improving} (we will focus on the first two methods).
An advantage of the proposed defense is that it is agnostic of the rejection loss function used for evaluation, and therefore can generalize to multiple choices of the rejection loss function (see Section~\ref{sec:experiment}).
We note that prior works such as \citet{laidlaw2019playing} and \citet{tramer2021detecting} have also proposed the idea of using the consistency check for rejecting inputs to a classifier (equivalent to rejecting inputs based on their distance to the decision boundary). However, we are the first to systematically implement and evaluate such a defense, and propose strong adaptive attacks for evaluating it.

\subsection{Adaptive Attacks}
\label{sec:adaptive_attacks_main}

In this section, we design principled adaptive attacks to evaluate the proposed defense. 
By Definition~\ref{def:robust-curve} and Proposition~\ref{proposition1}, in order to compute the robustness curve, we need to design both inner and outer attacks and then combine them to get the final attack. We design multiple inner and outer attacks sketched below and ensemble them to get the strongest final attack. More details can be found in Appendix~\ref{sec:app_adaptive_attacks}

\mypara{Inner Attack.}
For a given $\alpha \in [0, 1]$ on the robustness curve, the goal of the inner attack is to find an input $\bfx^\prime \in \nei(\bfx, \alpha\epsilon)$ that is rejected. For CPR, this translates to finding $\bfx^\prime \in \nei(\bfx, \alpha\epsilon)$ such that the base model has different predictions on $\bfx^\prime$ and $T(\bfx^\prime)$.

Our first attack is Low-Confidence Inner Attack (\textbf{LCIA}), which finds $\bfx^\prime$ by minimizing the confidence of the base model within $\nei(\bfx, \alpha\epsilon)$.
Recall that the mapping $T(\bfx^\prime)$ attempts to minimize the base model's probability on the predicted class $\widehat{y}(\bfx^\prime)$.
So, if the base model has low confidence on $\bfx^\prime$, then it will very likely have even lower probability for $\widehat{y}(\bfx^\prime)$ on $T(\bfx^\prime)$, and thus have different predictions on $T(\bfx^\prime)$ and $\bfx^\prime$. 
The attack objective in this case is $\,\bfx^\prime = \argmax_{\bfz \in \nei(\bfx, \alpha\epsilon)} \,-\log h_{\max}(\bfz \semic \bftheta)$,
where $h_{\max}(\bfz \semic \bftheta) = \max_{y \in \outputs} h_y(\bfz \semic \bftheta) $ is the prediction confidence. We use the temperature-scaled log-sum-exponential approximation to the $\max$ function in order to make it differentiable.  
We also consider a variant of LCIA, named Consistent-Low-Confidence Inner Attack (\textbf{CLCIA}), that minimizes the confidence of the base model on both $\bfx'$ and $T(\bfx')$ (details in Appendix~\ref{sec:app_adaptive_attacks}).
 
The third attack is Prediction-Disagreement Inner Attack (\textbf{PDIA}), an adaptive multi-target attack based on the BPDA method~\cite{athalye2018obfuscated}. 
We use BPDA since $T(\bfx)$ does not have a closed-form expression and is not differentiable.
This attack considers all possible target classes and attempts to find $\xb^\prime$ such that the base model has high probability for the target class at $\xb^\prime$ and a low probability for the target class at $T(\xb^\prime)$ (thereby encouraging rejection).
The attack objective is: for each target class $j=1,\dots, k$, 
\begin{align*}
    \bfx_j^\prime = \argmax_{\bfz \in \nei(\bfx, \alpha\epsilon)} \Big[ \log h_j(\bfz \semic \bftheta) \,-\,  \log h_j(T(\bfz) \semic \bftheta) \Big].
\end{align*}
Then we select the strongest adversarial example $\bfx^\prime$ from the multiple target classes 
as follows:
\begin{align}
\bfx^\prime ~=~& \bfx_{j^\star}^\prime  ~~~~\text{where}~~~~ \\ \nonumber 
j^{\star} ~=~& \argmax_{j \in [k]} \, \Big[ \log h_j(\bfx_j^\prime \semic \bftheta) \,-\,  \log h_j(T(\bfx_j^\prime) \semic \bftheta) \Big].
\end{align} 

\mypara{Outer Attack.}
Given a clean input $(\bfx, y)$, the goal of the outer attack is to find $\bfx'' \in \nei(\bfx, \epsilon)$ such that the base model has a consistent incorrect prediction with high confidence on both $\bfx''$ and $T(\bfx'')$ (ensuring that $\bfx''$ is accepted and mis-classified). 
We propose the Consistent High-Confidence Misclassification Outer Attack (\textbf{CHCMOA}), an adaptive multi-target attack based on BPDA. 
The attack objective is: for each target class $j \in [k] \setminus \{y\}$,
\begin{align*}
 \bfx_j'' = \argmax_{\bfz \in \nei(\bfx, \epsilon)} \Big[ \log h_j(\bfz \semic \bftheta) \,+\, \log h_j(T(\bfz) \semic \bftheta) \Big].
\end{align*}
Then we select the strongest adversarial example $\bfx''$ via: 
\begin{align}
\bfx'' ~=~& \bfx_{j^\star}'' ~~~~\text{where}~~~~ \\ \nonumber
j^\star ~=~& \argmax_{j\in [k] \setminus \{y\}} \Big[ \log h_j(\bfx_j'' \semic \bftheta) \,+\, \log h_j(T(\bfx_j'') \semic \bftheta) \Big].
\end{align}
We also consider the High-Confidence Misclassification Outer Attack (\textbf{HCMOA}), which solves for an adversarial input $\bfx''$ such that the base model has incorrect prediction with high confidence on $\bfx''$ (details in Appendix~\ref{sec:app_adaptive_attacks}).

\mypara{Final Attack.} 
To get the strongest evaluation, our final attack is an {\em ensemble of all attacks}, including those designed above and some existing ones.
We apply each attack in the ensemble with different hyper-parameters on each clean test input. 
If any of the attacks achieves the attack goal on an input, then the attack is considered to be successful on it. 

\section{Experiments}
\label{sec:experiment}

In this section, we perform experiments to evaluate the proposed method CPR and compare it with competitive baseline methods.
The code for our work can be found at \url{https://github.com/jfc43/stratified-adv-rej}.
Our main findings are summarized as follows: {\bf 1)} CPR outperforms the baselines significantly in terms of the total robust loss with respect to different rejection losses, under both seen attacks and unseen attacks; {\bf 2)}  CPR usually has significantly higher robustness with rejection compared to the baselines for small to moderate $\alpha$ under both seen attacks and unseen attacks; 
{\bf 3)}  CPR also has strong performance under the traditional metrics such as robust accuracy with detection ($1 - \rerej_{\epsilon}(f, 0)$). 

\begin{table*}[!t]
    \centering
    \begin{adjustbox}{width=2\columnwidth,center}
		\begin{tabular}{l|l|c|c|c|c|c|c|c|c|c|c}
			\toprule
			\multirow{3}{0.08\linewidth}{Dataset} &  \multirow{3}{0.12\linewidth}{Method}  &  \multicolumn{5}{c|}{Total Robust Loss under Seen Attacks $\downarrow$} &  \multicolumn{5}{c}{Total Robust Loss under Unseen Attacks $\downarrow$} \\ \cline{3-12}
			  &  & \multicolumn{3}{c|}{Step Rej. Loss} & \multicolumn{2}{c|}{Ramp Rej. Loss} & \multicolumn{3}{c|}{Step Rej. Loss} & \multicolumn{2}{c}{Ramp Rej. Loss} \\ \cline{3-12}
			  &  & $\alpha_0=0.0$  & $\alpha_0=0.05$ & $\alpha_0=0.1$ & $t=2$ & $t=4$ & $\alpha_0=0.0$  & $\alpha_0=0.05$ & $\alpha_0=0.1$ & $t=2$ & $t=4$ \\  \hline \hline
			\multirow{4}{0.08\linewidth}{MNIST} 
            & AT+CR & 0.084 & 0.085 & 0.085 & 0.097 & 0.087 & 1.000 & 1.000 & 1.000 & 1.000 & 1.000 \\
            & TRADES+CR & 0.060 & 0.061 & 0.061 & \textbf{0.075} & 0.065 & 1.000 & 1.000 & 1.000 & 1.000 & 1.000 \\
            & CCAT & 0.168 & 0.926 & 1.000 & 0.945 & 0.894 & 0.245 & 0.994 & 1.000 & 0.955 & 0.914 \\
            & RCD & 0.135 & 0.135 & 0.135 & 0.135 & 0.135 & 1.000 & 1.000 & 1.000 & 1.000 & 1.000 \\
            & ATRR & 0.088 & 0.088 & 0.089 & 0.107 & 0.095 & 1.000 & 1.000 & 1.000 & 1.000 & 1.000 \\ \cline{2-12}
            & AT+CPR (Ours) & \textbf{0.039} & \textbf{0.039} & \textbf{0.040} & 0.133 & 0.059 & \textbf{0.096} & \textbf{0.097} & \textbf{0.098} & \textbf{0.187} & \textbf{0.116} \\
            & TRADES+CPR (Ours) & 0.042 & 0.042 & 0.042 & 0.130 & \textbf{0.058} & 0.133 & 0.133 & 0.133 & 0.208 & 0.145 \\ \hline \hline
			\multirow{4}{0.08\linewidth}{SVHN}
			& AT+CR & 0.539 & 0.539 & 0.539 & 0.542 & 0.540 & 0.882 & 0.882 & 0.882 & 0.882 & 0.882 \\
                & TRADES+CR & 0.471 & 0.472 & 0.472 & 0.475 & 0.473 & 0.874 & 0.874 & 0.874 & 0.874 & 0.874 \\
                & CCAT & 0.547 & 1.000 & 1.000 & 0.977 & 0.956 & 0.945 & 1.000 & 1.000 & 0.999 & 0.998 \\
                & RCD & 0.662 & 0.662 & 0.662 & 0.662 & 0.662 & 0.903 & 0.903 & 0.903 & 0.903 & 0.903 \\
                & ATRR & 0.552 & 0.552 & 0.552 & 0.566 & 0.556 & 0.885 & 0.885 & 0.885 & 0.891 & 0.887 \\ \cline{2-12}
                & AT+CPR (Ours) & 0.442 & 0.442 & 0.442 & 0.454 & 0.444 & 0.853 & 0.853 & 0.853 & 0.856 & 0.854 \\
                & TRADES+CPR (Ours) & \textbf{0.380} & \textbf{0.380} & \textbf{0.380} & \textbf{0.390} & \textbf{0.381} & \textbf{0.813} & \textbf{0.813} & \textbf{0.813} & \textbf{0.816} & \textbf{0.814} \\ \hline \hline
			\multirow{4}{0.08\linewidth}{CIFAR-10}
			& AT+CR & 0.500 & 0.501 & 0.501 & 0.507 & 0.503 & 0.895 & 0.895 & 0.895 & 0.895 & 0.895 \\
                & TRADES+CR & 0.500 & 0.500 & 0.500 & 0.503 & 0.501 & 0.849 & 0.849 & 0.849 & 0.849 & 0.849 \\
                & CCAT & 0.723 & 1.000 & 1.000 & 0.984 & 0.969 & 0.912 & 1.000 & 1.000 & 0.998 & 0.997 \\
                & RCD & 0.533 & 0.533 & 0.533 & 0.533 & 0.533 & 0.905 & 0.905 & 0.905 & 0.905 & 0.905 \\
                & ATRR & 0.512 & 0.513 & 0.513 & 0.518 & 0.515 & 0.887 & 0.887 & 0.887 & 0.887 & 0.887 \\ \cline{2-12}
                & AT+CPR (Ours) & 0.433 & 0.433 & 0.433 & 0.443 & 0.435 & 0.829 & 0.829 & 0.829 & 0.836 & 0.830 \\
                & TRADES+CPR (Ours) & \textbf{0.429} & \textbf{0.429} & \textbf{0.429} & \textbf{0.438} & \textbf{0.430} & \textbf{0.781} & \textbf{0.781} & \textbf{0.781} & \textbf{0.787} & \textbf{0.782} \\
			\bottomrule
		\end{tabular}
	\end{adjustbox}
	\caption[]{\small The total robust loss for different rejection loss functions under both seen and unseen attacks. The best result is \textbf{boldfaced}. }
	\label{tab:total-loss-results}
\end{table*}

\subsection{Experimental Setup}
We briefly describe the experimental setup here. The detailed setup can be found in Appendix~\ref{sec:detailed-setup}.

\mypara{Datasets. } 
We use the MNIST~\cite{lecun1998mnist}, SVHN~\cite{netzer2011reading}, and CIFAR-10~\cite{krizhevsky2009learning} datasets. Each dataset has a test set containing 10,000 images. Following~\cite{stutz2020ccat}, we compute the accuracy of the models on the first 9,000 images of the test set and compute the robustness of the models on the first 1,000 images of the test set. We use the last 1,000 images of the test set as a held-out validation set for selecting the hyper-parameters of the methods (\eg the rejection threshold).  

\mypara{Baselines. } 
We consider the following baselines: (1) AT + CR: adversarial training (AT)~\cite{madry2018towards} with Confidence-based Rejection; (2) TRADES + CR: TRADES~\cite{zhang2019theoretically} with Confidence-based Rejection; (3) CCAT~\cite{stutz2020ccat}; (4) RCD~\cite{sheikholeslami2021provably}; (5) ATRR~\cite{pang2022two}. 

\mypara{CPR Setup. } 
CPR requires a base model. We consider robust base models trained using the well-known standard adversarial training (AT)~\cite{madry2018towards} and TRADES~\cite{zhang2019theoretically}. Our experimental results show that CPR can boost the robustness of these models.

CPR has three hyper-parameters: the consistency radius $\tilde{\epsilon}$, the number of PGD steps $m$, and the PGD step size $\eta$. The $\widetilde{\epsilon}$ value controls the rejection rate of the selective classifier. We choose it such that CPR does not reject more than a small fraction of the correctly-classified clean inputs (however, it can reject a majority of the mis-classified inputs that are likely to be close to the decision boundary of the base model; rejecting such inputs is reasonable).
The rejection rate of the selective classifier $\bfg_{\widetilde{\epsilon}}$ on correctly-classified clean inputs from a distribution $\mathcal{D}$ is given by $\,\expec_{(\bfx, y) \sim \mathcal{D}}\big[ \indicator\{ \bfg_{\widetilde{\epsilon}}(\bfx \semic \bftheta) = \bot \} ~|~ \widehat{y}(\bfx) = y \big]$, 
which can be estimated using a labeled validation dataset.
We choose a large enough $\widetilde{\epsilon} > 0$ such that this rejection rate is approximately  $\,p_{rej}$, where $p_{rej}$ is a user-specified rejection rate.  The number of PGD steps $m$ also affects the robustness.
We use a validation set to select suitable $\tilde{\epsilon}$ and $m$ (details in Appendix~\ref{sec:cpr-hyper-selection}).
We do not tune the PGD step size $\eta$ and set it to a fixed value. 

On MNIST, we set $\tilde{\epsilon}=0.1$ (such that $p_{rej} = 1\%$), $\,m=20$, and $\eta=0.01$.
On SVHN and CIFAR-10, we set $\tilde{\epsilon}=0.0055$ (such that $p_{rej} =5\%$), $\,m=10$, and $\eta=0.001$.

\mypara{Evaluation Metrics. }
We use the \textit{robustness curve} at adversarial budget $\epsilon$ (Eq.~(\ref{eq:robust_accuracy_rejection})) and the total robust loss to evaluate all the methods.  
The curve is calculated for $\alpha$ values from the set $\,\{0, 0.01, 0.05, 0.1, 0.15, 0.2, 0.25, 0.3, 0.5, 1\}$ since in practice, we are mainly interested in the robustness with rejection for small values of $\alpha$.
We plot the robustness curve, with $\alpha$ along the $x$-axis and the robustness with rejection along the $y$-axis.
We also evaluate the \textit{total robust loss} $\tlrej_\epsilon(f; \lrej)$ with respect to the ramp rejection loss (Eq. (\ref{eq:ramp_rejection_loss})) for $t \in \{2, 4\}$, and the step rejection loss (Eq. (\ref{eq:step_rejection_loss})) for $\alpha_0 \in \{0, 0.05, 0.1\}$.
This gives a single metric that summarizes the robustness curve of the different methods.


\mypara{Evaluation.} 
We consider $\ell_\infty$-norm bounded attacks and generate adversarial examples to compute the robustness with rejection metric via \textit{an ensemble of adversarial attacks}. The worst-case robustness is reported under the attack ensemble. A detailed discussion of the unified approach for designing strong adaptive attacks for CPR and all the baseline methods is given in Appendix~\ref{sec:app_adaptive_attacks}. We consider both \textit{seen} attacks and \textit{unseen} attacks. Suppose $\epsilon'$ is the attack perturbation budget for testing. For seen attacks, the perturbation budget $\epsilon'=\epsilon$ (on MNIST, $\,\epsilon'=0.3$, while on SVHN and CIFAR-10, $\,\epsilon'=\frac{8}{255}$). For unseen attacks, the perturbation budget $\epsilon'>\epsilon$ (on MNIST, $\,\epsilon'=0.4$, while on SVHN and CIFAR-10, $\,\epsilon'=\frac{16}{255}$).
Finally, we use the same approach to set the rejection threshold for all the baselines. Specifically, on MNIST, we set the threshold such that only 1\% of clean correctly-classified validation inputs are rejected. On SVHN and CIFAR-10, we set the threshold such that only 5\% of clean correctly-classified validation inputs are rejected.

\begin{figure*}[!htb]
	\centering
	\includegraphics[width=0.4\linewidth]{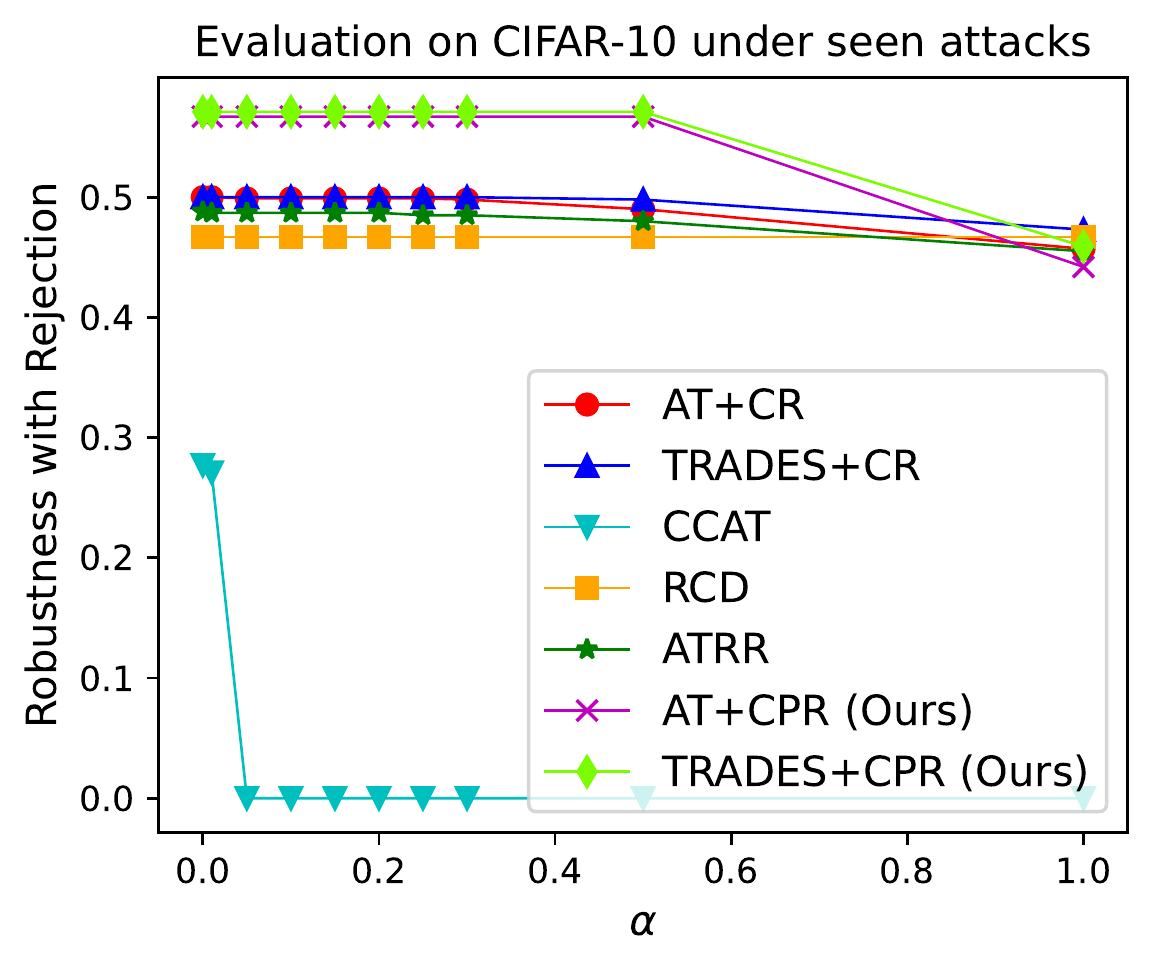}
	\includegraphics[width=0.4\linewidth]{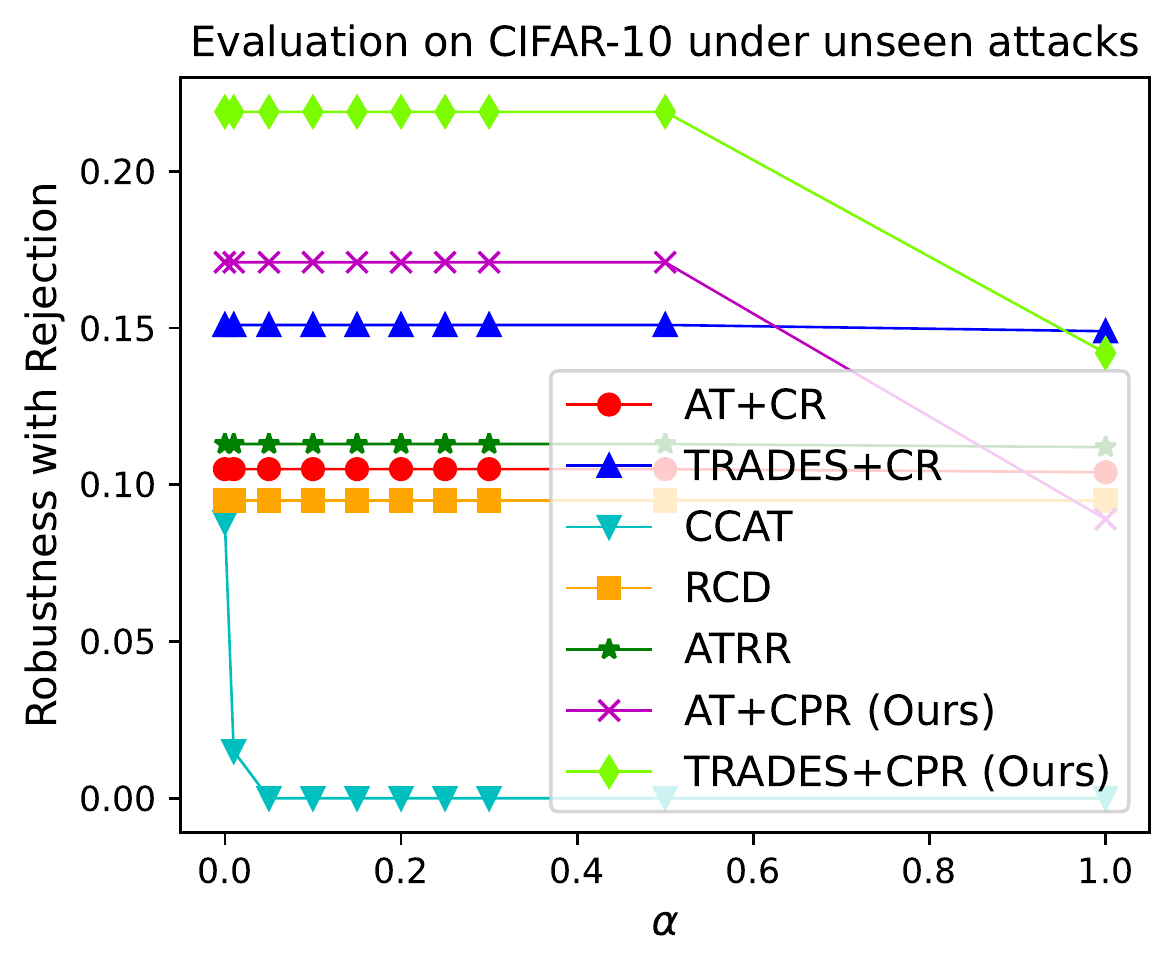}
        \vspace{-3mm}
	\caption{\small The robustness curves of all methods on CIFAR-10 for both seen and unseen attacks.  More results are in Appendix~\ref{sec:full-rob-curve-results}}
	\label{fig:main-results}
\end{figure*}

\begin{table*}[htb]
    \centering
    \begin{adjustbox}{width=1.8\columnwidth,center}
		\begin{tabular}{l|l|c|c|c|c|c|c}
			\toprule
			\multirow{3}{0.08\linewidth}{Dataset} &  \multirow{3}{0.08\linewidth}{Method} & \multicolumn{3}{c|}{Robustness with Reject at $\alpha=0$ $\downarrow$} & \multicolumn{3}{c}{Robustness with Reject at $\alpha=1$ $\downarrow$} \\  \cline{3-8}
           & & \multicolumn{3}{c|}{Outer Attack} & \multicolumn{3}{c}{Inner Attack (with Ensemble Outer Attacks)} \\ \cline{3-8}
          & & AutoAttack & HCMOA & CHCMOA & LCIA & CLCIA & PDIA \\ \hline \hline
			\multirow{2}{0.12\linewidth}{MNIST}  
			& AT+CPR & 97.60 & 97.80 & \textbf{96.10} & 50.90 & \textbf{12.30} & 65.80 \\
            & TRADES+CPR & 98.10 & 98.40 & \textbf{95.80} & 7.60 & 5.10 & \textbf{0.40} \\ \hline 
			\multirow{2}{0.12\linewidth}{{SVHN}}
			& AT+CPR & 64.20 & 57.50 & \textbf{56.70} & \textbf{42.60} & 49.40 & 43.90 \\
            & TRADES+CPR & 71.10 & 63.40 & \textbf{62.90} & \textbf{51.00} & 56.10 & 54.30 \\ \hline 
			\multirow{2}{0.12\linewidth}{CIFAR-10}
			& AT+CPR & 60.70 & 57.50 & \textbf{57.40} & \textbf{44.40} & 48.30 & 48.50 \\
            & TRADES+CPR & 61.50 & 57.40 & \textbf{57.20} & \textbf{45.90} & 51.10 & 53.10 \\ 
			\bottomrule
		\end{tabular}
	\end{adjustbox}
	\caption[]{\small Ablation study on the outer and inner adaptive attacks for CPR. Refer to Appendix~\ref{sec:app_adaptive_attacks} for more details on these attacks and the choice of metrics. All values are in percentage, and smaller values correspond to a stronger attack. \textbf{Bold} values show the strongest attacks. }
    \label{tab:cpr-seen-attack-ablation}
\end{table*}


\subsection{Results}
\label{sec:results} 

\mypara{Evaluating the Total Robust Loss. } Table~\ref{tab:total-loss-results} compares the total robust loss of different methods for different rejection loss functions under both seen attacks and unseen attacks. 
The proposed method CPR outperforms the baselines significantly in almost all cases. 
The only exception is on MNIST for the ramp rejection loss with $t=2$ under seen attacks, where CPR is worse than TRADES+CR and some baselines.
This is because small perturbations on MNIST are easy to correctly classify and the ramp rejection loss penalizes rejecting large perturbations more, while CPR tends to reject large perturbations. 
In all other cases, CPR performs significantly better than the baselines. For instance, under unseen attacks, CPR reduces the total robust loss (with respect to different rejection losses) by at least 60.8\%, 6.6\% and 7.3\% on MNIST, SVHN and CIFAR-10, respectively. Under seen attacks, CPR reduces the total robust loss by at least 18.0\% and 12.8\% on SVHN and CIFAR-10, respectively. 

\mypara{Evaluating the Robustness Curve. } Figure~\ref{fig:main-results} compares the robustness with rejection $\rarej_{\epsilon}(f, \alpha)$ of the different methods as a function of $\alpha$ under both seen and unseen attacks on CIFAR-10 (complete results on all the datasets can be found in Appendix~\ref{sec:full-rob-curve-results}). Our method CPR usually has significantly higher robustness with rejection compared to the baselines for small to moderate $\alpha$. The robustness with rejection of CPR only drops for large $\alpha$ values, which suffers less rejection loss and thus leads to smaller total robust loss (as predicted by our analysis). We also note that CCAT has worse performance than other methods since our adaptive attacks are stronger than the PGD-with-backtracking attack used in CCAT paper~\cite{stutz2020ccat} (see Appendix~\ref{sec:attacking-ccat}).


\mypara{Ablation Study on Attacks. }
We performed an ablation experiment to study the strength of each inner and outer attack in the attack ensemble. The results in Table~\ref{tab:cpr-seen-attack-ablation} show that for the outer attack, CHCMOA is consistently the best across all datasets. For the inner attack, LCIA is usually the best, and CLCIA and PDIA are strong on MNIST. 
We emphasize that we use an ensemble of all the attacks to get the strongest final evaluation.
%

\mypara{Ablation Study on Hyper-parameters. } 
We performed an ablation experiment to study the effect of the hyper-parameters $\tilde{\epsilon}$ and $m$ used by CPR.
The results in Appendix~\ref{sec:cpr-hyper-ablation} show that larger $\tilde{\epsilon}$ (consistency radius) leads to better robustness with rejection at $\alpha=0$. However, it also leads to lower robustness with rejection when $\alpha$ is large, which suggests that CPR rejects more perturbed inputs. Similarly, larger $m$ (number of PGD steps) also leads to better robustness with rejection at $\alpha=0$, but can lead to a lower robustness with rejection when $\alpha$ is large.
We set $m = 10$ in our main experiments, which is usually sufficient for good performance, and larger values lead to minimal improvements.

\mypara{Evaluating Traditional Metrics. } We also evaluate different methods on the traditional metrics, including accuracy with rejection, rejection rate on clean test inputs, an F1 score-like metric (harmonic mean of accuracy-with-rejection and $1 \,-$ rejection rate), and the robust accuracy with detection defined in \citet{tramer2021detecting}.  
The definitions of these metrics and the results are included in Appendix~\ref{sec:trad-metrics}. These results show that CPR has comparable performance to the baselines on clean test inputs, and also significantly outperforms the baselines w.r.t. robust accuracy with detection.

\section{Conclusion and Discussion}
\label{sec:discussion}

This work studied adversarially-robust classification with rejection in the practical setting where rejection carries a loss that is monotonically non-increasing with the perturbation magnitude. 
We proposed the total robust loss as a generalization of the robust error for selective classifiers where rejection carries a loss, and the robustness curve as a tool to study the total robust loss.
We provided an analysis of the setting and proposed a novel defense CPR for robustifying any given base model, which significantly outperforms previous methods under strong adaptive attacks.

\mypara{Limitations \& Future Work.}
In this work, we focused on rejection losses that are a function of perturbation magnitude. There could be other types of rejection losses, \eg $\lrej(r, \bfx)$ that also depend on the input $\bfx$. 
Our defense has an increased computational cost at prediction (test) time since it requires $m$ PGD steps to find $T(\bfx)$ in Algo.~\ref{alg:consistent-prediction-rej}, which requires $m$ forward and backward passes in addition to the two forward passes required to get the predictions of $\bfx$ and $T(\bfx)$.
Designing defense methods that have lower computational cost and better robustness than CPR, and tackling more general rejection losses are interesting future problems.


\section*{Acknowledgements}
The work is partially supported by Air Force Grant FA9550-18-1-0166, the National Science Foundation (NSF) Grants CCF-FMitF-1836978, IIS-2008559, SaTC-Frontiers-1804648, CCF-2046710 and CCF-1652140, and ARO grant number W911NF-17-1-0405. Jiefeng Chen, Jihye Choi, and Somesh Jha are partially supported by the DARPA-GARD problem under agreement number 885000.
Jayaram Raghuram is partially supported through the National Science Foundation grants CNS-2112562, CNS-2107060, CNS-2003129, CNS-1838733, CNS-1647152, and the US Department of Commerce grant 70NANB21H043.




\bibliography{paper}

\begin{thebibliography}{37}
\providecommand{\natexlab}[1]{#1}
\providecommand{\url}[1]{\texttt{#1}}
\expandafter\ifx\csname urlstyle\endcsname\relax
  \providecommand{\doi}[1]{doi: #1}\else
  \providecommand{\doi}{doi: \begingroup \urlstyle{rm}\Url}\fi

\bibitem[Athalye et~al.(2018)Athalye, Carlini, and
  Wagner]{athalye2018obfuscated}
Athalye, A., Carlini, N., and Wagner, D.~A.
\newblock Obfuscated gradients give a false sense of security: {C}ircumventing
  defenses to adversarial examples.
\newblock In \emph{Proceedings of the 35th International Conference on Machine
  Learning ({ICML})}, volume~80 of \emph{Proceedings of Machine Learning
  Research}, pp.\  274--283. {PMLR}, 2018.
\newblock URL \url{http://proceedings.mlr.press/v80/athalye18a.html}.

\bibitem[Balcan et~al.(2022)Balcan, Blum, Hanneke, and
  Sharma]{balcan2022robustly}
Balcan, M.-F., Blum, A., Hanneke, S., and Sharma, D.
\newblock Robustly-reliable learners under poisoning attacks.
\newblock In \emph{Conference on Learning Theory}, volume 178, pp.\
  4498--4534. {PMLR}, 2022.
\newblock URL \url{https://proceedings.mlr.press/v178/balcan22a.html}.

\bibitem[Balcan et~al.(2023)Balcan, Blum, Sharma, and
  Zhang]{balcan2023analysis}
Balcan, M.-F., Blum, A., Sharma, D., and Zhang, H.
\newblock An analysis of robustness of non-lipschitz networks.
\newblock \emph{Journal of Machine Learning Research}, 24\penalty0
  (98):\penalty0 1--43, 2023.

\bibitem[Biggio \& Roli(2018)Biggio and Roli]{biggio2018wild}
Biggio, B. and Roli, F.
\newblock Wild patterns: Ten years after the rise of adversarial machine
  learning.
\newblock \emph{Pattern Recognition}, 84:\penalty0 317--331, 2018.
\newblock \doi{10.1016/j.patcog.2018.07.023}.

\bibitem[Biggio et~al.(2013)Biggio, Corona, Maiorca, Nelson, Srndic, Laskov,
  Giacinto, and Roli]{BiggioCMNSLGR13}
Biggio, B., Corona, I., Maiorca, D., Nelson, B., Srndic, N., Laskov, P.,
  Giacinto, G., and Roli, F.
\newblock Evasion attacks against machine learning at test time.
\newblock In \emph{Machine Learning and Knowledge Discovery in Databases -
  European Conference, {ECML} {PKDD} 2013, Prague, Czech Republic, September
  23-27, 2013, Proceedings, Part {III}}, volume 8190 of \emph{Lecture Notes in
  Computer Science}, pp.\  387--402. Springer, 2013.
\newblock \doi{10.1007/978-3-642-40994-3\_25}.
\newblock URL \url{https://doi.org/10.1007/978-3-642-40994-3\_25}.

\bibitem[Charoenphakdee et~al.(2021)Charoenphakdee, Cui, Zhang, and
  Sugiyama]{charoenphakdee2021classification}
Charoenphakdee, N., Cui, Z., Zhang, Y., and Sugiyama, M.
\newblock Classification with rejection based on cost-sensitive classification.
\newblock In \emph{Proceedings of the 38th International Conference on Machine
  Learning ({ICML})}, volume 139 of \emph{Proceedings of Machine Learning
  Research}, pp.\  1507--1517. {PMLR}, 2021.
\newblock URL \url{http://proceedings.mlr.press/v139/charoenphakdee21a.html}.

\bibitem[Chen et~al.(2023)Chen, Yoon, Ebrahimi, Arik, Jha, and
  Pfister]{chen2023aspest}
Chen, J., Yoon, J., Ebrahimi, S., Arik, S.~{\"{O}}., Jha, S., and Pfister, T.
\newblock {ASPEST:} bridging the gap between active learning and selective
  prediction.
\newblock \emph{CoRR}, abs/2304.03870, 2023.
\newblock \doi{10.48550/arXiv.2304.03870}.
\newblock URL \url{https://doi.org/10.48550/arXiv.2304.03870}.

\bibitem[Cortes et~al.(2016)Cortes, DeSalvo, and Mohri]{cortes2016learning}
Cortes, C., DeSalvo, G., and Mohri, M.
\newblock Learning with rejection.
\newblock In \emph{Algorithmic Learning Theory - 27th International Conference,
  {ALT}}, volume 9925 of \emph{Lecture Notes in Computer Science}, pp.\
  67--82, 2016.
\newblock \doi{10.1007/978-3-319-46379-7\_5}.
\newblock URL \url{https://doi.org/10.1007/978-3-319-46379-7\_5}.

\bibitem[Croce \& Hein(2020)Croce and Hein]{croce2020reliable}
Croce, F. and Hein, M.
\newblock Reliable evaluation of adversarial robustness with an ensemble of
  diverse parameter-free attacks.
\newblock In \emph{International conference on machine learning}, pp.\
  2206--2216. PMLR, 2020.

\bibitem[Croce et~al.(2020)Croce, Andriushchenko, Sehwag, Flammarion, Chiang,
  Mittal, and Hein]{croce2020robustbench}
Croce, F., Andriushchenko, M., Sehwag, V., Flammarion, N., Chiang, M., Mittal,
  P., and Hein, M.
\newblock Robustbench: a standardized adversarial robustness benchmark.
\newblock \emph{CoRR}, abs/2010.09670, 2020.
\newblock URL \url{https://arxiv.org/abs/2010.09670}.

\bibitem[Cunningham \& Regan(2015)Cunningham and
  Regan]{cunningham2015autonomous}
Cunningham, M. and Regan, M.~A.
\newblock Autonomous vehicles: human factors issues and future research.
\newblock In \emph{Proceedings of the 2015 Australasian Road safety
  conference}, volume~14, 2015.

\bibitem[Geifman \& El{-}Yaniv(2019)Geifman and
  El{-}Yaniv]{geifman2019selectivenet}
Geifman, Y. and El{-}Yaniv, R.
\newblock {SelectiveNet}: {A} deep neural network with an integrated reject
  option.
\newblock In \emph{Proceedings of the 36th International Conference on Machine
  Learning ({ICML})}, volume~97 of \emph{Proceedings of Machine Learning
  Research}, pp.\  2151--2159. {PMLR}, 2019.
\newblock URL \url{http://proceedings.mlr.press/v97/geifman19a.html}.

\bibitem[Goldwasser et~al.(2020)Goldwasser, Kalai, Kalai, and
  Montasser]{goldwasser2020beyond}
Goldwasser, S., Kalai, A.~T., Kalai, Y., and Montasser, O.
\newblock Beyond perturbations: Learning guarantees with arbitrary adversarial
  test examples.
\newblock In \emph{Advances in Neural Information Processing Systems 33: Annual
  Conference on Neural Information Processing Systems ({NeurIPS})}, 2020.

\bibitem[Gowal et~al.(2019)Gowal, Dvijotham, Stanforth, Bunel, Qin, Uesato,
  Arandjelovic, Mann, and Kohli]{GowalDSBQUAMK19}
Gowal, S., Dvijotham, K., Stanforth, R., Bunel, R., Qin, C., Uesato, J.,
  Arandjelovic, R., Mann, T.~A., and Kohli, P.
\newblock Scalable verified training for provably robust image classification.
\newblock In \emph{2019 {IEEE/CVF} International Conference on Computer Vision,
  {ICCV} 2019, Seoul, Korea (South), October 27 - November 2, 2019}, pp.\
  4841--4850. {IEEE}, 2019.
\newblock \doi{10.1109/ICCV.2019.00494}.
\newblock URL \url{https://doi.org/10.1109/ICCV.2019.00494}.

\bibitem[He et~al.(2016)He, Zhang, Ren, and Sun]{HeZRS16}
He, K., Zhang, X., Ren, S., and Sun, J.
\newblock Deep residual learning for image recognition.
\newblock In \emph{2016 {IEEE} Conference on Computer Vision and Pattern
  Recognition, {CVPR} 2016, Las Vegas, NV, USA, June 27-30, 2016}, pp.\
  770--778. {IEEE} Computer Society, 2016.
\newblock \doi{10.1109/CVPR.2016.90}.
\newblock URL \url{https://doi.org/10.1109/CVPR.2016.90}.

\bibitem[Hosseini et~al.(2017)Hosseini, Chen, Kannan, Zhang, and
  Poovendran]{hosseini2017blocking}
Hosseini, H., Chen, Y., Kannan, S., Zhang, B., and Poovendran, R.
\newblock Blocking transferability of adversarial examples in black-box
  learning systems.
\newblock \emph{CoRR}, abs/1703.04318, 2017.
\newblock URL \url{http://arxiv.org/abs/1703.04318}.

\bibitem[Kato et~al.(2020)Kato, Cui, and Fukuhara]{kato2020atro}
Kato, M., Cui, Z., and Fukuhara, Y.
\newblock {ATRO:} adversarial training with a rejection option.
\newblock \emph{CoRR}, abs/2010.12905, 2020.
\newblock URL \url{https://arxiv.org/abs/2010.12905}.

\bibitem[Krizhevsky et~al.(2009)Krizhevsky, Hinton,
  et~al.]{krizhevsky2009learning}
Krizhevsky, A., Hinton, G., et~al.
\newblock Learning multiple layers of features from tiny images.
\newblock 2009.

\bibitem[Laidlaw \& Feizi(2019)Laidlaw and Feizi]{laidlaw2019playing}
Laidlaw, C. and Feizi, S.
\newblock Playing it safe: {A}dversarial robustness with an abstain option.
\newblock \emph{CoRR}, abs/1911.11253, 2019.
\newblock URL \url{http://arxiv.org/abs/1911.11253}.

\bibitem[Laidlaw et~al.(2021)Laidlaw, Singla, and Feizi]{Laidlaw0F21}
Laidlaw, C., Singla, S., and Feizi, S.
\newblock Perceptual adversarial robustness: Defense against unseen threat
  models.
\newblock In \emph{9th International Conference on Learning Representations,
  {ICLR} 2021, Virtual Event, Austria, May 3-7, 2021}. OpenReview.net, 2021.
\newblock URL \url{https://openreview.net/forum?id=dFwBosAcJkN}.

\bibitem[LeCun(1998)]{lecun1998mnist}
LeCun, Y.
\newblock The {MNIST} database of handwritten digits.
\newblock 1998.
\newblock URL \url{http://yann.lecun.com/exdb/mnist/}.

\bibitem[LeCun et~al.(1989)LeCun, Boser, Denker, Henderson, Howard, Hubbard,
  and Jackel]{CunBDHHHJ89}
LeCun, Y., Boser, B.~E., Denker, J.~S., Henderson, D., Howard, R.~E., Hubbard,
  W.~E., and Jackel, L.~D.
\newblock Handwritten digit recognition with a back-propagation network.
\newblock In \emph{Advances in Neural Information Processing Systems 2, {[NIPS}
  Conference, Denver, Colorado, USA, November 27-30, 1989]}, pp.\  396--404.
  Morgan Kaufmann, 1989.

\bibitem[Madry et~al.(2018)Madry, Makelov, Schmidt, Tsipras, and
  Vladu]{madry2018towards}
Madry, A., Makelov, A., Schmidt, L., Tsipras, D., and Vladu, A.
\newblock Towards deep learning models resistant to adversarial attacks.
\newblock In \emph{6th International Conference on Learning Representations,
  Conference Track Proceedings}. OpenReview.net, 2018.
\newblock URL \url{https://openreview.net/forum?id=rJzIBfZAb}.

\bibitem[Markoff(2016)]{John16}
Markoff, J.
\newblock For now, self-driving cars still need humans.
\newblock
  \url{https://www.nytimes.com/2016/01/18/technology/driverless-cars-limits-include-human-nature.html},
  2016.
\newblock The New York Times.

\bibitem[Mozannar \& Sontag(2020)Mozannar and Sontag]{mozannar2020consistent}
Mozannar, H. and Sontag, D.
\newblock Consistent estimators for learning to defer to an expert.
\newblock In \emph{International Conference on Machine Learning}, pp.\
  7076--7087. PMLR, 2020.

\bibitem[Netzer et~al.(2011)Netzer, Wang, Coates, Bissacco, Wu, and
  Ng]{netzer2011reading}
Netzer, Y., Wang, T., Coates, A., Bissacco, A., Wu, B., and Ng, A.~Y.
\newblock Reading digits in natural images with unsupervised feature learning.
\newblock 2011.

\bibitem[Pang et~al.(2021)Pang, Yang, Dong, Su, and Zhu]{pang2021bag}
Pang, T., Yang, X., Dong, Y., Su, H., and Zhu, J.
\newblock Bag of tricks for adversarial training.
\newblock In \emph{9th International Conference on Learning Representations
  ({ICLR})}. OpenReview.net, 2021.
\newblock URL \url{https://openreview.net/forum?id=Xb8xvrtB8Ce}.

\bibitem[Pang et~al.(2022)Pang, Zhang, He, Dong, Su, Chen, Zhu, and
  Liu]{pang2022two}
Pang, T., Zhang, H., He, D., Dong, Y., Su, H., Chen, W., Zhu, J., and Liu,
  T.-Y.
\newblock Two coupled rejection metrics can tell adversarial examples apart.
\newblock In \emph{Proceedings of the {IEEE/CVF} Conference on Computer Vision
  and Pattern Recognition}, pp.\  15223--15233, 2022.

\bibitem[Sheikholeslami et~al.(2021)Sheikholeslami, Lotfi, and
  Kolter]{sheikholeslami2021provably}
Sheikholeslami, F., Lotfi, A., and Kolter, J.~Z.
\newblock Provably robust classification of adversarial examples with
  detection.
\newblock In \emph{9th International Conference on Learning Representations
  ({ICLR})}. OpenReview.net, 2021.
\newblock URL \url{https://openreview.net/forum?id=sRA5rLNpmQc}.

\bibitem[Stutz et~al.(2020)Stutz, Hein, and Schiele]{stutz2020ccat}
Stutz, D., Hein, M., and Schiele, B.
\newblock Confidence-calibrated adversarial training: Generalizing to unseen
  attacks.
\newblock In \emph{Proceedings of the 37th International Conference on Machine
  Learning ({ICML})}, volume 119 of \emph{Proceedings of Machine Learning
  Research}, pp.\  9155--9166. {PMLR}, 2020.
\newblock URL \url{http://proceedings.mlr.press/v119/stutz20a.html}.

\bibitem[Szegedy et~al.(2014)Szegedy, Zaremba, Sutskever, Bruna, Erhan,
  Goodfellow, and Fergus]{SzegedyZSBEGF13}
Szegedy, C., Zaremba, W., Sutskever, I., Bruna, J., Erhan, D., Goodfellow,
  I.~J., and Fergus, R.
\newblock Intriguing properties of neural networks.
\newblock In \emph{2nd International Conference on Learning Representations,
  {ICLR} 2014, Banff, AB, Canada, April 14-16, 2014, Conference Track
  Proceedings}, 2014.
\newblock URL \url{http://arxiv.org/abs/1312.6199}.

\bibitem[Tax \& Duin(2008)Tax and Duin]{tax2008growing}
Tax, D. M.~J. and Duin, R. P.~W.
\newblock Growing a multi-class classifier with a reject option.
\newblock \emph{Pattern Recognition Letters}, 29\penalty0 (10):\penalty0
  1565--1570, 2008.
\newblock \doi{10.1016/j.patrec.2008.03.010}.
\newblock URL \url{https://doi.org/10.1016/j.patrec.2008.03.010}.

\bibitem[Tram{\`{e}}r(2022)]{tramer2021detecting}
Tram{\`{e}}r, F.
\newblock Detecting adversarial examples is (nearly) as hard as classifying
  them.
\newblock In \emph{International Conference on Machine Learning, {ICML} 2022,
  17-23 July 2022, Baltimore, Maryland, {USA}}, volume 162 of \emph{Proceedings
  of Machine Learning Research}, pp.\  21692--21702. {PMLR}, 2022.
\newblock URL \url{https://proceedings.mlr.press/v162/tramer22a.html}.

\bibitem[Tram{\`{e}}r et~al.(2020)Tram{\`{e}}r, Carlini, Brendel, and
  Madry]{tramer2020adaptive}
Tram{\`{e}}r, F., Carlini, N., Brendel, W., and Madry, A.
\newblock On adaptive attacks to adversarial example defenses.
\newblock In \emph{Advances in Neural Information Processing Systems 33: Annual
  Conference on Neural Information Processing Systems}, 2020.

\bibitem[Wang et~al.(2020)Wang, Zou, Yi, Bailey, Ma, and Gu]{wang2020improving}
Wang, Y., Zou, D., Yi, J., Bailey, J., Ma, X., and Gu, Q.
\newblock Improving adversarial robustness requires revisiting misclassified
  examples.
\newblock In \emph{8th International Conference on Learning Representations
  ({ICLR})}. OpenReview.net, 2020.
\newblock URL \url{https://openreview.net/forum?id=rklOg6EFwS}.

\bibitem[Yin et~al.(2020)Yin, Kolouri, and Rohde]{yin2020generative}
Yin, X., Kolouri, S., and Rohde, G.~K.
\newblock {GAT:} generative adversarial training for adversarial example
  detection and robust classification.
\newblock In \emph{8th International Conference on Learning Representations
  ({ICLR})}. OpenReview.net, 2020.
\newblock URL \url{https://openreview.net/forum?id=SJeQEp4YDH}.

\bibitem[Zhang et~al.(2019)Zhang, Yu, Jiao, Xing, Ghaoui, and
  Jordan]{zhang2019theoretically}
Zhang, H., Yu, Y., Jiao, J., Xing, E.~P., Ghaoui, L.~E., and Jordan, M.~I.
\newblock Theoretically principled trade-off between robustness and accuracy.
\newblock In \emph{Proceedings of the 36th International Conference on Machine
  Learning ({ICML})}, volume~97 of \emph{Proceedings of Machine Learning
  Research}, pp.\  7472--7482. {PMLR}, 2019.
\newblock URL \url{http://proceedings.mlr.press/v97/zhang19p.html}.

\end{thebibliography}
\bibliographystyle{icml2023}

\newpage
\onecolumn
\appendix

\begin{center}
	\textbf{\LARGE Supplementary Material }
\end{center}


In Section~\ref{sec:societal-impacts}, we discuss some potential societal impacts of our work. In Section~\ref{sec:proofs}, we provide detailed proofs for the theorems and also present some additional theoretical results. In Section~\ref{sec:appendix_total_robust_loss}, we discuss the calculation of the total robust loss for specific instantiations of the rejection loss $\lrej$. In Section~\ref{sec:detailed-setup}, we describe the detailed setup for the experiments. In Section~\ref{sec:app_adaptive_attacks}, we discuss the unified approach for designing adaptive attacks for the proposed method CPR and all the baseline methods. In Section~\ref{sec:additional-exp-results}, we present some additional experimental results.

\section{Societal Impacts} 
\label{sec:societal-impacts}
Our method can build robust models with small costs due to mis-classification and rejection facing adversarial input perturbations, which can help build robust intelligent systems for applications. The proposed stratified rejection setting, which takes into account the cost of rejection, may inspire the development of improved defense methods along this line, which can have a positive impact in practice. We foresee no immediate negative societal impact.

\section{Additional Theory and Proofs}
\label{sec:proofs}

\subsection{Alternate Definition of \texorpdfstring{$\rerej_{\epsilon}(f, \alpha)$}{R(f, a)}}
\label{sec:proof_prop1}

\begin{proposition}
\label{proposition1}
For any $\epsilon \geq 0$ and $\alpha \in [0, 1]$, the metric $\rerej_{\epsilon}(f, \alpha)$ can be equivalently defined as
\begin{align}
\label{eq:robust_error_rejection_alt}
 \rerej_{\epsilon}(f, \alpha) ~= \expec_{(\xb,y) \sim \calD} \Big[ \max_{\xb' \in \nei(\xb, \,\alpha\epsilon)} \ind\big\{ f(\xb') = \rejsym \big\} ~\vee \max_{\xb'' \in \nei(\xb, \epsilon) } \ind\big\{ f(\xb'') \not \in \{y, \rejsym\} \big\} \Big].
\end{align}
This definition allows us to interpret the metric $\rerej_{\epsilon}(f, \alpha)$ as consisting of two types of errors: 1) errors due to rejecting small perturbations within the $\alpha\epsilon$-neighborhood and 2) mis-classification errors within the $\epsilon$-neighborhood.
\end{proposition}
\begin{proof}
Consider the original definition of $\rerej_{\epsilon}(f, \alpha)$ in Eq. (\ref{eq:robust_error_rejection})
\begin{align}
\label{eq:robust_error_rejection_repl}
 \rerej_{ \epsilon}(f, \alpha) ~= \expec_{(\xb,y) \sim \calD} \Big[ \max_{\xb' \in \nei(\xb, \,\alpha\epsilon)} \ind\big\{ f(\xb') \neq y \big\} ~\vee \max_{\xb'' \in \nei(\xb, \epsilon) } \ind\big\{ f(\xb'') \not \in \{y, \rejsym\} \big\} \Big].
\end{align}
The error in the first term inside the expectation can be split into the error due to rejection and the error due to mis-classification, \ie
\begin{equation*}
\ind\big\{ f(\xb') \neq y \big\} ~=~ \ind\big\{ f(\xb') = \rejsym \big\} \,\vee\, \ind\big\{ f(\xb') \notin \{y, \rejsym\} \big\}.
\end{equation*}
The maximum of this error over the $\alpha\epsilon$-neighborhood can be expressed as
\begin{align*}
\max_{\xb' \in \nei(\xb, \,\alpha\epsilon)} \ind\big\{ f(\xb') \neq y \big\} ~=\, \max_{\xb' \in \nei(\xb, \,\alpha\epsilon)} \ind\big\{ f(\xb') = \rejsym \big\} ~\vee \max_{\xb' \in \nei(\xb, \,\alpha\epsilon)} \!\!\ind\big\{ f(\xb') \notin \{y, \rejsym\} \big\},
\end{align*}
which is easily verified for binary indicator functions.
Substituting the above result into Eq. (\ref{eq:robust_error_rejection_repl}), we get
\begin{align*}
\rerej_{ \epsilon}(f, \alpha) ~&= \expec_{(\xb,y) \sim \calD} \Big[ \max_{\xb' \in \nei(\xb, \,\alpha\epsilon)} \ind\big\{ f(\xb') = \rejsym \big\} ~\vee \max_{\xb' \in \nei(\xb, \,\alpha\epsilon)} \!\!\ind\big\{ f(\xb') \notin \{y, \rejsym\} \big\} ~\vee \max_{\xb'' \in \nei(\xb, \epsilon) } \ind\big\{ f(\xb'') \not \in \{y, \rejsym\} \big\} \Big] \\
&= \expec_{(\xb,y) \sim \calD} \Big[ \max_{\xb' \in \nei(\xb, \,\alpha\epsilon)} \ind\big\{ f(\xb') = \rejsym \big\} ~\vee \max_{\xb'' \in \nei(\xb, \epsilon) } \ind\big\{ f(\xb'') \not \in \{y, \rejsym\} \big\} \Big].
\end{align*}
In the last step, we combined the second and third terms inside the expectation using the observation that
\begin{align*}
\max_{\xb' \in \nei(\xb, \,\alpha\epsilon)} \!\!\ind\big\{ f(\xb') \notin \{y, \rejsym\} \big\} ~\vee \max_{\xb' \in \nei(\xb, \epsilon)} \!\!\ind\big\{ f(\xb') \notin \{y, \rejsym\} \big\} ~=~ \max_{\xb' \in \nei(\xb, \epsilon)} \ind\big\{ f(\xb') \notin \{y, \rejsym\} \big\}.
\end{align*}
This shows the equivalence of the two definitions of $\rerej_{ \epsilon}(f, \alpha)$.
\end{proof}
The definition in Eq. (\ref{eq:robust_error_rejection_alt}) serves as motivation for the first adaptive attack to create adversarial examples $\bfx'$ within the neighborhood $\nei(\bfx, \alpha\epsilon)$ that are rejected by the defense method. 

\subsection{Proof of Lemma~\ref{corr_total_robust_loss}}

We first prove a more general lemma in Lemma~\ref{lemma_total_loss}, from which integration by parts gives Lemma~\ref{corr_total_robust_loss}.
Note that the step rejection loss and ramp rejection loss satisfy the conditions in Lemma~\ref{corr_total_robust_loss} which gives them a simpler expression to compute the total robust loss. 

\label{sec:proof_lemma1}
\begin{lemma}\label{lemma_total_loss}
Let $s(\alpha) :=  \rerej_{\epsilon}(f, \alpha)$ and assume it is right-semicontinuous. 
For any monotonically non-increasing $\lrej: [0, \infty) \rightarrow [0,1]$, the total robust loss can be computed by:
\begin{align}
\label{eq:total_loss_compute}
\tlrej_\epsilon(f; \lrej) ~=~ \rerej_{\epsilon}(f, 0) ~+~ \left(\lrej(0) \,-\, 1\right) \,p_\textrm{rej} 
~+~ \int_0^1 \lrej(\alpha \epsilon) ~\mathrm{d} s(\alpha),
\end{align}
where the integral is the Riemann–Stieltjes integral, and 
\begin{align*}
  p_\textrm{rej} \,:=\, \Pr\Big[ f(\xb) = \rejsym ~\wedge~ \big(\forall \xb' \in \nei(\xb, \epsilon),~f(\xb') \in \{y, \rejsym\} \big) \Big]
\end{align*}
is the probability that the clean input $(\bfx, y)$ is rejected and no perturbations of $\bfx$ within the $\epsilon$-ball are misclassified.
\end{lemma}
\begin{proof}
Let $\calW$ denote the event that there exists $\xb' \in \nei(\xb, \epsilon)$ such that $\xb'$ is misclassified, i.e., $f(\xb') \not\in \{y, \rejsym\}$. 
Let $\calC$ denote the event that the clean input $\xb$ is rejected, i.e., $ f(\xb) = \rejsym$.  Clearly, $p_\textrm{rej} = \Pr[ \calC \setminus \calW]$ where $\setminus $ is the set minus. Finally, let $\calR$ denote the event that there exists $\xb' \in \nei(\xb, \epsilon)$ such that $\xb'$ is rejected. 

We only consider $\calW \cup \calC \cup \calR$, since otherwise $(\xb,y)$ contributes a loss 0 to $\tlrej_\epsilon(f; \lrej)$. The union can be partitioned into three disjoint subsets.
\begin{itemize}
    \item $\calW$ : Such an $(\xb,y)$ contributes a loss $1$. Since $\rerej_{\epsilon}(f, 0) =  \Pr[\calW \cup \calC]$, we have $\,\Pr[\calW] = \rerej_{\epsilon}(f, 0) - \Pr[\calC \setminus \calW] = \rerej_{\epsilon}(f, 0) - p_\textrm{rej}$. Then this subset contributes a loss $\,\rerej_{\epsilon}(f, 0) - p_\textrm{rej}$. 
    \item $\calC \setminus \calW$ : Such a data point contributes a loss $\lrej(0)$, given the assumption that $\lrej$ is monotonically non-increasing. Then this subset contributes a loss $\,\lrej(0) p_\textrm{rej}$.
    \item $\calR \setminus (\calW \cup \calC)$ : That is, there exists no $\xb' \in \nei(\xb, \epsilon)$ that is misclassified, the clean input $\xb$ is accepted, but there exists some $\xb' \in \nei(\xb, \epsilon)$ that is rejected. Let $L_3$ denote the loss contributed by this subset.
\end{itemize}
Now, it is sufficient to show that  
\[
L_3 = \,\int_0^1 \lrej(\alpha \epsilon) \,\mathrm{d} s(\alpha).
\]
For any positive integer $t$, let $a_0=0 \le a_1 \le \dots \le a_{t-1} \le a_t = 1$ be an arbitrary sequence on $[0, 1]$. Let
\begin{align*}
    M_i ~&=~ \sup \{\lrej(\alpha \epsilon), ~~a_{i-1} \le \alpha \le a_i\},
    \\
    m_i ~&=~ \inf \{\lrej(\alpha \epsilon), ~~a_{i-1} \le \alpha \le a_i\},
    \\
    U(s, \lrej) ~&=~ \sum_{i=1}^t \,M_i \,(s(a_i) - s(a_{i-1})),
    \\
    L(s, \lrej) ~&=~ \sum_{i=1}^t \,m_i \,(s(a_i) - s(a_{i-1})).
\end{align*}
%
%
Let $\calR_i$ denote the event that there exists $\xb'$ such that $\,a_{i-1}\epsilon < d(\xb', \xb) \le a_i\epsilon$ and $\xb'\,$ is rejected. 
Then $\,\calR = \cup_{i=1}^t \calR_i$, and
\begin{align*}
    s(a_i) \,-\, s(a_{i-1}) ~&=~ \rerej_{\epsilon}(f, a_i) \,-\, \rerej_{\epsilon}(f, a_{i-1})
    \\
    &=~ \Pr\big[\calR_i \setminus (\cup_{j=1}^{i-1} \calR_j) \setminus (\calW \cup \calC)\big].
\end{align*}
 
Since $\lrej$ is monotonically non-increasing, each data point in $\calR_i \setminus (\cup_{j=1}^{i-1} \calR_j) \setminus (\calW \cup \calC)$ should contribute a loss that is within $[m_i, M_i]$. 
Therefore,
\begin{align*}
    L(s, \lrej) ~\le~ L_3 ~\le~ U(s, \lrej)
\end{align*}
for any sequence $\{a_i\}_{i=0}^t$. When $s(\alpha)$ is right-semicontinous, the Riemann–Stieltjes integral exists, and $L_3 = \int_0^1 \lrej(\alpha \epsilon) \,\mathrm{d} s(\alpha)$. This completes the proof.
\end{proof}

\subsection{Proof of Theorem~\ref{thm:main}}
\label{sec:proof_thm1}

\begin{theorem}[Restatement of Theorem~\ref{thm:main}]
Consider binary classification. Let $f^*(\xb)$ be any classifier without a rejection option. For any $\delta \in [0,1]$ and $\epsilon \geq 0$, there exists a selective classifier $f_\delta$, whose robust error curve is bounded by: 
\begin{align}
    \rerej_{ \epsilon}(f_\delta, \alpha) \,\le\, \ore_{\epsilon'}(f^*), ~\forall \alpha \in [0,1]
\end{align}
where $\epsilon' = \max\{(\alpha+\delta)\epsilon, (1-\delta)\epsilon\}$.
Moreover, the bound is tight: for any $\alpha \in [0,1]$, there exist simple data distributions and $f^*$ such that any $f$ must have $\rerej_{ \epsilon}(f, \alpha) \ge  \ore_{\epsilon'}(f^*)$.
\end{theorem}

\begin{proof}
For any $r>0$, let $\nei(f^*, r)$ denote the region within distance $r$ to the decision boundary of $f^*$:
\begin{align*}
    \nei(f^*, r) := \{\xb \in \inputs: \exists \xb' \in \nei(\bfx, r) \textrm{~and~} f^*(\xb') \neq f^*(\xb)\}.
\end{align*}
Consider a parameter $\delta \in [0, 1]$ and construct a selective classifier $f_\delta$ as follows:
\begin{align}
    f_\delta(\xb) ~:=~ 
    \begin{cases}
    \rejsym & \textrm{if~} \xb \in \nei(f^*, \delta\epsilon), 
    \\
    f^*(\xb) & \textrm{otherwise.}  
    \end{cases}
\end{align}
We will show that any sample $(\xb, y)$ contributing error to $\rerej_{\epsilon}(f_\delta, \alpha)$ must also contribute error to $\ore_{\epsilon'}(f^*)$, where $\epsilon' = \max\{ (\alpha + \delta)\epsilon,  (1- \delta) \epsilon \}$. This will prove that $\rerej_{ \epsilon}(f_\delta, \alpha) \le \ore_{\epsilon'}(f^*)$.
 
Consider the following two cases:
\begin{itemize}
    \item 
    Consider the first type of error in $\rerej_{\epsilon}(f_\delta, \alpha)$: $ \max_{\xb' \in \nei(\xb, \alpha \epsilon)} \ind\left[ f_\delta(\xb') \neq y\right] = 1$. This implies that there exists $\xb' \in \nei(\xb, \alpha \epsilon)$ such that $f_\delta(\xb') \neq y$. 
    So there are two subcases to consider:
    \begin{itemize}
        \item[(1)] $\xb' \in \nei(f^*,\delta \epsilon)$: in this case $\,\xb \in \nei(f^*, (\delta + \alpha)\epsilon)$.
        \item[(2)] $f^*(\xb') \neq y$: in this case either $f^*(\xb) \neq y$, or $f^*(\xb) = y \neq f^*(\xb') $ and thus $\,\xb \in \nei(f^*, \alpha\epsilon )$.
    \end{itemize}  
    In summary, either $\,f^*(\xb) \neq y$ or $\xb \in \nei(f^*,  (\alpha+ \delta)\epsilon)$. 
 
    \item 
    Next consider the second type of error in $\rerej_{ \epsilon}(f_\delta, \alpha)$: $\max_{\xb'' \in \nei(\xb, \epsilon) } \ind\left[ f_\delta(\xb'') \not \in \{y, \rejsym\} \right] = 1$. This means there exists an  $\xb'' \in \nei(\xb, \epsilon)$ such that $f_\delta(\xb'') \not\in \{y, \rejsym\}$, \ie $\xb'' \not\in \nei(f^*, \delta\epsilon)$ and $f^*(\xb'') \neq y$. This implies that {\em all} $\mathbf{z} \in \nei(\xb'', \delta \epsilon)$ should have $f^*(\mathbf{z}) = f^*(\xb'') \neq y$. 
    In particular, there exists $\mathbf{z} \in \nei(\xb'', \delta \epsilon)$ with $d(\mathbf{z}, \xb) \le  (1 - \delta)\epsilon$ and $f^*(\mathbf{z}) \neq y$. 
    It can be verified that $\,\mathbf{z} =  \delta \xb + (1 - \delta) \xb''$, which is a point on the line joining $\xb$ and $\xb''$, satisfies the above condition.
    In summary, either $f^*(\xb) \neq y$, or $f^*(\xb) = y \neq f^*(\mathbf{z})$ and thus $\,\xb \in \nei(f^*, (1 - \delta)\epsilon)$. 

\end{itemize}
Overall, a sample $(\xb, y)$ contributing error to $\rerej_{ \epsilon}(f_\delta, \alpha)$ must satisfy either $f^*(\xb) \neq y$ or $\xb \in \nei(f^*, \epsilon')$.
Clearly, such a sample also contributes an error to $\ore_{\epsilon'}(f^*)$. Therefore, we have
\begin{align}
    \rerej_{ \epsilon}(f_\delta, \alpha) \,\le\, \ore_{\epsilon'}(f^*), \forall \alpha \in [0,1].
\end{align}
 
To show that the bound is tight, consider the following data distribution. Let $\xb \in \mathbb{R}$ and $y \in \{-1, +1\}$,  $\, \alpha  \in [0, 1]$, and let $\beta \in (0, \frac{1}{2})$ be some constant. Let the distribution be: $(\xb,y)$ is $(-4\epsilon, -1)$ with probability $\frac{1-\beta}{2}$, $\,(- \frac{\alpha \epsilon}{4}, -1)$ with probability $\frac{\beta}{2}$, $\,(\frac{\alpha \epsilon}{4}, +1)$ with probability $\frac{\beta}{2}$, and $(4\epsilon, +1)$ with probability $\frac{1-\beta}{2}$. Let $\delta = \frac{1-\alpha}{2}, \,f^*(\xb) := \textrm{sign}(\xb + \epsilon)$. It is clear that $\ore_{\epsilon'}(f^*) = \ore_{(1 + \alpha )\epsilon/2}(f^*) = \frac{\beta}{2}$. It is also clear that any $f$ must have $\rerej_{ \epsilon}(f, \alpha) \ge \frac{\beta}{2}$ since the points $-\frac{\alpha \epsilon}{4}$ and $\frac{\alpha \epsilon}{4}$ have distance only $\frac{\alpha \epsilon}{2}$ but have different labels. 
\end{proof}

We note that the proof generalizes that of Theorem 5 in~\cite{tramer2021detecting}. In particular, our theorem includes the latter as a special case (corresponding to $\alpha = 0$ and $\delta = \frac{1}{2}$). 

\subsection{Comparing the Robustness Curves and Total Robust Losses of \texorpdfstring{$f_\delta$}{fdelta} and \texorpdfstring{$f^*$}{f*} in Theorem~\ref{thm:main}}
\label{sec:closed-form-solution}

Theorem~\ref{thm:main} investigates the feasibility of a good selective classifier (w.r.t.\ the robustness curve and the total robust loss), by constructing a classifier $f_\delta$. 
It is meaningful to perform a fine-grained analysis into when $f_\delta$ has a better total robust loss than the classifier $f^*$ without rejection. 

For simplicity, we assume $f^*$ has a 0 standard error on the clean data distribution. The analysis for the case with a nonzero standard error is similar.

\begin{theorem}
\label{thm:closed-form-solution}
Consider binary classification. Let $f^*(\xb)$ be any classifier without a rejection option that has $0$ standard error on the data distribution.  Suppose the data density for data points with a distance of $r$ to the decision boundary of $f^*$ is $p(r)$. Consider the selective classifier $f_\delta$ defined in Theorem~\ref{thm:main}:
\begin{align}
    f_\delta(\xb) ~:=~ 
    \begin{cases}
    \rejsym & \textrm{if~} \xb \in \nei(f^*, \delta\epsilon), 
    \\
    f^*(\xb) & \textrm{otherwise.}  
    \end{cases}
\end{align}
Then for $\delta \in [0, \frac{1}{2}]$, the total robust loss of $f_\delta$ with respect to the rejection loss $\lrej$ can be computed by: 
\begin{align}
    \label{eq:closed-form-total-robust-loss}
    \tlrej_\epsilon(f_\delta; \lrej) &= \int_{0}^{(1-\delta)\epsilon} p(r) \mathrm{d} r + \int_{(1-\delta)\epsilon}^{(1+\delta)\epsilon} \lrej(r-\delta\epsilon) \,p(r) \mathrm{d} r.
\end{align}
Also, the curve of robust error of $f_\delta$ can be computed by: 
\begin{align}
    \label{eq:closed-form-robust-error-curve}
    \rerej_{\epsilon}(f_\delta, \alpha) &= \begin{cases}
        \int_{0}^{(1-\delta)\epsilon} p(r) \mathrm{d} r  & \quad \text{if } \alpha \in [0, 1-2\delta], \\
      \int_{0}^{ (\alpha+\delta)\epsilon} p(r) \mathrm{d} r & \quad \text{if } \alpha \in (1-2\delta, 1].
  \end{cases}
\end{align}
\end{theorem}
\begin{proof} 
First consider the total robust loss. When $0 < r\leq (1-\delta)\epsilon$, the data point will contribute a loss of $1$ to the total robust loss. When $(1-\delta)\epsilon <r \leq (1+\delta)\epsilon$, the data point will contribute a loss of $\lrej(r-\delta\epsilon)$ to the total robust loss. When $r > (1+\delta)\epsilon$, the data point will contribute $0$ loss to the total robust loss. Thus, we have Eq. (\ref{eq:closed-form-total-robust-loss}).

For the curve of robust error, the data points with $0<r\leq (1-\delta)\epsilon$ will always contribute a loss of $1$. When $\alpha \in [0, 1-2\delta]$, no data points with $r \ge (1-\delta)\epsilon$ can contribute a loss, either by small perturbations to get a rejection or mis-classification, or by large perturbations to get a mis-classification. When $\alpha \in (1-2\delta, 1]$, data points with $r \in [(1-\delta)\epsilon, (\alpha + \delta) \epsilon]$ will contribute a loss of $1$ by small perturbations to get a rejection. Thus, we have Eq. (\ref{eq:closed-form-robust-error-curve}).  
\end{proof}

Now we are ready to compare $f_\delta$ to $f^*$. 

First consider the curve of robust error. It is easy to know that the curve of robust error of $f^*$ is $\rerej_{\epsilon}(f^*, \alpha) = \int_{0}^{\epsilon} p(r) \mathrm{d} r$ for $\alpha\in [0,1]$. When $\alpha \in [0, 1-\delta]$, we have $\rerej_{\epsilon}(f^*, \alpha)\geq \rerej_{\epsilon}(f_\delta, \alpha)$; when $\alpha \in (1-\delta, 1]$, we have $\rerej_{\epsilon}(f^*, \alpha) \leq \rerej_{\epsilon}(f_\delta, \alpha)$. When $\alpha \in [0, 1-2\delta]$, if $\int_{(1-\delta)\epsilon}^{\epsilon} p(r) \mathrm{d} r$ is large, then $f_\delta$ will have much lower robust error with rejection than $f^*$, since $\rerej_{\epsilon}(f^*, \alpha) - \rerej_{\epsilon}(f_\delta, \alpha) = \int_{(1-\delta)\epsilon}^{\epsilon} p(r) \mathrm{d} r$; when $\alpha \in (1-2\delta, 1-\delta]$, if $\int_{(\alpha+\delta)\epsilon}^{\epsilon} p(r) \mathrm{d} r$ is large, then $f_\delta$ will also have much lower robust error with rejection than $f^*$, since $\rerej_{\epsilon}(f^*, \alpha) - \rerej_{\epsilon}(f_\delta, \alpha) = \int_{(\alpha+\delta)\epsilon}^{\epsilon} p(r) \mathrm{d} r$. 

Now consider the total robust loss.
The total robust loss of $f^*$ is $\int_{0}^{\epsilon} p(r) \mathrm{d} r$, which is larger than that of $f_\delta$ by $ \int_{(1-\delta)\epsilon}^{\epsilon} p(r) \mathrm{d} r - \int_{(1-\delta)\epsilon}^{(1+\delta)\epsilon} \lrej(r-\delta\epsilon) \,p(r) \mathrm{d} r$. 
More precisely, both $f^*$ and $f_\delta$ will get mis-classification loss on some perturbations of points with distance in the range $[0, (1-\delta) \epsilon]$ from the decision boundary of $f^*$. This is because there always exist some large perturbations of these points crossing the boundary.
On the other hand, $f^*$ simply gets mis-classification loss from perturbations of the points with distance in the range $[(1-\delta)\epsilon, \epsilon]$. While for all points with distance larger than $(1-\delta)\epsilon$, $f_\delta$ can correctly classify all their small perturbations of magnitude at most $(1-2\delta)\epsilon$, and reject or correctly classify their larger perturbations. So it only gets rejection loss for large magnitudes, which can then potentially lead to smaller total robust losses than $f^*$ for monotonically non-increasing rejection losses. 

Therefore, for some data distributions, there exists a fixed $\delta$ such that $f_\delta$ can get small total robust losses with respect to a wide range of reasonable rejection losses. For example, consider the step rejection losses $\lrej(r)=\ind\{r \,\leq\, \alpha_0\epsilon\}$ with parameter in the range $\alpha_0\in [0,\bar{\alpha}_0]$ where $\bar{\alpha}_0\in [0,1]$. If we set $\delta=\frac{1-\bar{\alpha}_0}{2}$, then $1-2\delta=\bar{\alpha}_0$ and $\alpha \in [0, 1-2\delta]$. The total robust loss of $f_\delta$ with respect to these rejection losses is $\int_{0}^{\frac{(1+\bar{\alpha}_0)}{2}\epsilon  } p(r) \mathrm{d} r$. In contrast, the total robust loss of $f^*$ is $\int_{0}^{\epsilon} p(r) \mathrm{d} r$, which can be significantly larger than that of $f_\delta$. The total robust loss of $f_\delta$ is smaller than that of $f^*$ by the amount $\int_{\frac{(1+\bar{\alpha}_0)}{2}\epsilon  }^{\epsilon} p(r) \mathrm{d}r$. If the probability mass of points with $r \in [\frac{(1+\bar{\alpha}_0)}{2}\epsilon, \epsilon]$ is large, then the improvement is significant.

\section{Calculating the Total Robust Loss}
\label{sec:appendix_total_robust_loss}
In this section, we discuss the calculation of the total robust loss for specific instantiations of the rejection loss $\lrej$.
Given the curve of robust error $\,\{s(\alpha) := \rerej_{ \epsilon}(f, \alpha) ~:~ \alpha \in [0,1]\}$ and a specific choice of rejection loss $\lrej(r)$ that is monotonically non-increasing, we can use Eq. (\ref{eq:total_loss_compute}) from Lemma~\ref{lemma_total_loss} to calculate the total robust loss:
\begin{align}
\label{eq:copy_total_robust_loss}
    \tlrej_\epsilon(f; \lrej) ~=~ \rerej_{\epsilon}(f, 0) ~+~ (\lrej(0) \,-\, 1) \,p_\textrm{rej} ~+~ \int_0^1 \lrej(\alpha \epsilon) ~\mathrm{d} s(\alpha).
\end{align}
As discussed in Corollary~\ref{corr_total_robust_loss}, let us additionally choose the rejection loss to be differentiable almost everywhere, and satisfy the conditions $\lrej(0) = 1$ and $\lrej(\epsilon) = 0$. This is satisfied \eg by the ramp and step rejection losses defined in Eqs. (\ref{eq:ramp_rejection_loss}) and (\ref{eq:step_rejection_loss}).
Applying the product rule (or integration by parts), the integral term in the total robust loss can be expressed as
\begin{align*}
\int_0^1 \lrej(\alpha \epsilon) ~\mathrm{d} s(\alpha) ~&=~ \lrej(\epsilon)\,s(1) ~-~ \lrej(0)\,s(0) ~-~ \int_0^1 s(\alpha) ~\mathrm{d} \lrej(\alpha \epsilon) \\
&=~ -s(0) ~-~ \int_0^1 s(\alpha) ~\mathrm{d} \lrej(\alpha \epsilon).
\end{align*}
Substituting the above into Eq. (\ref{eq:copy_total_robust_loss}), the total robust loss simplifies into
\begin{align}
\label{eq:copy_total_robust_loss_simplified}
    \tlrej_\epsilon(f; \lrej) ~=~ -\int_0^1 s(\alpha) ~\mathrm{d} \lrej(\alpha \epsilon).
\end{align}
From the above expression, we next calculate the total robust loss for the ramp and step rejection losses.

\subsection{Ramp Rejection loss}
Recall that the ramp rejection loss is defined as
\begin{align*}
\lrej(r) ~=~ \left(1 \,-\, \frac{r}{\epsilon}\right)^t, ~~r \in [0, \epsilon]
\end{align*}
for some $t \geq 1$.
We have $\,\lrej(\alpha \epsilon) \,=\, (1 - \alpha)^t\,$ and
\begin{align*}
\mathrm{d} \lrej(\alpha \epsilon) ~=\, -t \,(1 \,-\, \alpha)^{t - 1} \,\mathrm{d} \alpha.
\end{align*}
Substituting the above into Eq. (\ref{eq:copy_total_robust_loss_simplified}) gives the total robust loss
\begin{align*}
\tlrej_\epsilon(f; \lrej) ~=~ t \int_0^1 s(\alpha) \,(1 \,-\, \alpha)^{t - 1} \,\mathrm{d} \alpha.
\end{align*}
For the special case $t = 1$, this reduces to the area under the robust error curve $\,\int_0^1 s(\alpha) \,\mathrm{d} \alpha$. \\
For the special case $t = 2$, this reduces to $\,2 \int_0^1 s(\alpha) \,(1 - \alpha) \,\mathrm{d} \alpha~~$ (and so on for larger $t$). \\
In our experiments, we calculate the total robust loss for $t \in \{1, 2, 3, 4\}$.
%
Since we calculate the robust error curve only for a finite set of $\alpha$ values, we approximate the above integrals using the trapezoidal rule. We use finely-spaced $\alpha$ values in $[0, 1]$ with a spacing of $0.01$, and use linear interpolation to obtain intermediate (missing) values of the robust error curve.

\subsection{Step Rejection loss}
Recall that the step rejection loss is defined as
\begin{align*}
\lrej(r) ~=~ \ind\{r \,\leq\, \alpha_0\epsilon\}, ~~r \in [0, \epsilon]
\end{align*}
for some $\alpha_0 \in [0, 1]$. In this case, there are two ways to calculate the total robust loss, both leading to the same result.

\mypara{Approach 1}
For the step rejection loss, the derivative can be defined by appealing to the Dirac delta function $\delta(x)$. Specifically, $\lrej(\alpha \epsilon) = \ind\{\alpha \,\leq\, \alpha_0\}$ and
\begin{align*}
\mathrm{d} \lrej(\alpha \epsilon) ~= \,-\,\delta(\alpha - \alpha_0) \,\mathrm{d} \alpha,
\end{align*}
where the negative sign arises because the step loss drops from $1$ to $0$ at $\alpha = \alpha_0$.
Substituting the above into Eq. (\ref{eq:copy_total_robust_loss_simplified}) gives the total robust loss
\begin{align*}
\tlrej_\epsilon(f; \lrej) ~=~ \int_0^1 s(\alpha) \,\delta(\alpha - \alpha_0) \,\mathrm{d} \alpha ~=~ s(\alpha_0) ~=~ \rerej_{\epsilon}(f, \alpha_0).
\end{align*}

\mypara{Approach 2}
We directly use the definition of the total robust loss in Eq. (\ref{eq:copy_total_robust_loss}), which for the step rejection loss is
\begin{align*}
    \tlrej_\epsilon(f; \lrej) ~&=~ \rerej_{\epsilon}(f, 0) ~+~ \int_0^1 \ind\{\alpha \,\leq\, \alpha_0\}
    ~\mathrm{d} s(\alpha) \\
    &=~ s(0) ~+~ \int_0^{\alpha_0} 1 \,\mathrm{d} s(\alpha) ~=~ s(0) ~+~ s(\alpha_0) ~-~ s(0) \\
    &=~ s(\alpha_0) ~=~ \rerej_{\epsilon}(f, \alpha_0).
\end{align*}
For the step rejection loss, the total robust loss is equal to the value of the robust error curve at the point $\alpha = \alpha_0$.

\section{Experimental Setup}
\label{sec:detailed-setup}

We ran all our experiments with PyTorch and NVDIA GeForce RTX 2080Ti GPUs. We ran all the experiments once with fixed random seeds. 
The code for our work can be found at \url{https://github.com/jfc43/stratified-adv-rej}.

\subsection{Datasets}
\noindent{\bf MNIST. }
MNIST~\cite{lecun1998mnist} is a large dataset of handwritten digits with 10 classes. Each digit has 5,500 training images and 1,000 test images. Each image is 28$\times$28 grayscale. We normalize the range of pixel values to [0,1].

\noindent{\bf SVHN. }
SVHN~\cite{netzer2011reading} is a real-world image dataset of 32$\times$32 color images with 10 classes (one for each digit). It is obtained from house numbers in Google Street View images. The training set has 73,257 images and the original test set has 26,032 images. We use the first 10,000 images from the original test set as the test set for our experiments. We normalize the range of pixel values to [0,1].

\noindent{\bf CIFAR-10. } 
CIFAR-10~\cite{krizhevsky2009learning} is a dataset of 32$\times$32 color images with ten classes, each consisting of 5,000 training images and 1,000 test images. The classes correspond to categories such as dogs, frogs, ships, trucks, etc. We normalize the range of pixel values to [0,1].

\subsection{Baseline Methods}
\label{sec:baseline-details}

 We consider the following five baseline methods: (1) AT+CR: adversarial training (AT)~\cite{madry2018towards} with Confidence-based Rejection; (2)TRADES+CR: TRADES~\cite{zhang2019theoretically} with Confidence-based Rejection; (3) CCAT: confidence-calibrated adversarial training~\cite{stutz2020ccat}; (4) RCD: robust classification with detection~\cite{sheikholeslami2021provably}; (5) ATRR: adversarial training with rectified rejection~\cite{pang2022two}. We provide their training details below. 
 
\mypara{AT+CR. } We use the standard adversarial training (AT) proposed in \cite{madry2018towards} to train the base model. On MNIST, we use LetNet network architecture~\cite{CunBDHHHJ89} and train the network for 100 epochs with a batch size of 128. We use standard stochastic gradient descent (SGD) starting with a learning rate of 0.1. The learning rate is multiplied by 0.95 after each epoch. We use a momentum of 0.9 and do not use weight decay for SGD. We use the PGD attack to generate adversarial training examples with $\epsilon=0.3$, a step size of 0.01, 40 steps and a random start. On SVHN, we use ResNet-20 network architecture~\cite{HeZRS16} and train the network for 200 epochs with a batch size of 128. We use standard SGD starting with a learning rate of 0.1. The learning rate is multiplied by 0.95 after each epoch. We use a momentum of 0.9 and do not use weight decay for SGD. We use the PGD attack to generate adversarial training examples with $\epsilon=\frac{8}{255}$, a step size of $\frac{2}{255}$, 10 steps and a random start. On CIFAR-10, we use ResNet-20 network architecture and train the network following the suggestions in~\cite{pang2021bag}. Specifically, we train the network for 110 epochs with a batch size of 128 using SGD with Nesterov momentum and learning rate schedule. We set momentum 0.9 and $\ell_2$ weight decay with a coefficient of $5\times 10^{-4}$. The initial learning rate is 0.1 and it decreases by 0.1 at 100 and 105 epoch respectively. We augment the training images using random crop and random horizontal flip. We use the PGD attack to generate adversarial training examples with $\epsilon=\frac{8}{255}$, a step size of $\frac{2}{255}$, 10 steps and a random start. 

 \mypara{TRADES+CR. } TRADES is an adversarial training method proposed in~\cite{zhang2019theoretically}. We follow their original setup to train the models on MNIST and CIFAR-10. The setup for SVHN is the same as that for CIFAR-10.
 
\mypara{CCAT. } We follow the original training settings for CCAT in~\cite{stutz2020ccat} and train models on MNIST, SVHN and CIFAR-10 using standard stochastic gradient descent (SGD). On all datasets, we set $\rho=10$. On MNIST, we use LetNet network architecture and train the network for 100 epochs with a batch size of 100 and a learning rate of 0.1. On SVHN, we use ResNet-20 network architecture and train the network for 200 epochs with a batch size of 100 and a learning rate of 0.1. On CIFAR-10, we use ResNet-20 network architecture and train the network for 200 epochs with a batch size of 100 and a learning rate of 0.075. We augment the training images using random crop and random horizontal flip. On all datasets, we use learning rate schedule and the learning rate is multiplied by 0.95 after each epoch. We use a momentum of 0.9 and do not use weight decay for SGD. We use the PGD attack with backtracking to generate adversarial training examples: we use a learning rate of 0.005, a momentum of 0.9, a learning rate factor of 1.5, 40 steps and a random start. We randomly switch between random initialization and zero initialization. We train on 50\% clean and 50\% adversarial examples per batch. 

\mypara{RCD. } We consider a training loss proposed in \cite{sheikholeslami2021provably} (see their Equation 14) to train a neural network for (k+1)-class classification with the (k+1)-th class dedicated to the detection task. \cite{sheikholeslami2021provably} use interval bound propagation (IBP) technique~\cite{GowalDSBQUAMK19} to bound the training loss and then train the network by minimizing the tractable upper bound on the training loss (see their Equation 18) to achieve verified robustness. Since we don't consider verified robustness in our paper, we use the PGD attack to solve the inner maximization problems of the training objective instead. For readers' convenience, we describe their training objective here. Suppose the logits of the network is $\,\widetilde{\mathbf{h}}(\xb \semic \bftheta) = [\widetilde{h}_1(\xb \semic \bftheta), \cdots, \widetilde{h}_{k+1}(\xb \semic \bftheta)]$ and the softmax output of the network is $\,\mathbf{h}(\xb \semic \bftheta) = [h_1(\xb \semic \bftheta), \cdots, h_{k+1}(\xb \semic \bftheta)]$. The softmax output of the network is obtained by applying the softmax function to the logits. Then the training objective is
\begin{align}
\min_{\bftheta} ~ \expec_{(\bfx, y) \sim \calD}\bigg[ \max_{\bfx^\prime \,\in\, \nei(\bfx, \epsilon)}   - \log[h_y(\bfx^\prime \semic \bftheta)] 
~+~ \lambda_1 \,L_{\text{robust}}^{\text{abstain}}(\bfx, y; \bftheta)
~-~ \lambda_2 \,\log[h_y(\bfx \semic \bftheta)] 
\bigg].
\end{align} 
where 
\begin{align}
    L_{\text{robust}}^{\text{abstain}}(\bfx, y; \bftheta) = \max_{\bfx^\prime \,\in\, \nei(\bfx, \epsilon)} \min \bigg\{-\log(\frac{e^{\widetilde{h}_y(\bfx^\prime \semic \bftheta)}}{\sum_{i\in \mathcal{I}\setminus \{k+1\}} e^{\widetilde{h}_i(\bfx^\prime \semic \bftheta)}}), -\log(\frac{e^{\widetilde{h}_{k+1}(\bfx^\prime \semic \bftheta)}}{\sum_{i\in \mathcal{I}\setminus \{y\}} e^{\widetilde{h}_i(\bfx^\prime \semic \bftheta)}}) \bigg\},
\end{align}
and $\mathcal{I}=\{ 1, 2, \dots, k+1 \}$. \cite{sheikholeslami2021provably} always set $\lambda_1=1$ and to keep high clean accuracy, they set a larger $\lambda_2$ (e.g. they set $\lambda_2=2$ on MNIST and set $\lambda_2=2.9$ on CIFAR-10) since it is hard to get high clean accuracy when using the IBP technique to train models. Since we don't use the IBP technique, we simply set $\lambda_1=1$ and $\lambda_2=1$ (our results show that setting $\lambda_2=1$ leads to better results than that of setting $\lambda_2 \geq 2$). We train models on MNIST, SVHN, and CIFAR-10 using standard stochastic gradient descent (SGD). We split each training batch into two sub-batches of equal size, and use the first sub-batch for the last two loss terms in the training objective and use the second sub-batch for the first loss term in the training objective. On MNIST, we use LetNet network architecture and train the network for 100 epochs with a batch size of 128. We use standard SGD starting with a learning rate of 0.1. The learning rate is multiplied by 0.95 after each epoch. We use a momentum of 0.9 and do not use weight decay for SGD. We use the PGD attack to solve the inner maximization problems with $\epsilon=0.3$, a step size of 0.01, 40 steps and a random start. On SVHN, we use ResNet-20 network architecture and train the network for 200 epochs with a batch size of 128. We use standard SGD starting with a learning rate of 0.05. The learning rate is multiplied by 0.95 after each epoch. We use a momentum of 0.9 and do not use weightf decay for SGD. We use the PGD attack to solve the inner maximization problems with $\epsilon=\frac{8}{255}$, a step size of $\frac{2}{255}$, 10 steps and a random start. On CIFAR-10, we use ResNet-20 network architecture and train the network for 110 epochs with a batch size of 128 using SGD with Nesterov momentum and learning rate schedule. We set momentum 0.9 and $\ell_2$ weight decay with a coefficient of $5\times 10^{-4}$. The initial learning rate is 0.1 and it decreases by 0.1 at 100 and 105 epoch respectively. We augment the training images using random crop and random horizontal flip. We use the PGD attack to solve the inner maximization problems with $\epsilon=\frac{8}{255}$, a step size of $\frac{2}{255}$, 10 steps and a random start. 

\mypara{ATRR. } We follow the original training settings for ATRR in~\cite{pang2022two} and train models on MNIST, SVHN and CIFAR-10 using standard stochastic gradient descent (SGD). On MNIST, we use LetNet network architecture and train the network for 100 epochs with a batch size of 128. We use standard SGD starting with a learning rate of 0.1. The learning rate is multiplied by 0.95 after each epoch. We use a momentum of 0.9 and do not use weight decay for SGD. We use the PGD attack to generate adversarial training examples with $\epsilon=0.3$, a step size of 0.01, 40 steps and a random start. On SVHN, we use ResNet-20 network architecture and train the network for 200 epochs with a batch size of 128. We use standard SGD starting with a learning rate of 0.1. The learning rate is multiplied by 0.95 after each epoch. We use a momentum of 0.9 and do not use weight decay for SGD. We use the PGD attack to generate adversarial training examples with $\epsilon=\frac{8}{255}$, a step size of $\frac{2}{255}$, 10 steps and a random start. On CIFAR-10, we use ResNet-20 network architecture and train the network for 110 epochs with a batch size of 128 using SGD with Nesterov momentum and learning rate schedule. We set momentum 0.9 and $\ell_2$ weight decay with a coefficient of $5\times 10^{-4}$. The initial learning rate is 0.1 and it decreases by 0.1 at 100 and 105 epoch respectively. We augment the training images using random crop and random horizontal flip. We use the PGD attack to generate adversarial training examples with $\epsilon=\frac{8}{255}$, a step size of $\frac{2}{255}$, 10 steps and a random start.

\section{Designing Adaptive Attacks}
\label{sec:app_adaptive_attacks}
To compute the robust accuracy with rejection $\rarej_{\epsilon}(f, \alpha)$ for a given $\epsilon > 0$ and $\alpha \in [0, 1]$, we need to generate two adversarial examples $\bfx^\prime \in \nei(\bfx, \alpha\epsilon)$ and $\bfx'' \in \nei(\bfx, \epsilon)$ for each clean input $(\bfx, y)$. We call the attack for generating $\bfx^\prime \in \nei(\bfx, \alpha\epsilon)$ the inner attack, and the attack for generating $\bfx'' \in \nei(\bfx, \epsilon)$ the outer attack. 

For both the inner attack and outer attack, we use an ensemble of attacks and report the worst-case robustness. For the inner attack, if any of the attacks in the inner-attack ensemble finds an adversarial example $\bfx^\prime \in \nei(\bfx, \alpha\epsilon)$ that is rejected by the model, then we consider the inner attack to be successful on the clean input $(\bfx, y)$. For the outer attack, if any of the attacks in the outer-attack ensemble finds an adversarial example $\bfx'' \in \nei(\bfx, \epsilon)$ that is accepted and misclassified by the model, then we consider the outer attack to be successful on the clean input $(\bfx, y)$.  

In this section, we design adaptive attacks to generate $\bfx^\prime$ and $\bfx''$ for the different methods using the same underlying principles to ensure fair comparison. 
The following ideas are applied to design the attack loss functions for all the methods.
Whenever the attack objective is to maximize a probability (or a probability-like) term $p \in [0, 1]$, we use the loss function $\log(p) \in  (-\infty, 0]$. 
Similarly, the loss function $-\log(p) \in [0, \infty)$ is used to minimize a probability (or a probability-like) term $p$.

We first introduce the attack objectives for all the methods below, and then discuss how we solve the attack objectives. 

\subsection{Attack Objectives}
\label{sec:attack_obj}
In this section, we describe the attack objectives used to generate the adversarial examples $\bfx^\prime \in \nei(\bfx, \alpha\epsilon)$ (\ie the inner attack) and the adversarial examples $\bfx'' \in \nei(\bfx, \epsilon)$ (\ie the outer attack) from a clean input $(\bfx, y)$ for the different methods compared. The goal of the adversary to generate $\bfx^\prime$ is to make the defense method reject $\bfx^\prime$. The goal of the adversary to generate $\bfx''$ is to make the defense method accept and incorrectly classify $\bfx''$ into a class other than $y$. We next discuss the inner and outer attack objectives for the different methods considered.

\mypara{Confidence-based Rejection. } The methods AT+CR, TRADES+CR and CCAT~\cite{stutz2020ccat} use the classifier's confidence (\ie maximum softmax probability) as the score for rejection. Suppose the softmax output of the classifier is $\mathbf{h}(\xb \semic \bftheta) = [h_1(\xb \semic \bftheta), \cdots, h_k(\xb \semic \bftheta)]$, then the score for acceptance is the confidence given by $\,h_{\max}(\bfx \semic \bftheta) = \max_j h_j(\bfx \semic \bftheta)$. We use the log-sum-exp approximation of the max function to define a smooth inner attack objective that minimizes the confidence score $h_{\max}(\bfx \semic \bftheta)$. We use the fact that 
\[
\frac{1}{\tau} \,\log\Big( \sum_{i=1}^k e^{\tau s_i} \Big) ~\approx~ \max_{i \in [k]} s_i
\]
where the approximation becomes better for larger values of the temperature constant $\tau > 0$.
We would like to approximate the exact inner attack objective given by
\begin{align}
\label{eq:confidence_attack_exact}
\bfx^\prime ~&=~ \argmax_{\bfz \in \nei(\bfx, \alpha\epsilon)} -\, \log h_{\max}(\bfz \semic \bftheta). 
\end{align}
This attack aims to find an adversarial input $\bfx'$ that minimizes the confidence, thus causing the input to be rejected by methods using confidence-based rejection.
Let $\widetilde{\mathbf{h}}(\xb \semic \bftheta) = [\widetilde{h}_1(\xb \semic \bftheta), \cdots, \widetilde{h}_k(\xb \semic \bftheta)]$ denote the logits corresponding to the classifier prediction, with $\widetilde{h}_{\max}(\bfx \semic \bftheta)$ being the maximum logit.
The attack objective~(\ref{eq:confidence_attack_exact}) can be approximated as
\begin{align}
\label{eq:confidence_attack_approx}
\bfx^\prime ~&=~ \argmax_{\bfz \in \nei(\bfx, \alpha\epsilon)}  -\log h_{\max}(\bfz \semic \bftheta) \nonumber \\
~&=~ \argmax_{\bfz \in \nei(\bfx, \alpha\epsilon)} -\widetilde{h}_{\max}(\bfz \semic \bftheta) ~+~ \log\Big( \sum_{i=1}^k e^{\widetilde{h}_i(\bfz \semic \bftheta)} \Big) \nonumber \\
&\approx~ \argmax_{\bfz \in \nei(\bfx, \alpha\epsilon)} -\frac{1}{\tau} \,\log\Big( \sum_{i=1}^k e^{\tau \widetilde{h}_i(\bfz \semic \bftheta)} \Big) ~+~ \log\Big( \sum_{i=1}^k e^{\widetilde{h}_i(\bfz \semic \bftheta)} \Big).
\end{align}
We name this attack with objective~(\ref{eq:confidence_attack_approx}) Low Confidence Inner Attack (\textbf{LCIA}). In our experiments, we set $\tau=100$.  

For the outer attack, we use the multi-targeted PGD approach. Specifically, for each target label $j\neq y$, we generate an adversarial example $\bfx_j'' \in \nei(\bfx, \epsilon)$ via the following objective:
\begin{align}
\label{eq:adaptive_attack_outer_conf_multitargeted}
 \bfx_j'' ~=~ \argmax_{\bfz \in \nei(\bfx, \epsilon)} \,  \log h_j(\bfz \semic \bftheta).
\end{align}
Then we select the strongest adversarial example $\bfx''$ via: 
\begin{align}
\label{eq:adaptive_attack_outer_conf_multitargeted1}
\bfx'' ~=~ \bfx''_{j^\star} ~~~~s.t.~~~~ j^\star ~=~ \argmax_{ j \in [k] \setminus \{y\} } \,  \log h_j(\bfx_j'' \semic \bftheta).
\end{align}
By solving this objective, the adversary attempts to find an adversarial example that is misclassified with high confidence. The goal of the adversary is to make the selective classifier accept and incorrectly classify the adversarial input. We name this attack with objective~(\ref{eq:adaptive_attack_outer_conf_multitargeted}) High Confidence Misclassification Outer Attack (\textbf{HCMOA}). Note that this HCMOA attack is stronger than the PGD attack with backtracking proposed in \cite{stutz2020ccat} for evaluating the robustness of CCAT (see Appendix~\ref{sec:attacking-ccat} for a comparison). 

\mypara{RCD. } The RCD method~\cite{sheikholeslami2021provably} trains a ($k+1$)-way classifier such that class $k + 1$ is treated as the rejection class.
Suppose the softmax output of the classifier is $\,\mathbf{h}(\xb \semic \bftheta) = [h_1(\xb \semic \bftheta), \cdots, h_{k+1}(\xb \semic \bftheta)]$. For the inner attack, we generate the adversarial example $\bfx^\prime \in \nei(\bfx, \alpha\epsilon)$ using the following objective:
\begin{align}
\label{eq:adaptive_attack_inner_RCD}
\bfx^\prime ~=~ \argmax_{\bfz \in \nei(\bfx, \alpha\epsilon)} \log h_{k+1}(\bfz \semic \bftheta).
\end{align}
The goal of the adversary is to make the method reject the adversarial input by pushing the probability of class $k + 1$ close to $1$. We name this attack with objective~(\ref{eq:adaptive_attack_inner_RCD}) RCD Inner Attack (\textbf{RCDIA}). 

For the outer attack, we use the multi-targeted PGD approach. Specifically, for each target label $j\notin \{y,k+1\}$, we generate the adversarial example $\bfx_j'' \in \nei(\bfx, \epsilon)$ via the following objective:
\begin{align}
\label{eq:adaptive_attack_outer_RCD_multitargeted}
 \bfx_j'' ~=~  \argmax_{\bfz \in \nei(\bfx, \epsilon)} \,\log h_j(\bfz \semic \bftheta).
\end{align}
Then we select the strongest adversarial example $\bfx''$ via: 
\begin{align}
\bfx'' ~=~ \bfx''_{j^\star} ~~~~s.t.~~~~ j^\star ~=~ \argmax_{j\notin \{y,k+1\}} \,\log h_j(\bfx_j'' \semic \bftheta).
\end{align}
Here, the goal of the adversary is to make the selective classifier accept and incorrectly classify the adversarial input, and this objective achieves this by increasing the probability of a class that is different from both the true class $y$ and the rejection class $k + 1$ to be close to 1. We name this attack with objective~(\ref{eq:adaptive_attack_outer_RCD_multitargeted}) RCD Outer Attack (\textbf{RCDOA}). 

\mypara{ATRR. } The ATRR method~\cite{pang2022two} uses a rectified confidence score for rejection. Suppose the softmax output of the classifier is $\mathbf{h}(\xb \semic \bftheta) = [h_1(\xb \semic \bftheta), \cdots, h_k(\xb \semic \bftheta)]$ and the auxiliary function is $A(\bfx;\phi) \in [0,1]$. 
The rectified confidence is defined by \citet{pang2022two} as the product of the auxiliary function and the classifier's confidence, \ie $\,A(\bfx;\phi) \,h_{\max}(\bfx \semic \bftheta)$. For the inner attack, we generate the adversarial example $\bfx^\prime \in \nei(\bfx, \alpha\epsilon)$ using the following objective:
\begin{align}
    \label{eq:adaptive_attack_inner_ATRR}
    \bfx^\prime &~=~ \argmax_{\bfz \in \nei(\bfx, \alpha\epsilon)} -\, \log A(\bfz; \phi) ~-~ \log h_{\max}(\bfz \semic \bftheta) \nonumber \\ 
     &\approx~ \argmax_{\bfz \in \nei(\bfx, \alpha\epsilon)} -\, \log A(\bfz; \phi) ~+~ \Big[\log\Big( \sum_{i=1}^k e^{\widetilde{h}_i(\bfz \semic \bftheta)} \Big) -\frac{1}{\tau} \,\log\Big( \sum_{i=1}^k e^{\tau \widetilde{h}_i(\bfz \semic \bftheta)} \Big) \Big].
\end{align}
Here, we use the log-sum-exp approximation of the max function $\,h_{\max}(\bfz \semic \bftheta) = \max_j h_j(\bfz \semic \bftheta)$ and set $\tau=100$.
The goal is to minimize the rectified confidence by either minimizing the auxiliary function (first term) or the classifier's confidence by pushing its predictions to be close to uniform (second term). By minimizing the rectified confidence score, the adversary attempts to make the ATRR method reject the perturbed input $\bfx^\prime$. We name this attack with objective~(\ref{eq:adaptive_attack_inner_ATRR}) ATRR Inner Attack (\textbf{ATRRIA}). 

For the outer attack, we use the multi-targeted PGD approach. Specifically, for each target label $j\neq y$, we generate the adversarial example $\bfx_j'' \in \nei(\bfx, \epsilon)$ via the following objective:
\begin{align}
\label{eq:adaptive_attack_outer_ATRR_multitargeted}
 \bfx_j'' ~=~ \argmax_{\bfz \in \nei(\bfx, \epsilon)} \,  \log \Big[ A(\bfz; \phi) \, h_j(\bfz \semic \bftheta) \Big].
\end{align}
Then we select the strongest adversarial example $\bfx''$ via: 
\begin{align}
\bfx'' ~=~ \bfx''_{j^\star} ~~~~s.t.~~~~ j^\star ~=~ \argmax_{j \in [k] \setminus \{y\} }  \,  \log \Big[ A(\bfx_j''; \phi) \, h_j(\bfx_j'' \semic \bftheta) \Big].
\end{align}
The goal of the adversary is to make the selective classifier accept and incorrectly classify the adversarial input. 
The objective achieves this by pushing the rectified confidence as well as the predicted probability of a class different from y close to $1$. This ensures that adversarial input is accepted as well as incorrectly classified. We name this attack with objective~(\ref{eq:adaptive_attack_outer_ATRR_multitargeted}) ATRR Outer Attack (\textbf{ATRROA}). 

\mypara{CPR (proposed defense).} The goal of the inner attack for CPR is to find $\bfx^\prime \in \nei(\bfx, \alpha\epsilon)$ such that the base model has different predictions on $\bfx^\prime$ and $T(\bfx^\prime)$, thus ensuring rejection. We consider three adaptive inner attacks that can achieve this goal. The first one is the Low-Confidence Inner Attack (\textbf{LCIA}) introduced in Section~\ref{sec:adaptive_attacks_main}. 
This attack aims to find $\bfx^\prime \in \nei(\bfx, \alpha\epsilon)$ where the base model has low confidence. 
Recall that the mapping $T(\bfx^\prime)$ attempts to minimize the base model's probability on the predicted class $\widehat{y}(\bfx^\prime)$.
So, if the base model has low confidence on $\bfx^\prime$, then it will very likely have even lower probability for $\widehat{y}(\bfx^\prime)$ on $T(\bfx^\prime)$, and thus have different predictions on $T(\bfx^\prime)$ and $\bfx^\prime$.

The second inner attack is a variant of LCIA, named Consistent-Low-Confidence Inner Attack (\textbf{CLCIA}), which attempts to find an adversarial example $\bfx^\prime$ by minimizing the confidence of the base model on both $\bfx^\prime$ and $T(\bfx^\prime)$.   
The CLCIA attack has the following objective:
\begin{align}
\label{eq:adaptive_attack_inner_cpr_bpda}
    \bfx^\prime ~=~ \argmax_{\bfz \in \nei(\bfx, \alpha\epsilon)} & \Big[ -\log h_{\max}(\bfz \semic \bftheta) ~-~ \log h_{\max}(T(\bfz) \semic \bftheta) \Big] \nonumber \\ 
    \approx~ \argmax_{\bfz \in \nei(\bfx, \alpha\epsilon)} & \Big[ -\frac{1}{\tau} \,\log\Big( \sum_{i=1}^k e^{\tau \widetilde{h}_i(\bfz \semic \bftheta)} \Big) ~+~ \log\Big( \sum_{i=1}^k e^{\widetilde{h}_i(\bfz \semic \bftheta)} \Big) \nonumber \\
    & -\frac{1}{\tau} \,\log\Big( \sum_{i=1}^k e^{\tau \widetilde{h}_i(T(\bfz) \semic \bftheta)} \Big) ~+~ \log\Big( \sum_{i=1}^k e^{\widetilde{h}_i(T(\bfz) \semic \bftheta)} \Big) \Big].
\end{align}
We have denoted the logits of the base classifier by $\,\widetilde{\bfh}(\bfx \semic \bftheta) = [\widetilde{h}_1(\bfx \semic \bftheta), \cdots, \widetilde{h}_k(\bfx \semic \bftheta)]$, and we apply the log-sum-exp approximation to the $\max$ function (as before) with $\tau = 100$.
We use the backward pass differentiable approximation (BPDA) method~\cite{athalye2018obfuscated} to solve this objective since $T(\bfx)$ does not have a closed-form expression and is not differentiable.

The third inner attack is a multi-targeted attack, also based on BPDA, 
which considers all possible target classes and attempts to find an $\xb^\prime$ such that the base model has high probability for the target class at $\xb^\prime$ and a low probability for the target class at $T(\xb^\prime)$ (thereby encouraging rejection). The attack objective is: for each target class $j=1,\dots, k$,
\begin{align}
\label{eq:adaptive_attack_inner_cpr_bpda_multitargeted}
    \bfx_j^\prime ~=~ \argmax_{\bfz \in \nei(\bfx, \alpha\epsilon)}  \Big[ \log h_j(\bfz \semic \bftheta) ~-~  \log h_j(T(\bfz) \semic \bftheta) \Big].
\end{align}
Then we select the strongest adversarial example $\bfx^\prime$ via: 
\begin{align}
\bfx^\prime ~=~ \bfx^\prime_{j^\star} ~~~~s.t.~~~~ j^\star ~=~ \argmax_{j \in [k]} \,\Big[ \log h_j(\bfx_j^\prime \semic \bftheta) ~-~  \log h_j(T(\bfx_j^\prime) \semic \bftheta) \Big].
\end{align}
We name this third inner attack Prediction-Disagreement Inner Attack (\textbf{PDIA}).

Given a clean input $(\bfx, y)$, the goal of the outer attack is to find $\bfx'' \in \nei(\bfx, \epsilon)$ such that the base model has consistent wrong predictions on both $\bfx''$ and $T(\bfx'')$. This ensures that $\bfx''$ is accepted and mis-classified. We consider two adaptive outer attacks that can achieve this goal. 
As discussed in Section~\ref{sec:adaptive_attacks_main}, the first outer attack is a multi-targeted attack based on BPDA with the following objective: for each target class $j \in [k] \setminus \{y\}$,
\begin{align}
\label{eq:adaptive_attack_outer_cpr_bpda_multitargeted}
 \bfx_j'' ~=~ \argmax_{\bfz \in \nei(\bfx, \epsilon)} \Big[ \log h_j(\bfz \semic \bftheta) \,+\, \log h_j(T(\bfz) \semic \bftheta) \Big].
\end{align}
Then we select the strongest adversarial example $\bfx''$ via: 
\begin{align}
\bfx'' ~=~ \bfx''_{j^\star} ~~~~s.t.~~~~ j^\star ~=~& \argmax_{j\in [k] \setminus \{y\}} \Big[ \log h_j(\bfx_j'' \semic \bftheta) \,+\, \log h_j(T(\bfx_j'') \semic \bftheta) \Big].
\end{align}
We name this outer attack Consistent High Confidence Misclassification Outer Attack (\textbf{CHCMOA}).

The second outer attack is the High Confidence Misclassification Outer Attack (\textbf{HCMOA}) that was discussed earlier for the methods based on confidence-based rejection.
The attack objective is given in Eqns. (\ref{eq:adaptive_attack_outer_conf_multitargeted}) and (\ref{eq:adaptive_attack_outer_conf_multitargeted1}).
The intuition for this attack is that if the base model has a high-confidence incorrect prediction on $\bfx''$, then it becomes hard for $T$ to change the incorrect prediction.

\subsection{Solving the Attack Objectives}
\label{sec:solving-attack-obj}

We use the PGD with momentum to solve the attack objectives, and use the PGD attack with multiple restarts for evaluating the robustness. Following \citep{stutz2020ccat}, we initialize the perturbation $\bfdelta$ uniformly over directions and norm as follows: 
\begin{align}
    \bfdelta ~=~ u\,\epsilon ~\frac{\bfdelta'}{\| \bfdelta' \|_\infty}, ~\bfdelta' \sim \mathcal{N}(\bfzero, \mathbf{I}), ~u \sim U(0,1)
\end{align}
where $\bfdelta'$ is sampled from the standard Gaussian and $u\in [0,1]$ is sampled from the uniform distribution. We also include zero initialization, \ie $\bfdelta=\bfzero$ as a candidate. We allocate one restart for zero initialization, and multiple restarts for the random initializations. We finally select the perturbation corresponding to the best objective value obtained throughout the optimization. 

When solving the inner attack objectives, we use a momentum factor of $0.9$, $200$ iterations, and $5$ random restarts. The base learning rate (\ie the attack step size) is varied over the set $\{0.1, 0.01, 0.005 \}$ for experiments on MNIST, and over the set $\{0.01, 0.005, 0.001 \}$ for experiments on SVHN and CIFAR-10. We report the worst-case results: for each clean input $\bfx$, if the PGD method with a particular base learning rate can find an $\bfx^\prime$ that is rejected, then we will use this $\bfx^\prime$ as the generated adversarial example and consider the inner attack to be successful.

When solving the outer attack objectives, we use a momentum factor of $0.9$, $200$ iterations, and $5$ random restarts. The base learning rate (\ie the attack step size) is varied over the set $\{0.1, 0.01, 0.005 \}$ for experiments on MNIST, and over the set $\{0.01, 0.005, 0.001 \}$ for experiments on SVHN and CIFAR-10. We report the worst-case results: for each clean input $\bfx$, if the PGD method with a particular base learning rate can find an $\bfx''$ that is accepted and misclassified, then we will use this $\bfx''$ as the generated adversarial example and consider the outer attack to be successful.

\mypara{BPDA approximation. }
Three of the adaptive attacks for CPR (namely CLCIA, PDIA, and CHCMOA) depend on $T(\bfx)$, which does not have a closed-form expression and is not differentiable. 
Therefore, it is not possible to calculate the exact gradient for these attack objectives.
We will use the BPDA approach~\cite{athalye2018obfuscated} to address this challenge, specifically using the straight-through estimator for the gradient.

Recall that $\,T(\bfx) \approx \argmax_{\widetilde{\xb} \in \nei(\xb, \widetilde{\epsilon})} \,\ell_{CE}(\widetilde{\xb}, \widehat{y}(\xb))$, and it is computed using Algorithm~\ref{alg:consistent-prediction-rej}.
For any of these attack objectives, let $\ell(T(\bfx))$ denote the terms dependent on $T(\cdot)$. For instance, $\,\ell(T(\bfx)) = \log h_j(T(\bfx) \semic \bftheta)\,$ for the CHCMOA attack.
Using the chain rule, we can express the gradient of $\ell$ as follows:
\begin{align}
\nabla_{\bfx} \ell(T(\bfx)) ~=~ \bfJ_{T}(\bfx)^{t} \,\nabla_{\bfu} \ell(\bfu) \big\vert_{\bfu = T(\bfx)}
\end{align}
where $\bfJ_{T}(\bfx)^t$ is the transpose of the $d \times d$ Jacobian of $T(\bfx)$.
For a small $\widetilde{\epsilon}$, we make the approximation that $\,T(\bfx) \approx \bfx\,$ during the backward pass, which in turn makes $\bfJ_{T}(\bfx)$ approximately equal to the identity matrix. This gives the following gradient estimate of the BPDA method, which we apply to solve the attack objectives of CLCIA, PDIA, and CHCMOA:
\begin{align}
\label{eq:gradient_bpda}
\nabla_{\bfx} \ell(T(\bfx)) ~=~ \nabla_{\bfu} \ell(\bfu) \big\vert_{\bfu = T(\bfx)}
\end{align}
During the forward pass, we perform an exact computation of $T(x)$ (using Algorithm~\ref{alg:consistent-prediction-rej}), but during the backward pass, we approximate the gradient of the attack objectives using Eqn.(\ref{eq:gradient_bpda}). 
We note that these adaptive attacks are more expensive to run because they require the computation of $T(\bfx)$ during each PGD step.

\subsection{Attack Ensemble}
\label{sec:attack-ensemble}

As discussed earlier, for each defense method, we consider an ensemble of inner and outer attacks and report the worst-case robustness with rejection under these attacks. We list the specific attacks in the ensemble for each defense method below: 

\mypara{Confidence-based Rejection.} The inner-attack ensemble only the includes Low Confidence Inner Attack (LCIA). The outer-attack ensemble includes the AutoAttack~\cite{croce2020reliable} and High Confidence Misclassification Outer Attack (HCMOA).

\mypara{RCD.}  The inner-attack ensemble only includes RCDIA. The outer-attack ensemble includes AutoAttack and RCDOA.

\mypara{ATRR.}  The inner-attack ensemble only includes ATRRIA. The outer-attack ensemble includes AutoAttack and ATRROA.

\mypara{CPR (proposed).} The inner-attack ensemble includes LCIA, CLCIA, and PDIA. The outer-attack ensemble includes AutoAttack, HCMOA and CHCMOA.
By including a number of strong attacks in the ensemble, we have attempted to perform a thorough evaluation of CPR.

\subsection{Evaluating Adaptive Attacks for CPR}
In Section~\ref{sec:results}, we perform an ablation study to compare the strength of the different adaptive inner and outer attacks for CPR. These results are in Table~\ref{tab:cpr-seen-attack-ablation}, and here we discuss the choice of metrics for this evaluation.
The outer attack only affects $\rarej_{\epsilon}(f, 0)$, while the inner attack affects $\rarej_{\epsilon}(f, \alpha)$ for $\alpha > 0$. 
Therefore, for the outer attacks we only need to compare $\rarej_{\epsilon}(f, 0)$. For the inner attacks we compare $\rarej_{\epsilon}(f, 1)$, while fixing the outer attack to be the strongest ensemble outer attack. 
This corresponds to the right end of the robustness curve, and gives a clear idea of the strength of the inner attack.

\section{Additional Experimental Results}
\label{sec:additional-exp-results}

\subsection{Evaluating Traditional Metrics} 
\label{sec:trad-metrics}

We evaluate our method and the baselines on traditional metrics, including accuracy with rejection on clean test inputs, rejection rate on clean test inputs and robust accuracy with detection defined in~\cite{tramer2021detecting}.

The accuracy with rejection on clean test inputs is defined as: $\Pr_{(\xb,y) \sim \calD} \{ f(\xb) = y \,|\, f(\xb) \neq \rejsym \}$.

The rejection rate on clean test inputs is defined as: $\Pr_{(\xb,y) \sim \calD} \{f(\xb) = \rejsym \}$.

The robust accuracy with detection is defined as: 
\begin{align*}
  \rarej_\epsilon(f) ~:=&~ 1-\orerej_\epsilon(f) ~=~ 1 ~- \expec_{(\xb,y) \sim \calD} \bigg[
    \ind\{ f(\xb) \neq y \}  \vee \max_{\xb^\prime \in \nei(\xb, \epsilon)} \ind\big\{ f(\xb^\prime) \not \in \{y, \rejsym\} \big\} \bigg].
\end{align*}
In order to define a single metric that combines the accuracy with rejection and the rejection rate metrics on clean test inputs, we propose to use an F1 score like metric that is based on their harmonic mean:
\begin{align*}
 \frac{2 \,\Pr_{(\xb,y) \sim \calD} \{ f(\xb) = y \,|\, f(\xb) \neq \rejsym \} \,\Pr_{(\xb,y) \sim \calD} \{f(\xb) \neq \rejsym \}}{\Pr_{(\xb,y) \sim \calD} \{ f(\xb) = y \,|\, f(\xb) \neq \rejsym \} \,+\, \Pr_{(\xb,y) \sim \calD} \{f(\xb) \neq \rejsym \}}
\end{align*}

The results for these metrics are given in Table~\ref{tab:eval-trad-metrics}. The results show that our method CPR has comparable performance to the baselines on clean test inputs, and also significantly outperforms the baselines on the robust accuracy with detection. 

\begin{table*}[htb]
    \centering
    \begin{adjustbox}{width=\columnwidth,center}
		\begin{tabular}{l|l|c|c|c|c|c}
			\toprule
			\multirow{2}{0.08\linewidth}{Dataset} &  \multirow{2}{0.08\linewidth}{Method} & \multicolumn{3}{|c|}{Clean Test Inputs} & Under Seen Attacks & Under Unseen Attacks \\  \cline{3-7}
			& & Acc. with Rej. $\uparrow$ & Rej. Rate $\downarrow$ & F1 Score $\uparrow$ & Robust Acc. with Det. $\uparrow$ & Robust Acc. with Det. $\uparrow$ \\ \hline \hline
			\multirow{4}{0.12\linewidth}{MNIST}  
			& AT & 98.81 & 0.00 & 99.40 & 84.70 & 0.00 \\
			& AT+CR & 99.55 & 1.79 & 98.87 & 91.60 & 0.00 \\
            & TRADES & 99.07 & 0.00 & \textbf{99.53} & 89.30 & 0.00 \\
			& TRADES+CR & 99.67 & 1.86 & 98.90 & 94.00 & 0.00 \\
			& CCAT & 99.90 & 1.82 & 99.03 & 83.20 & 75.50 \\
			& RCD & 99.02 & 0.00 & 99.51 & 86.50 & 0.00 \\
			& ATRR & 99.62 & 2.51 & 98.54 & 91.20 & 0.00 \\
			& AT+CPR (Ours)  & 99.60 & 1.99 & 98.80 & \textbf{96.10} & \textbf{90.40} \\ 
            & TRADES+CPR (Ours) & 99.63 & 1.63 & 98.99 & 95.80 & 86.70 \\ \hline 
			\multirow{4}{0.12\linewidth}{{SVHN}}
			& AT & 92.58 & 0.00 & 96.15 & 45.10 & 11.70 \\
			& AT+CR & 96.22 & 8.91 & 93.58 & 46.10 & 11.80 \\
            & TRADES & 92.19 & 0.00 & 95.94 & 52.00 & 12.30 \\
			& TRADES+CR & 95.47 & 9.06 & 93.15 & 52.90 & 12.60 \\
			& CCAT & 99.04 & 7.73 & 95.53 & 45.30 & 5.50 \\
			& RCD & 96.58 & 0.00 & \textbf{98.26} & 33.80 & 9.70 \\
			& ATRR & 96.14 & 8.98 & 93.51 & 44.80 & 11.50 \\
			& AT+CPR (Ours)  & 95.86 & 7.34 & 94.23 & 55.80 & 14.70 \\ 
            & TRADES+CPR (Ours) & 94.96 & 6.56 & 94.20 & \textbf{62.00} & \textbf{18.70} \\ \hline 
			\multirow{4}{0.12\linewidth}{CIFAR-10}
			& AT & 84.84 & 0.00 & 91.80 & 47.60 & 10.80 \\
			& AT+CR & 90.55 & 13.00 & 88.74 & 50.00 & 10.50 \\
            & TRADES & 82.12 & 0.00 & 90.18 & 48.70 & 15.20 \\
			& TRADES+CR & 86.57 & 9.59 & 88.45 & 50.00 & 15.10 \\
			& CCAT & 93.18 & 9.12 & 92.01 & 27.70 & 8.80 \\
			& RCD & 88.13 & 2.07 & \textbf{92.77} & 46.70 & 9.50 \\
			& ATRR & 89.36 & 12.09 & 88.63 & 48.80 & 11.30 \\
			& AT+CPR (Ours)  & 89.05 & 9.57 & 89.74 & 56.70 & 17.10 \\ 
            & TRADES+CPR (Ours) & 86.30 & 9.57 & 88.32 & \textbf{57.10} & \textbf{21.90} \\ 
			\bottomrule
		\end{tabular}
	\end{adjustbox}
	\caption[]{\small Evaluation of traditional metrics (percentages). Top-1 \textbf{boldfaced}. }
	\label{tab:eval-trad-metrics}
\end{table*}

\subsection{Evaluating Robustness Curve}
\label{sec:full-rob-curve-results}

 In Section \ref{sec:results}, we discussed the results of evaluating the robustness curve on the CIFAR-10 dataset. We present the complete results of evaluating the  
 robustness curve on all the datasets under both seen attacks and unseen attacks in Figure~\ref{fig:full-main-results}.
 The observations on MNIST and SVHN are similar to that of CIFAR-10, except that CCAT has much better robustness with rejection at $\alpha=0$ than the other baselines on MNIST under unseen attacks.

\begin{figure*}[htb]
	\centering
	\includegraphics[width=0.31\linewidth]{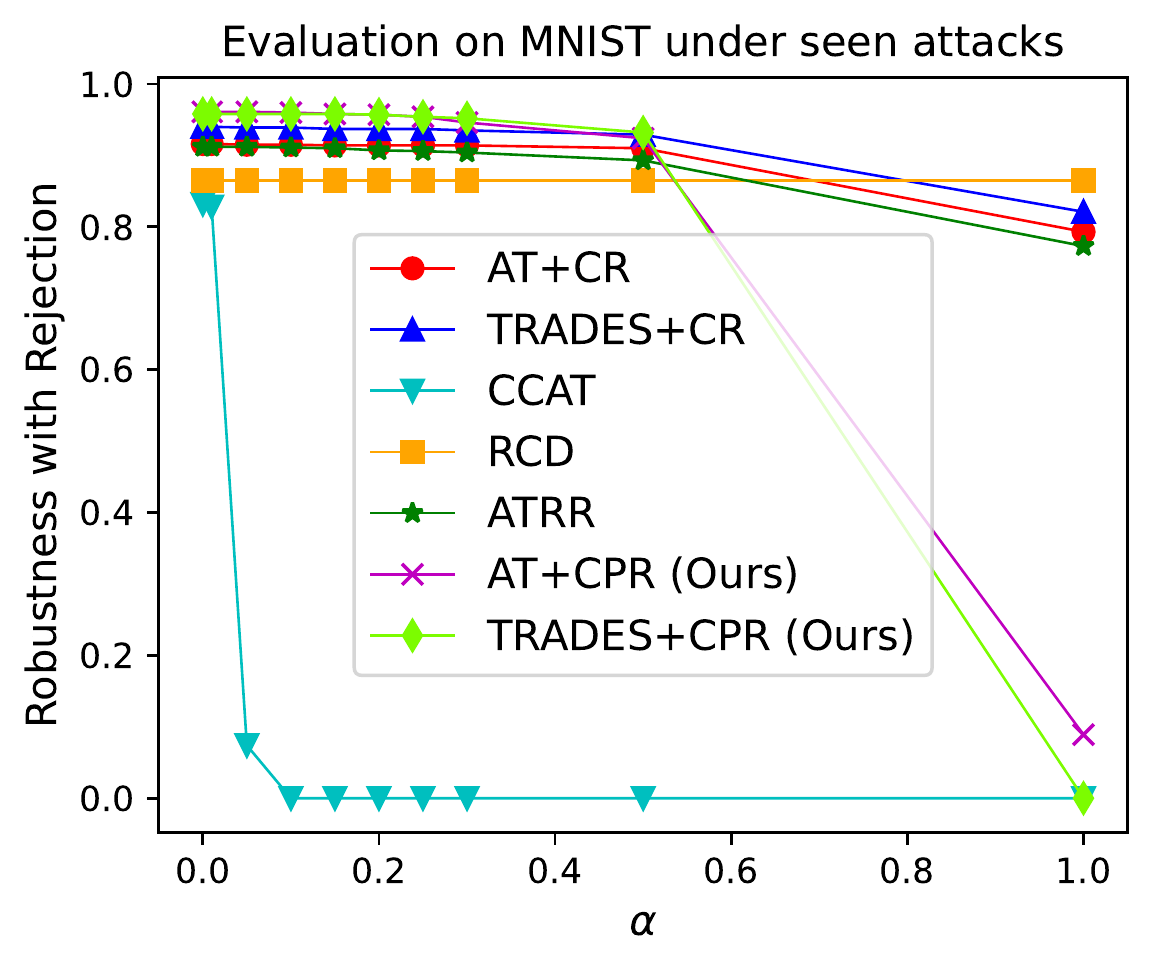}
	\includegraphics[width=0.31\linewidth]{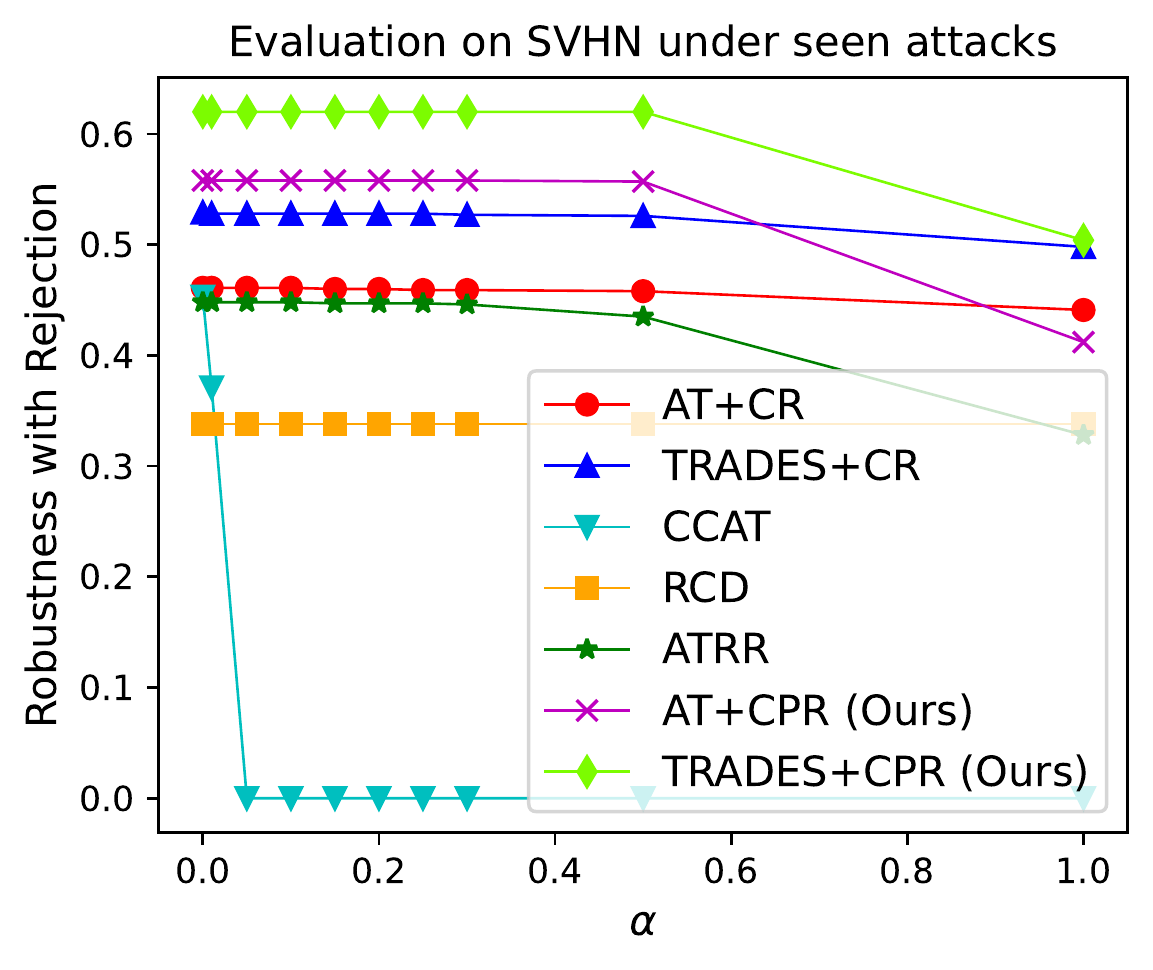}
	\includegraphics[width=0.31\linewidth]{figures/main/cifar10_result_seen_attack.pdf}
	\includegraphics[width=0.31\linewidth]{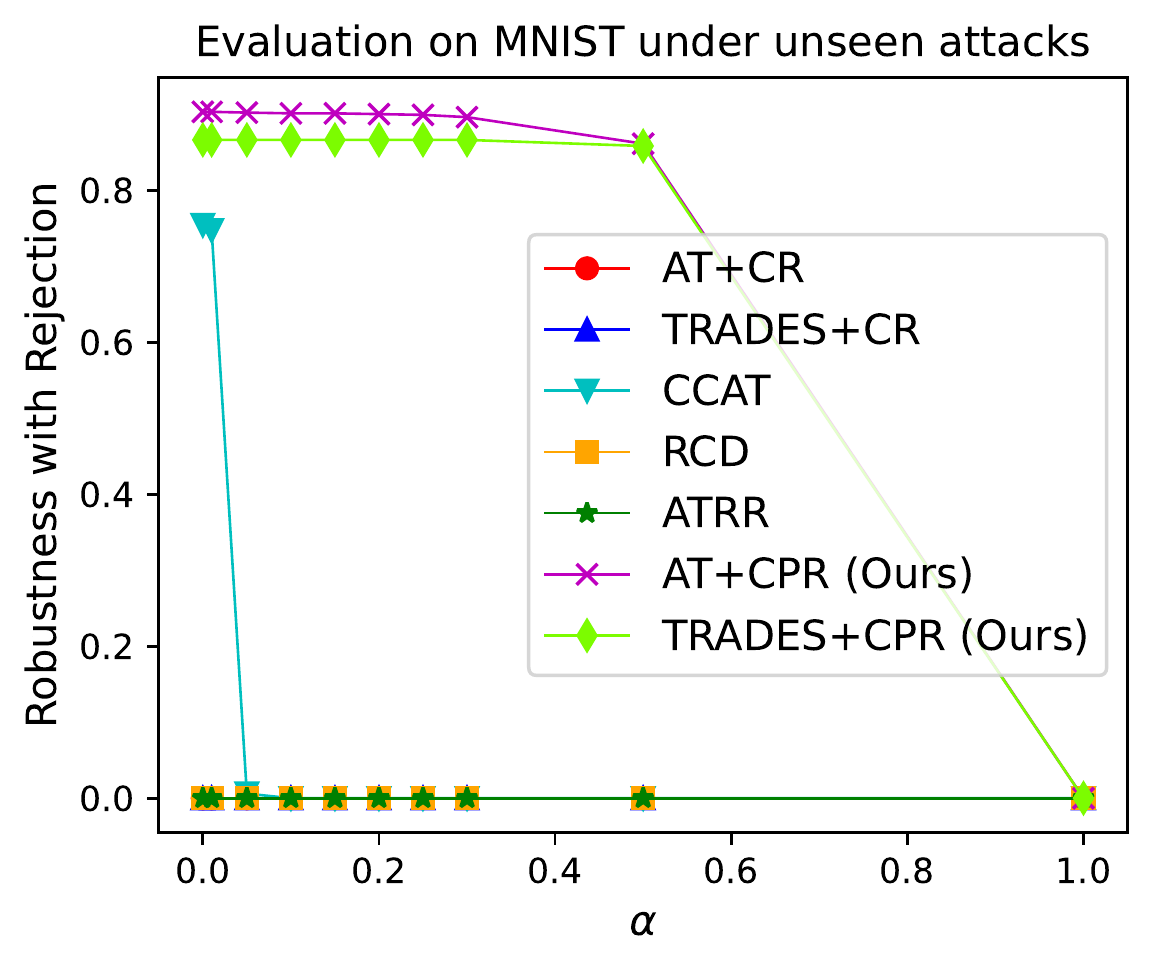}
	\includegraphics[width=0.31\linewidth]{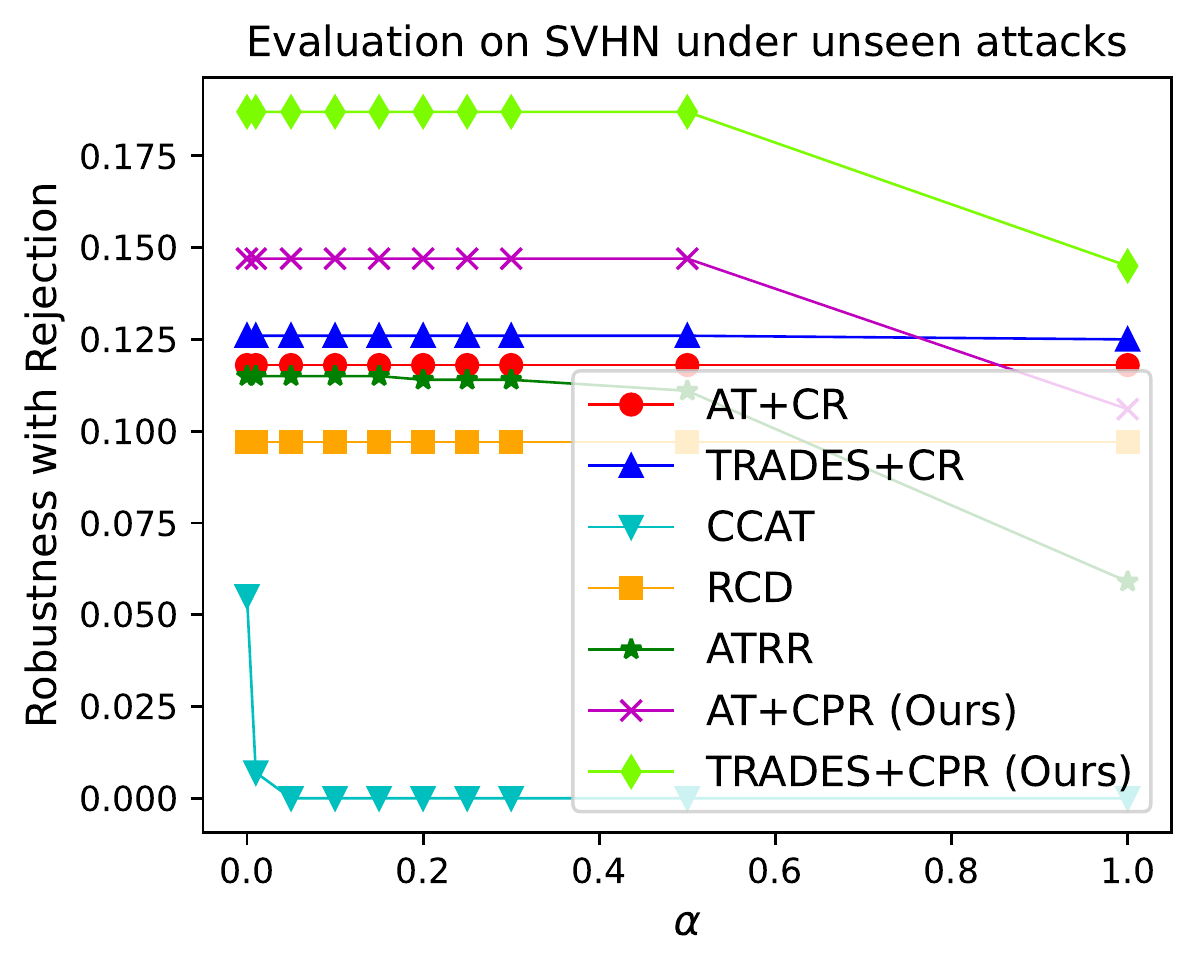}
	\includegraphics[width=0.31\linewidth]{figures/main/cifar10_result_unseen_attack.pdf}
	\caption{\small The robustness curve of our method CPR and the baselines. }
	\label{fig:full-main-results}
\end{figure*}

\subsection{Evaluating Total Robust Loss for More Rejection Loss Functions}

In this section, we evaluate the total robust loss of the different methods for some additional rejection loss functions. Specifically, we additionally consider the step rejection loss function with $\alpha_0\in \{0.01, 0.15, 0.2\}$ and the ramp rejection loss function with $t\in \{ 1, 3 \}$. These results are given in Table~\ref{tab:total-loss-additional-results}. 
Our findings are the same as those discussed in Section~\ref{sec:results}.

\begin{table}[htb]
    \centering
    \begin{adjustbox}{width=\columnwidth,center}
		\begin{tabular}{l|l|c|c|c|c|c|c|c|c|c|c}
			\toprule
			\multirow{3}{0.08\linewidth}{Dataset} &  \multirow{3}{0.12\linewidth}{Method}  &  \multicolumn{5}{c|}{Total Robust Loss under Seen Attacks $\downarrow$} &  \multicolumn{5}{c}{Total Robust Loss under Unseen Attacks $\downarrow$} \\ \cline{3-12}
			  &  & \multicolumn{3}{c|}{Step Rej. Loss} & \multicolumn{2}{c|}{Ramp Rej. Loss} & \multicolumn{3}{c|}{Step Rej. Loss} & \multicolumn{2}{c}{Ramp Rej. Loss} \\ \cline{3-12}
			  &  & $\alpha_0=0.01$  & $\alpha_0=0.15$ & $\alpha_0=0.2$ & $t=1$ & $t=3$ & $\alpha_0=0.01$  & $\alpha_0=0.15$ & $\alpha_0=0.2$ & $t=1$ & $t=3$ \\  \hline \hline
			\multirow{4}{0.08\linewidth}{MNIST} 
            & AT+CR & 0.084 & 0.086 & 0.086 & 0.117 & 0.090 & 1.000 & 1.000 & 1.000 & 1.000 & 1.000 \\
            & TRADES+CR & 0.060 & 0.063 & 0.063 & \textbf{0.095} & \textbf{0.068} & 1.000 & 1.000 & 1.000 & 1.000 & 1.000 \\
            & CCAT & 0.172 & 1.000 & 1.000 & 0.972 & 0.919 & 0.252 & 1.000 & 1.000 & 0.977 & 0.934 \\
            & RCD & 0.135 & 0.135 & 0.135 & 0.135 & 0.135 & 1.000 & 1.000 & 1.000 & 1.000 & 1.000 \\
            & ATRR & 0.088 & 0.090 & 0.093 & 0.131 & 0.099 & 1.000 & 1.000 & 1.000 & 1.000 & 1.000 \\ \cline{2-12}
            & AT+CPR (Ours) & \textbf{0.039} & \textbf{0.042} & \textbf{0.043} & 0.295 & 0.080 & \textbf{0.096} & \textbf{0.098} & \textbf{0.099} & \textbf{0.340} & \textbf{0.136} \\
            & TRADES+CPR (Ours) & 0.042 & 0.042 & 0.043 & 0.292 & 0.078 & 0.133 & 0.133 & 0.133 & 0.353 & 0.162 \\ \hline \hline
			\multirow{4}{0.08\linewidth}{SVHN}
			& AT+CR & 0.539 & 0.540 & 0.540 & 0.545 & 0.541 & 0.882 & 0.882 & 0.882 & 0.882 & 0.882 \\
                & TRADES+CR & 0.472 & 0.472 & 0.472 & 0.480 & 0.473 & 0.874 & 0.874 & 0.874 & 0.874 & 0.874 \\
                & CCAT & 0.629 & 1.000 & 1.000 & 0.988 & 0.967 & 0.993 & 1.000 & 1.000 & 1.000 & 0.999 \\
                & RCD & 0.662 & 0.662 & 0.662 & 0.662 & 0.662 & 0.903 & 0.903 & 0.903 & 0.903 & 0.903 \\
                & ATRR & 0.552 & 0.553 & 0.553 & 0.587 & 0.559 & 0.885 & 0.885 & 0.886 & 0.901 & 0.888 \\ \cline{2-12}
                & AT+CPR (Ours) & 0.442 & 0.442 & 0.442 & 0.479 & 0.447 & 0.853 & 0.853 & 0.853 & 0.863 & 0.854 \\
                & TRADES+CPR (Ours) & \textbf{0.380} & \textbf{0.380} & \textbf{0.380} & \textbf{0.409} & \textbf{0.384} & \textbf{0.813} & \textbf{0.813} & \textbf{0.813} & \textbf{0.823} & \textbf{0.814} \\ \hline \hline
			\multirow{4}{0.08\linewidth}{CIFAR-10}
			& AT+CR & 0.500 & 0.501 & 0.501 & 0.515 & 0.504 & 0.895 & 0.895 & 0.895 & 0.895 & 0.895 \\
                & TRADES+CR & 0.500 & 0.500 & 0.500 & 0.507 & 0.501 & 0.849 & 0.849 & 0.849 & 0.849 & 0.849 \\
                & CCAT & 0.729 & 1.000 & 1.000 & 0.992 & 0.976 & 0.985 & 1.000 & 1.000 & 0.999 & 0.998 \\
                & RCD & 0.533 & 0.533 & 0.533 & 0.533 & 0.533 & 0.905 & 0.905 & 0.905 & 0.905 & 0.905 \\
                & ATRR & 0.513 & 0.513 & 0.513 & 0.524 & 0.516 & 0.887 & 0.887 & 0.887 & 0.887 & 0.887 \\ \cline{2-12}
                & AT+CPR (Ours) & 0.433 & 0.433 & 0.433 & 0.464 & 0.437 & 0.829 & 0.829 & 0.829 & 0.849 & 0.832 \\
                & TRADES+CPR (Ours) & \textbf{0.429} & \textbf{0.429} & \textbf{0.429} & \textbf{0.457} & \textbf{0.433} & \textbf{0.781} & \textbf{0.781} & \textbf{0.781} & \textbf{0.800} & \textbf{0.783} \\
			\bottomrule
		\end{tabular}
	\end{adjustbox}
	\caption[]{\small The total robust losses for different rejection loss functions. The best result is \textbf{boldfaced}. }
	\label{tab:total-loss-additional-results}
\end{table}

\subsection{Comparing AT with Rejection and AT without Rejection}

We compare Adversarial Training (AT)  with Consistent Prediction-based Rejection (CPR) to AT  without rejection. Both of them use the same base model, which is trained by either standard Adversarial Training (AT) or TRADES. Since adversarial training without rejection accepts every input, the robust accuracy with rejection $\rarej_{\epsilon}(f, \alpha)$ is a constant (across $\alpha$), which is equal to the standard adversarial robustness. We use AutoAttack~\cite{croce2020reliable} to evaluate the adversarial robustness of adversarial training without rejection. 

From the results in Figure~\ref{fig:comparing-at-results}, we observe that AT with CPR is usually better (has higher robustness with rejection) than AT without rejection, especially under unseen attacks. 
Under seen attacks, AT without rejection is better than AT with CPR for large $\alpha$ since AT with CPR may reject large perturbations, which is considered to be an error by the robustness with rejection metric for large $\alpha$.  
However, if we allow rejecting large perturbations, AT with CPR is always better than AT without rejection, which suggests that including CPR can help improve the performance of adversarial training. 

\begin{figure}[htb]
	\centering
	\includegraphics[width=0.32\linewidth]{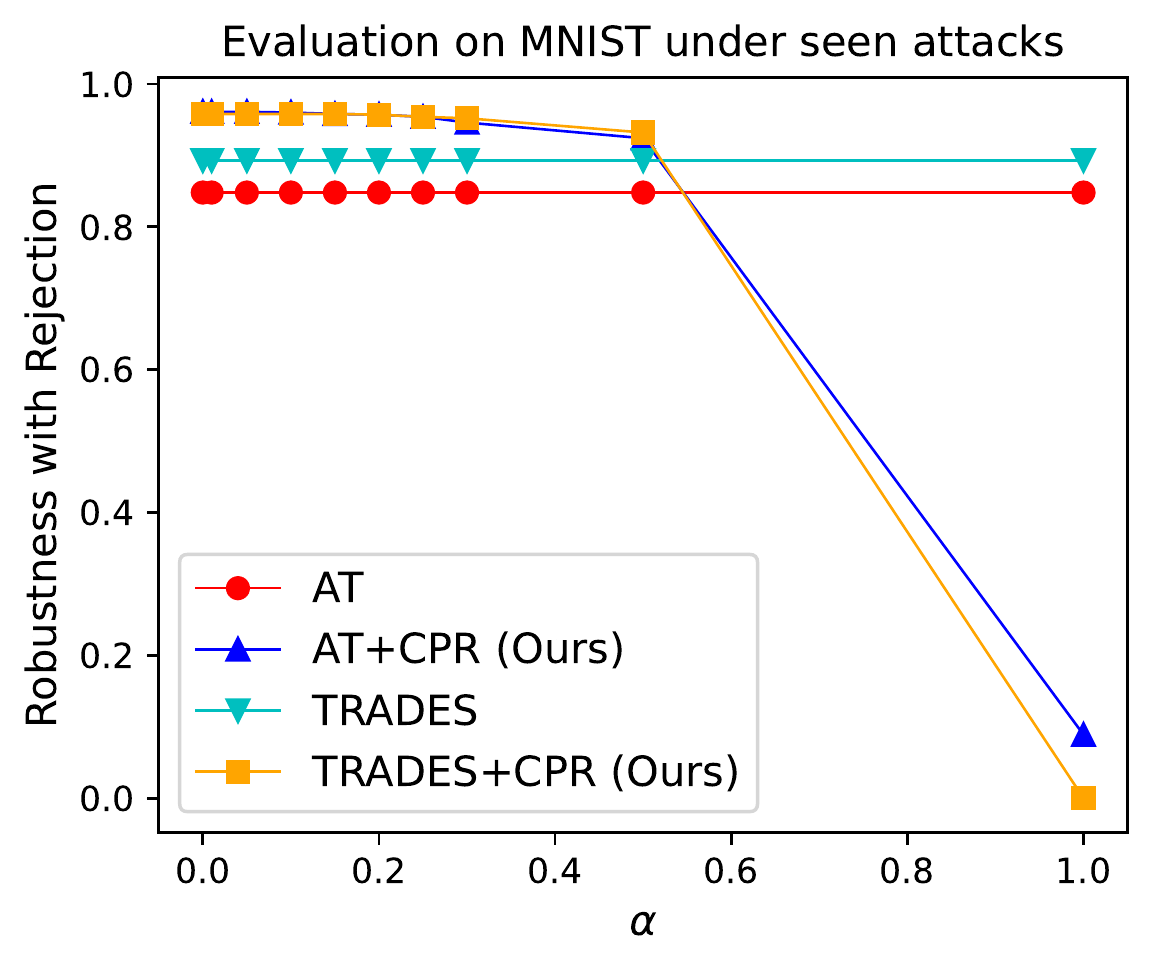}
	\includegraphics[width=0.32\linewidth]{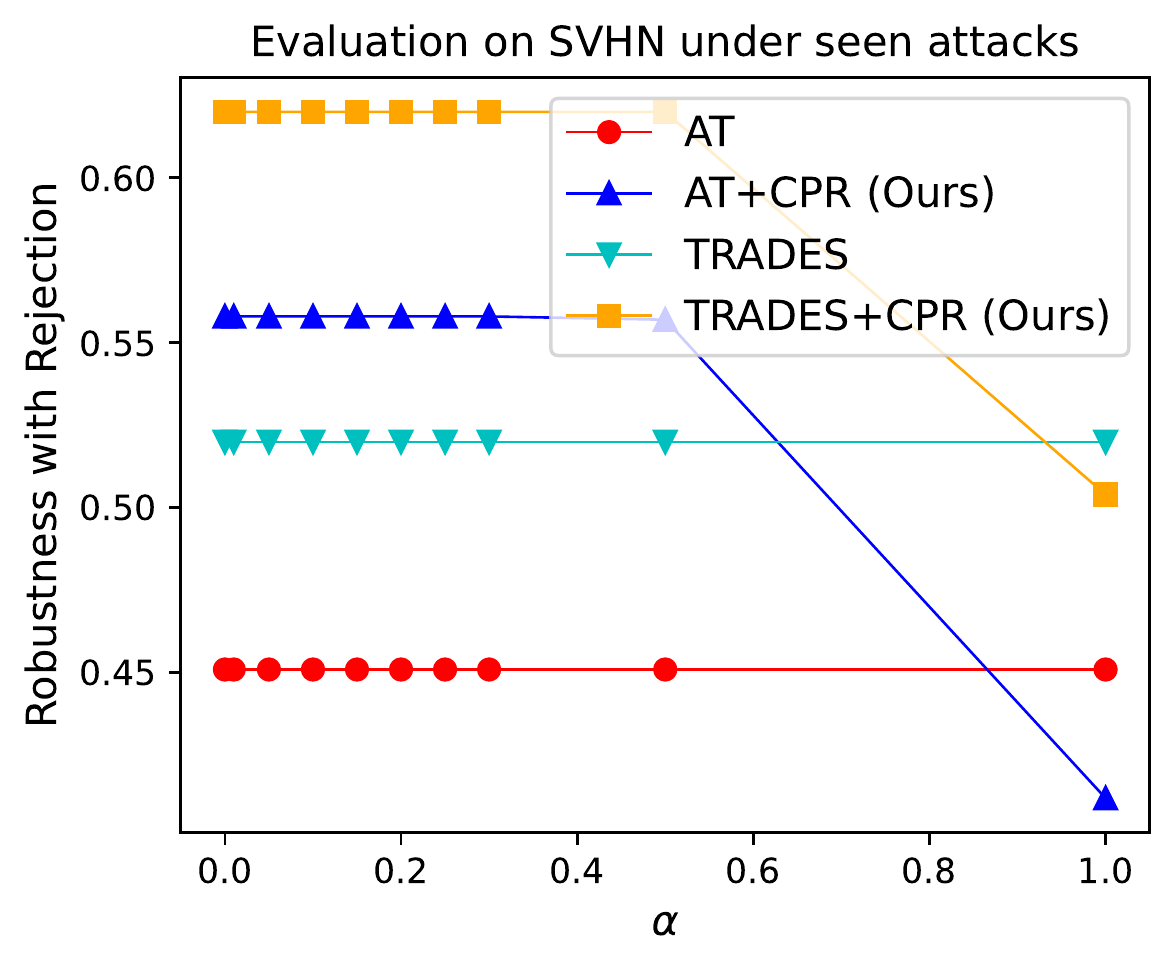}
	\includegraphics[width=0.32\linewidth]{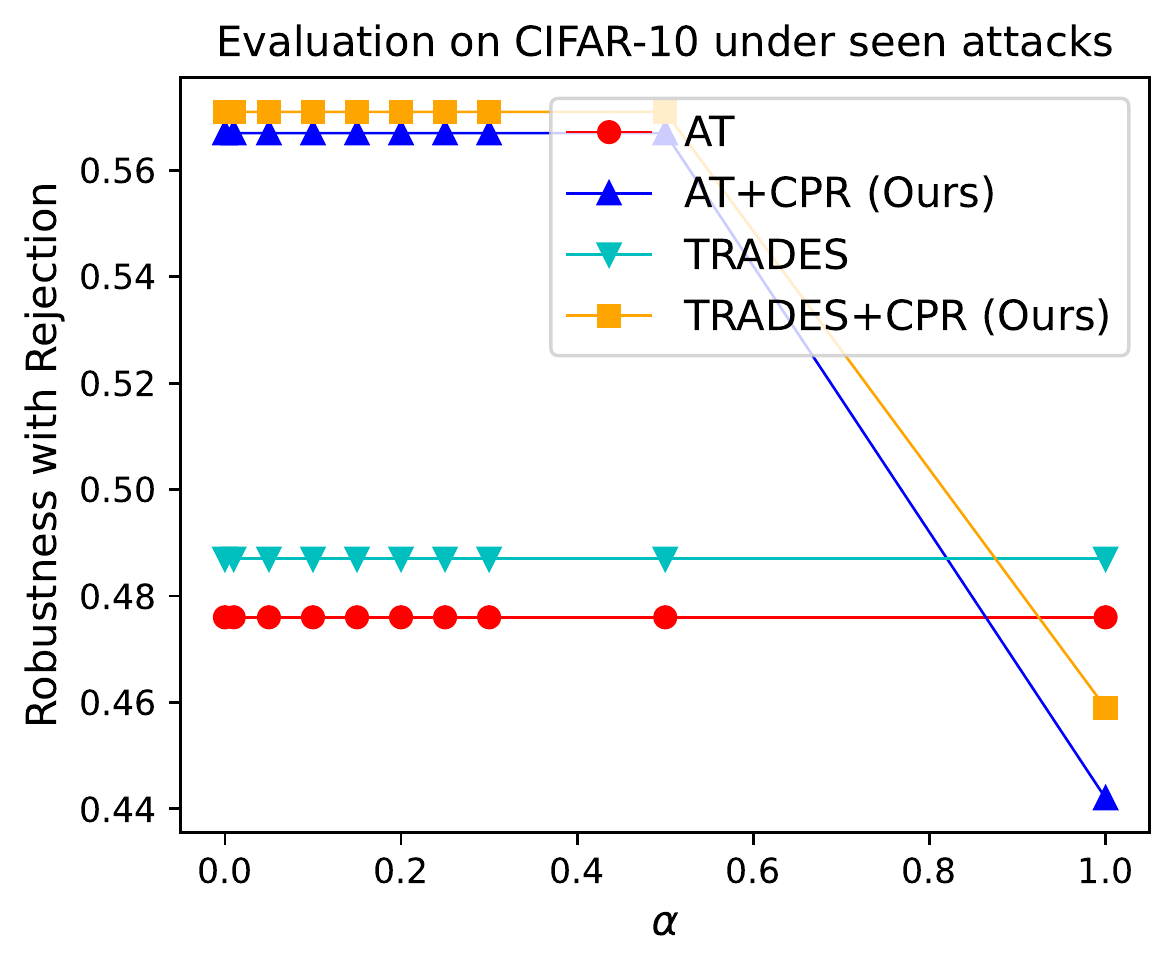}
	\includegraphics[width=0.32\linewidth]{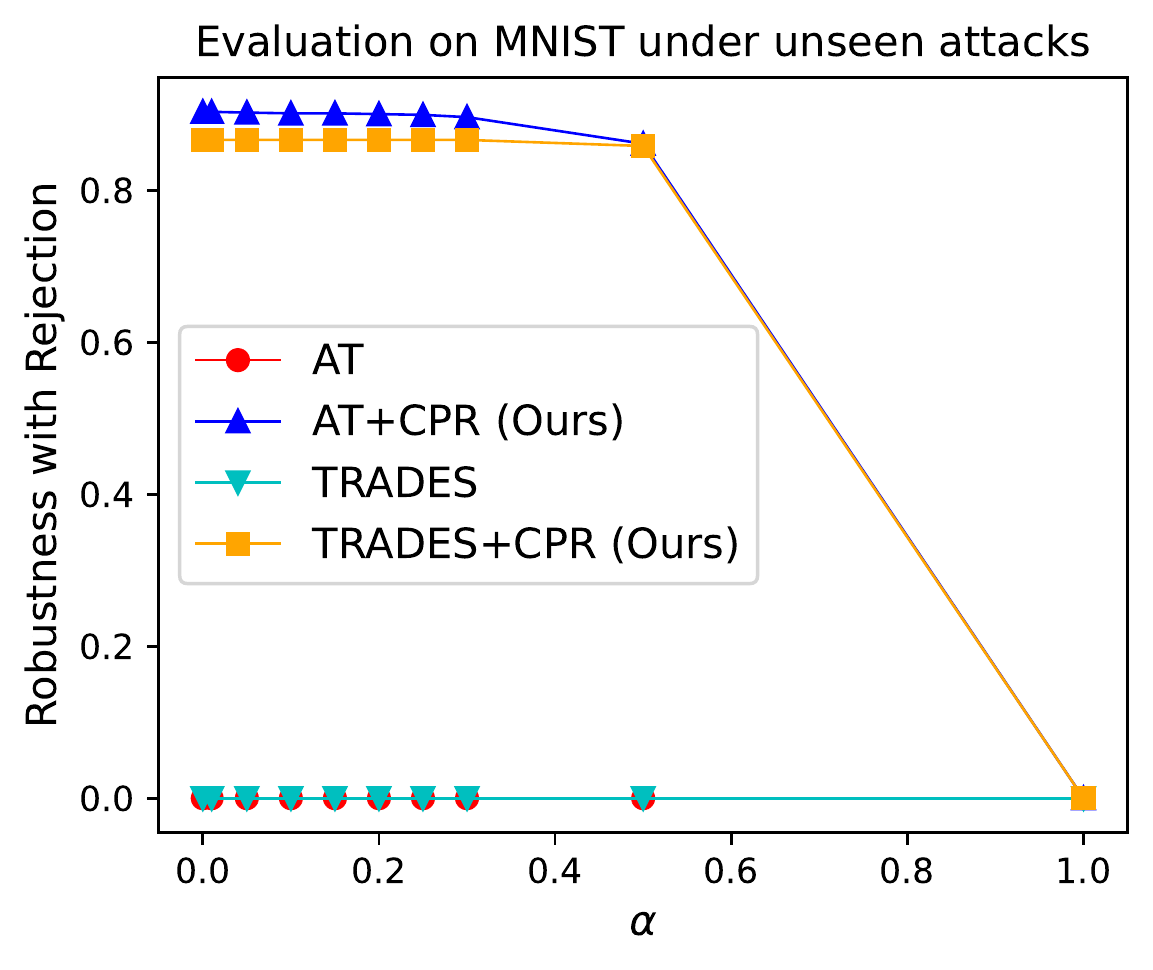}
	\includegraphics[width=0.32\linewidth]{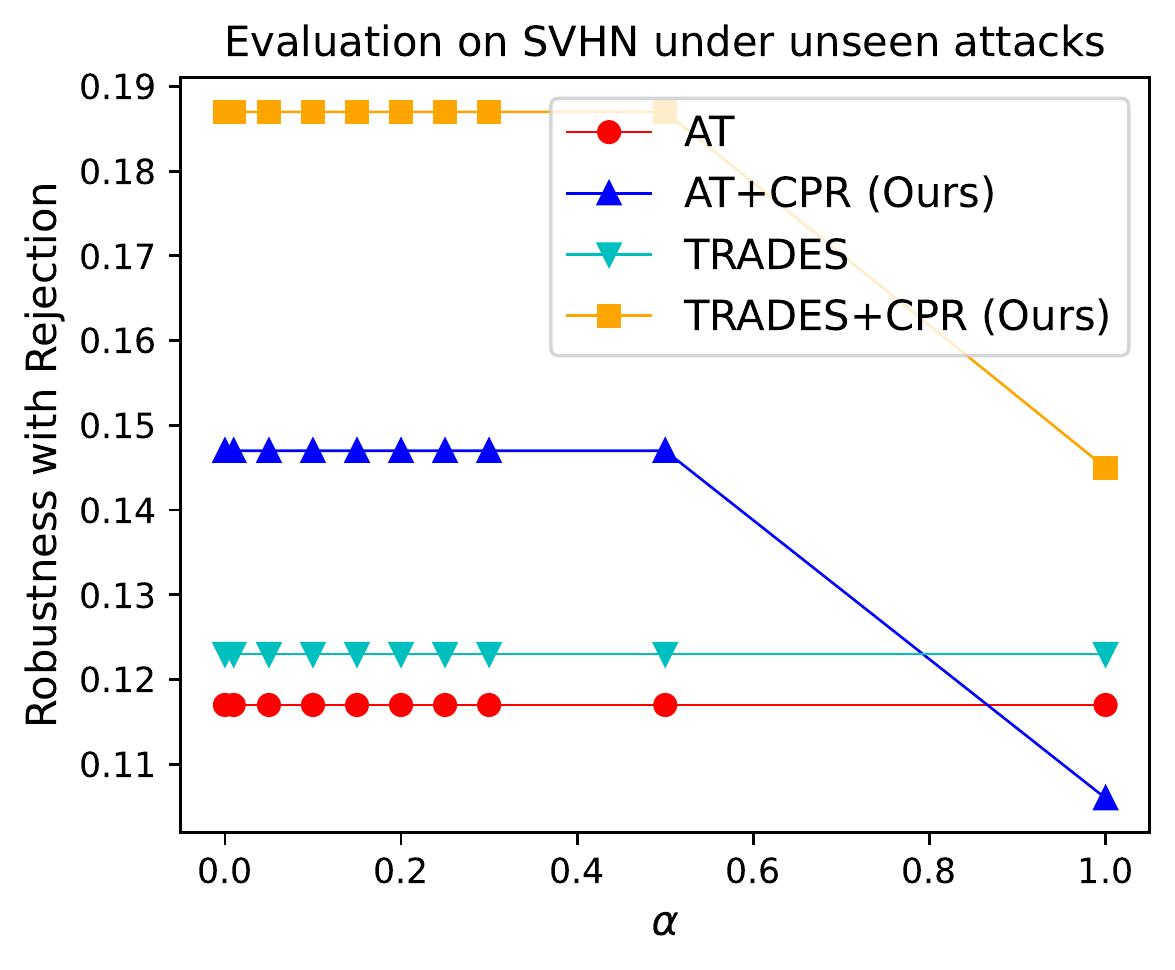}
	\includegraphics[width=0.32\linewidth]{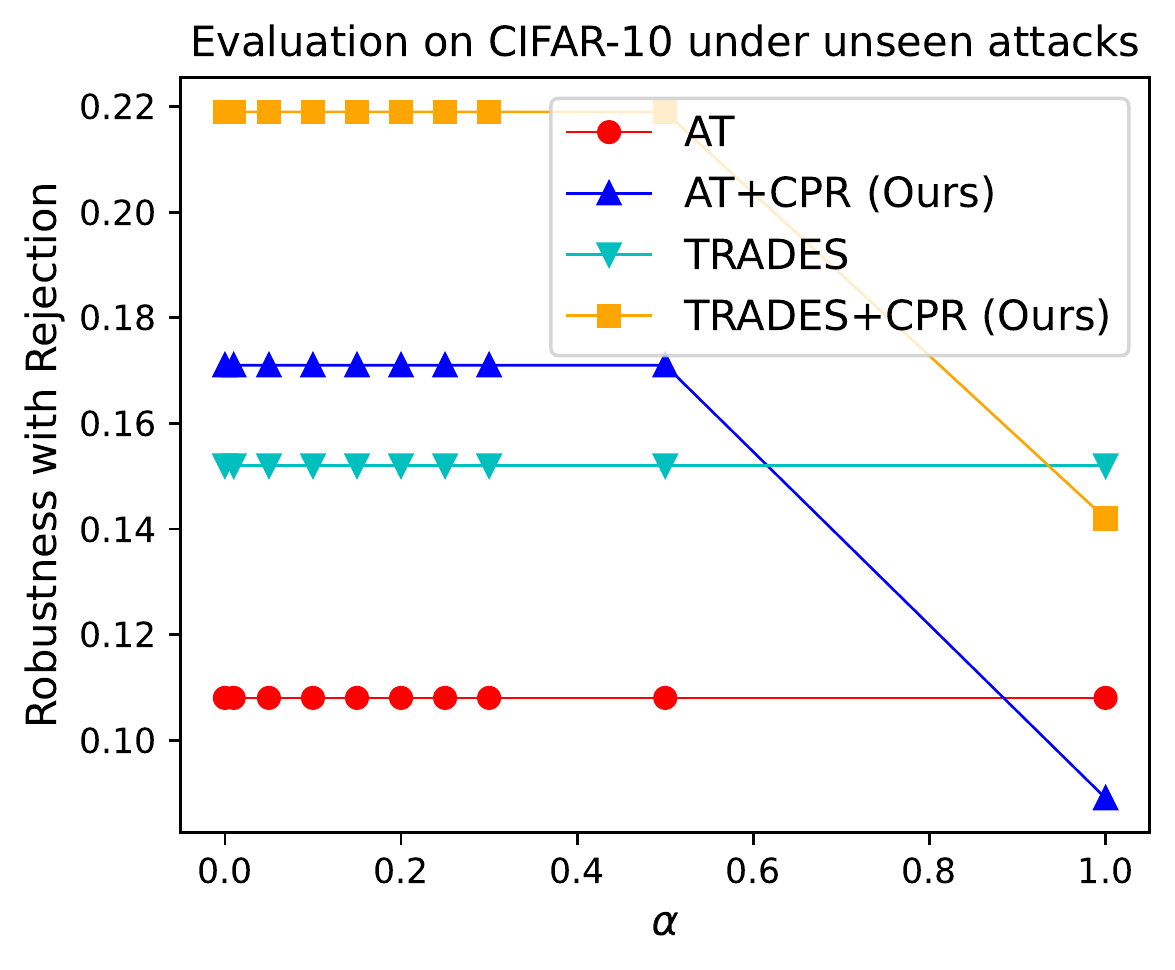}
	\caption{\small Results of comparing the robustness with rejection of AT with CPR and AT without CPR.  }
	\label{fig:comparing-at-results}
\end{figure}

\subsection{Ablation Study for CPR}
\label{sec:cpr-hyper-ablation}
We perform experiments to study the effect of the hyper-parameters $\tilde{\epsilon}$ and $m$ on the proposed method CPR. The results are shown in Figure~\ref{fig:ablation-vary-eps-results} and Figure~\ref{fig:ablation-vary-m-results}. From the results, we can see that larger $\tilde{\epsilon}$ leads to better robustness with rejection at $\alpha=0$. However, it also leads to lower robustness with rejection when $\alpha$ is large, which means CPR rejects more perturbed inputs. Lager $\tilde{\epsilon}$ will also lead to a larger rejection rate on the clean inputs. Besides, larger $m$ also leads to better robustness with rejection at $\alpha=0$, but may lead to lower robustness with rejection when $\alpha$ is large. Note that larger $m$ leads to more computational cost. We don't need to choose very large $m$ since the results show that as we increase $m$, the robustness with rejection at $\alpha=0$ improvement becomes minor. 
 
\begin{figure*}[htb]
	\centering
	\includegraphics[width=0.31\linewidth]{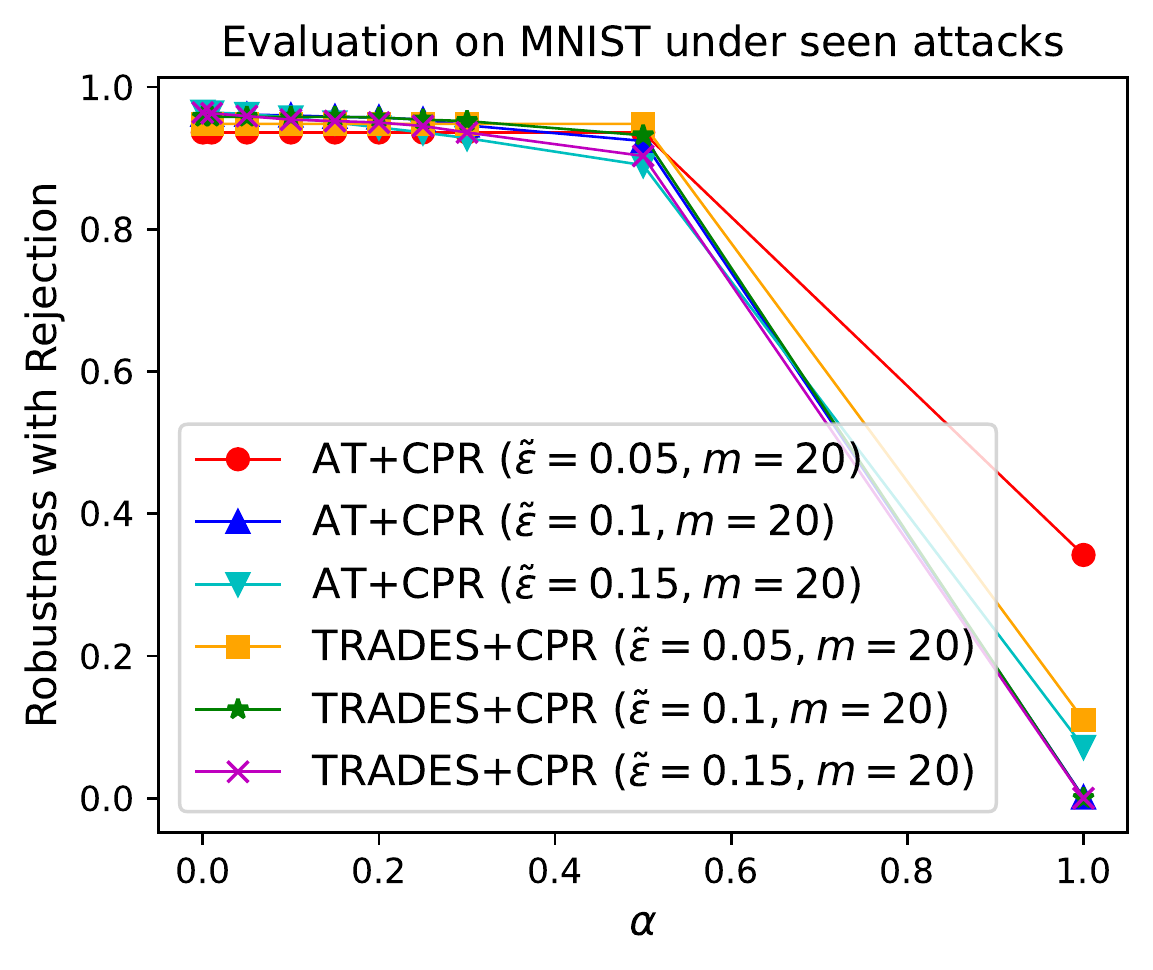}
	\includegraphics[width=0.31\linewidth]{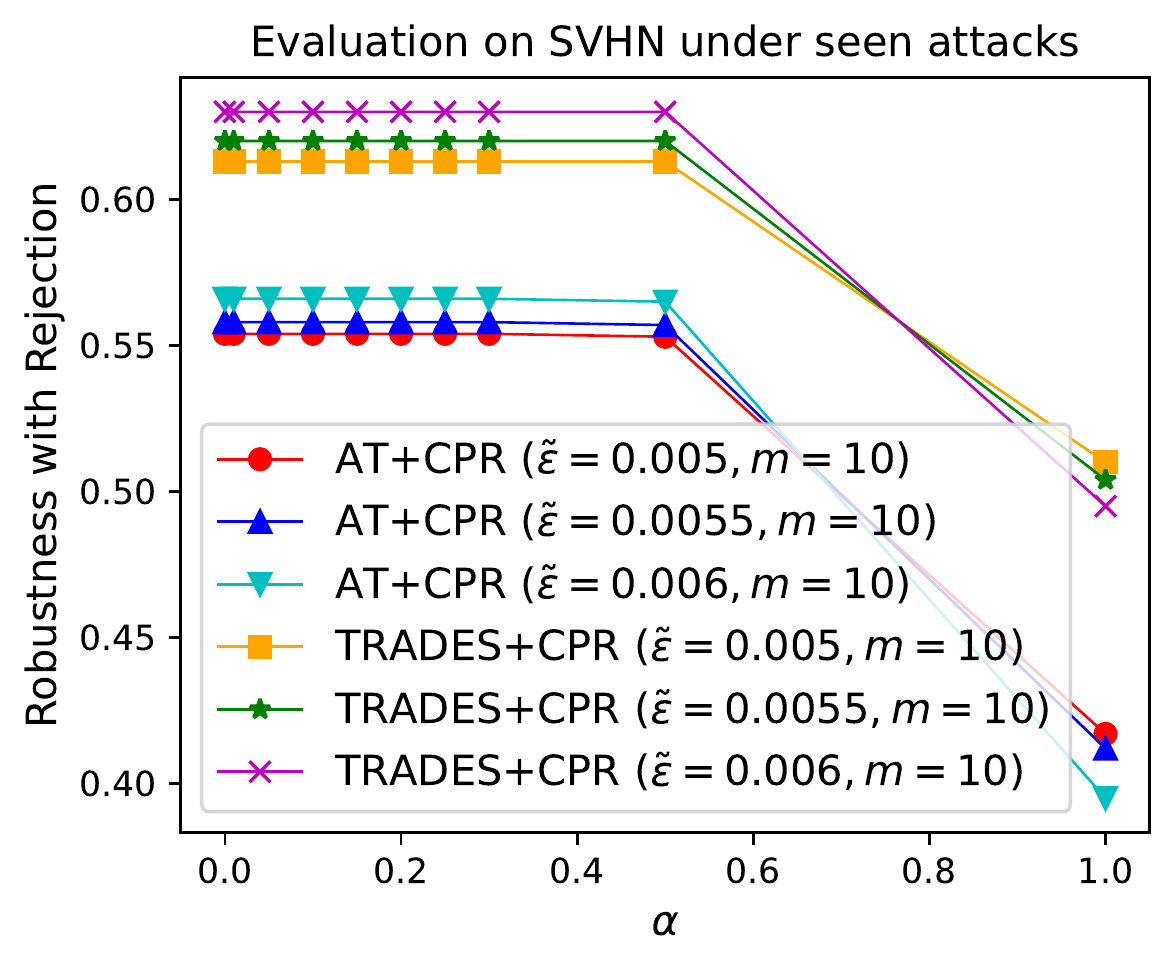}
	\includegraphics[width=0.31\linewidth]{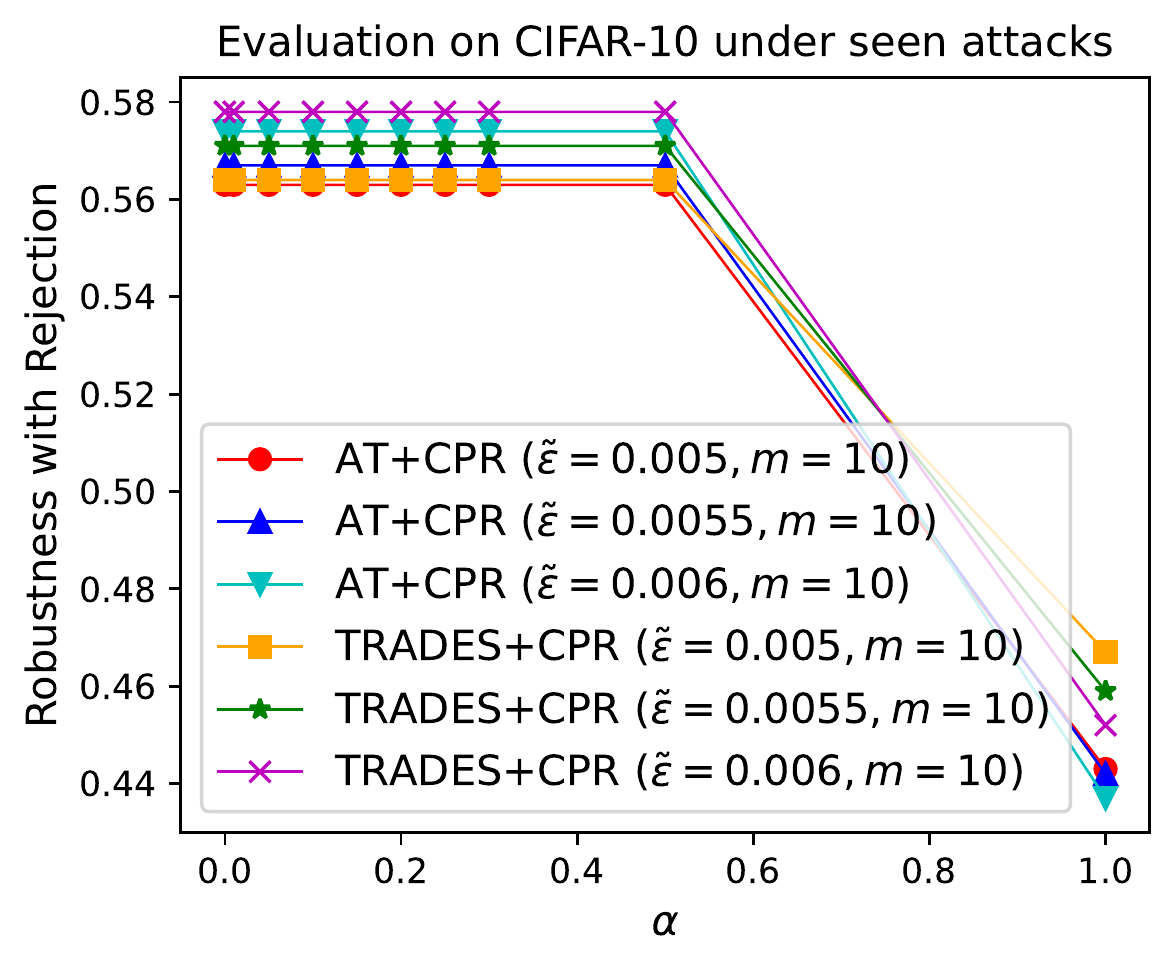}
	\includegraphics[width=0.31\linewidth]{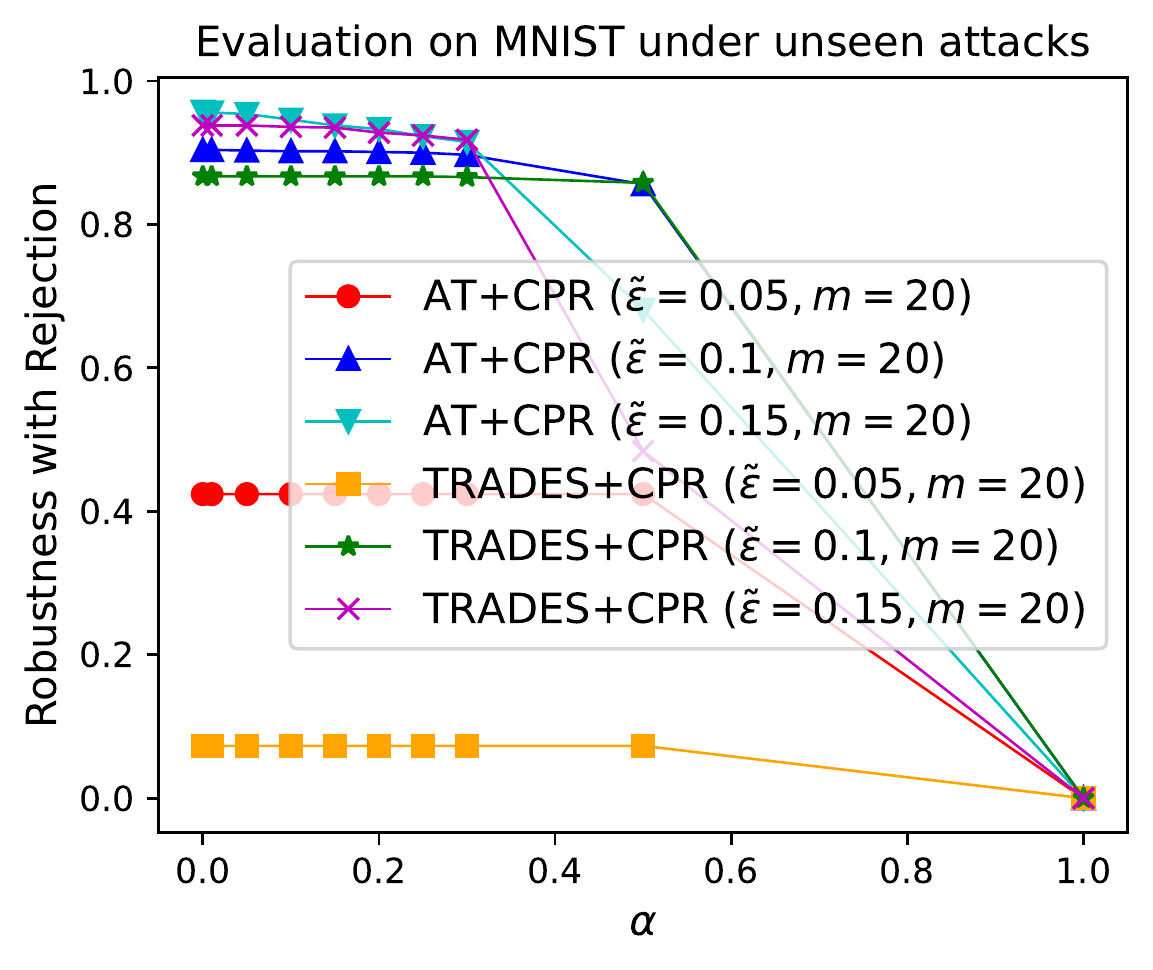}
	\includegraphics[width=0.31\linewidth]{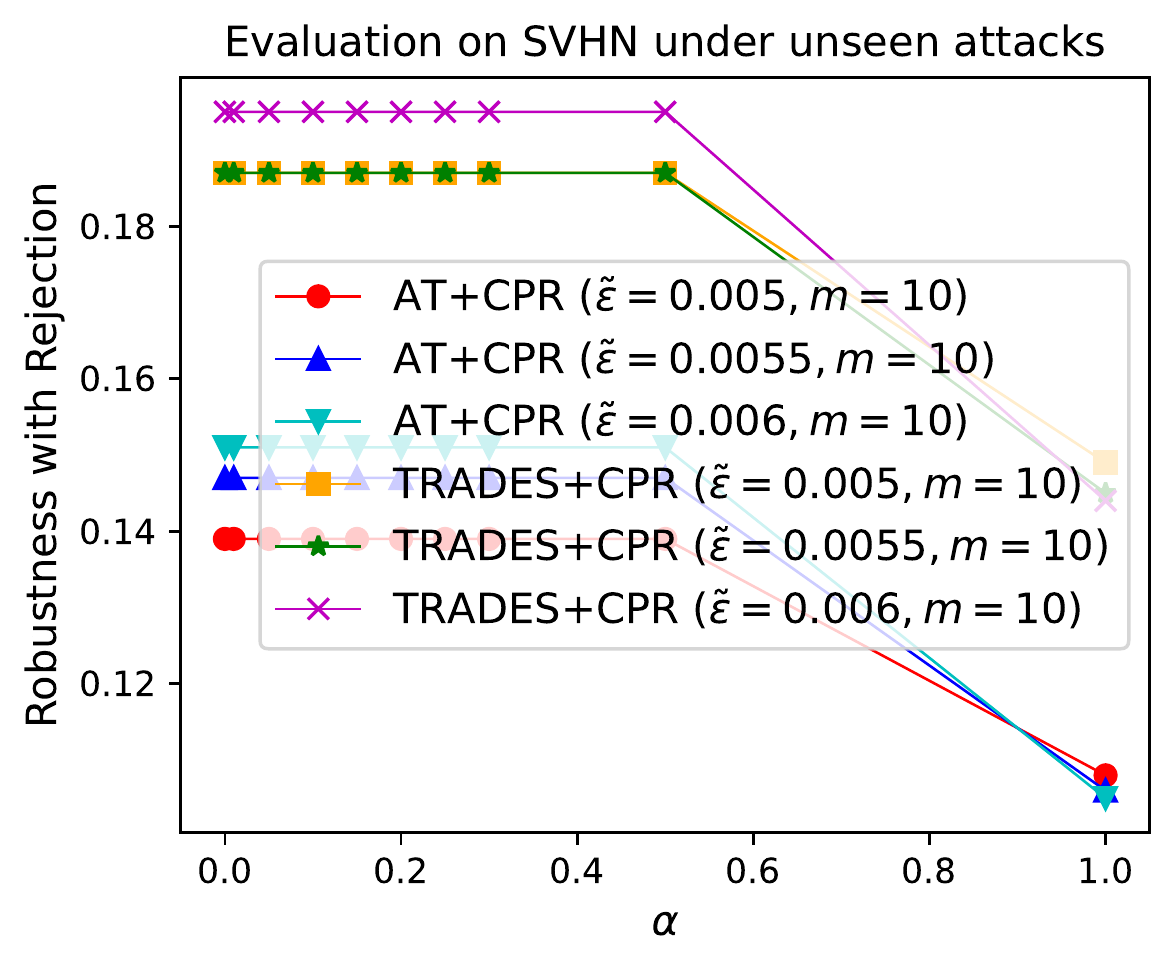}
	\includegraphics[width=0.31\linewidth]{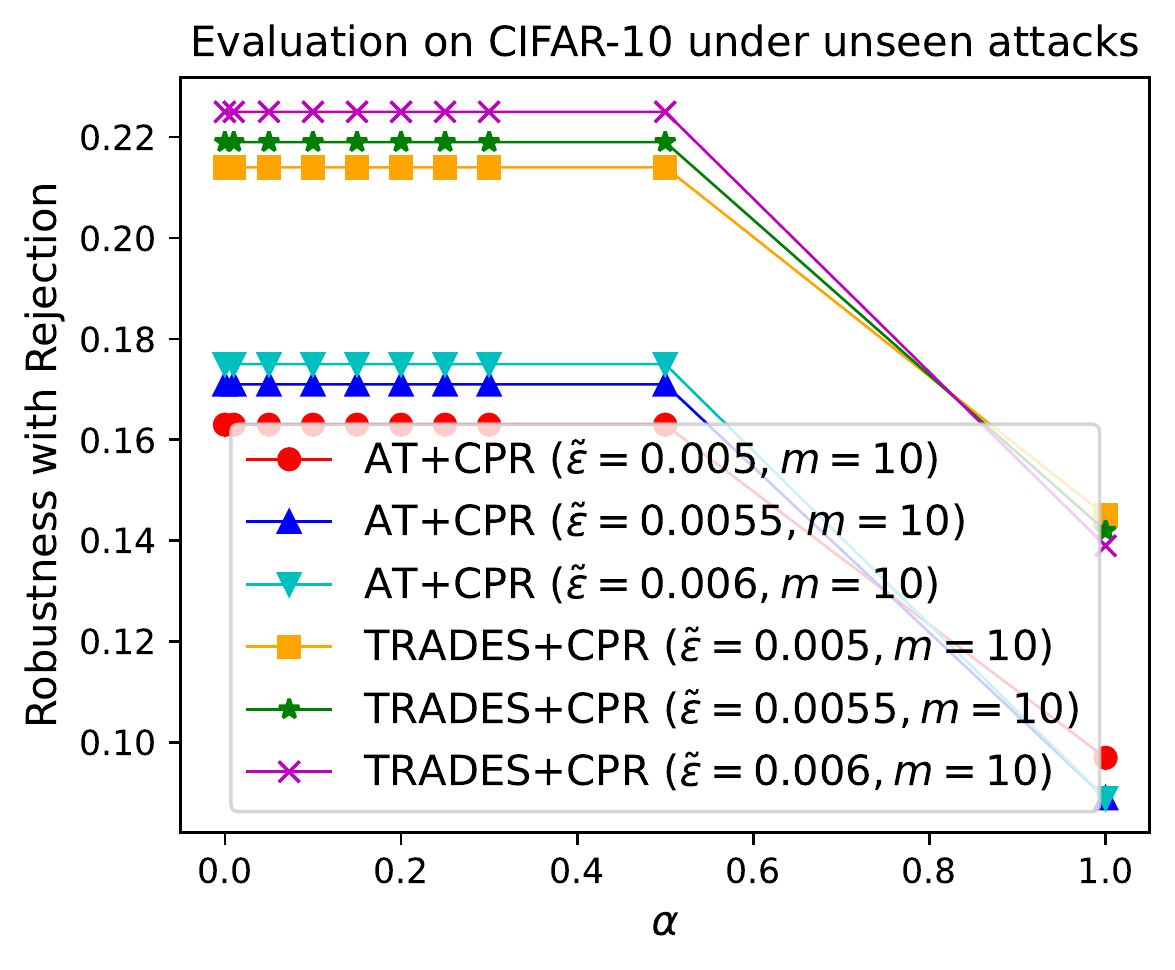}
	\caption{\small Ablation study for the proposed method CPR where we vary the hyper-parameter $\tilde{\epsilon}$ while fixing the hyper-parameter $m$. }
	\label{fig:ablation-vary-eps-results}
\end{figure*}

\begin{figure*}[htb]
	\centering
	\includegraphics[width=0.31\linewidth]{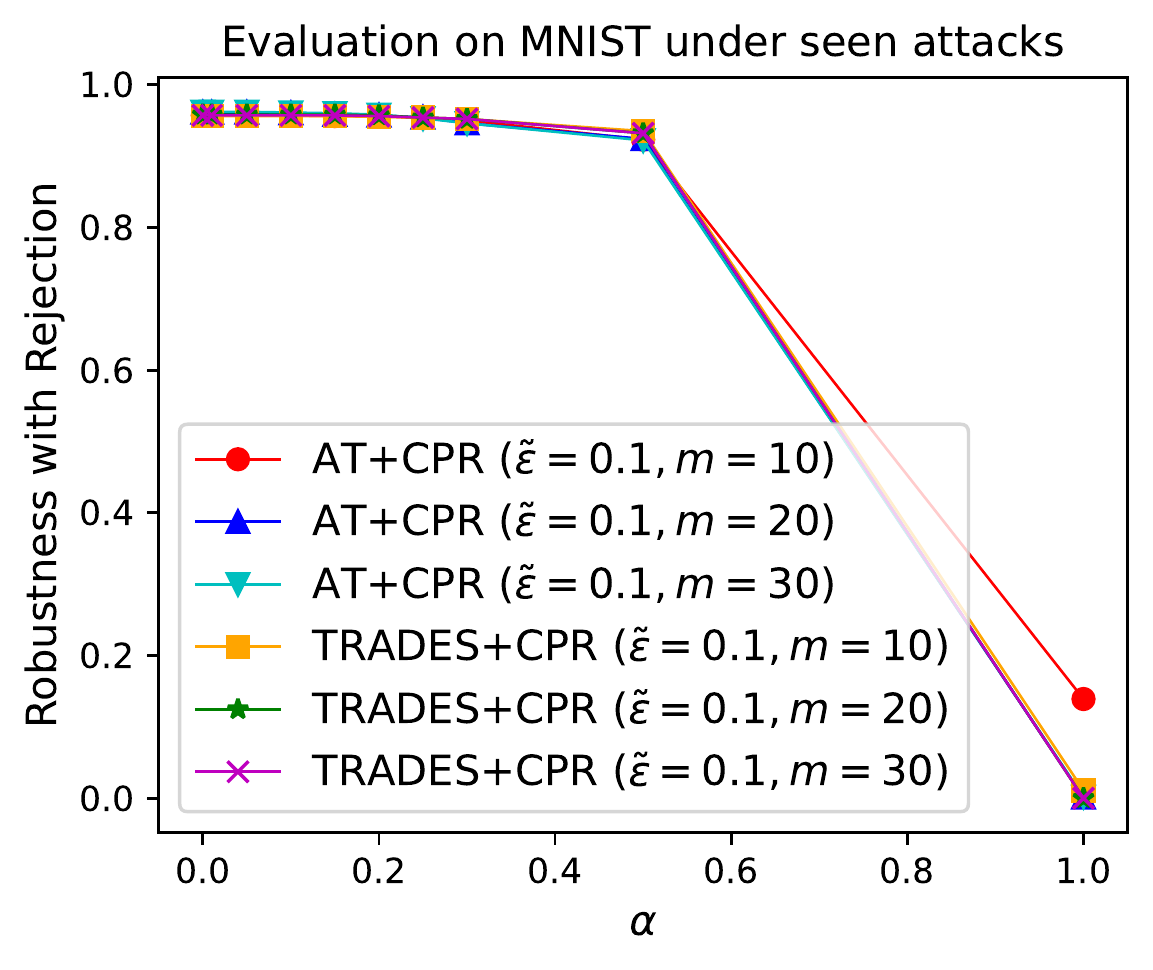}
	\includegraphics[width=0.31\linewidth]{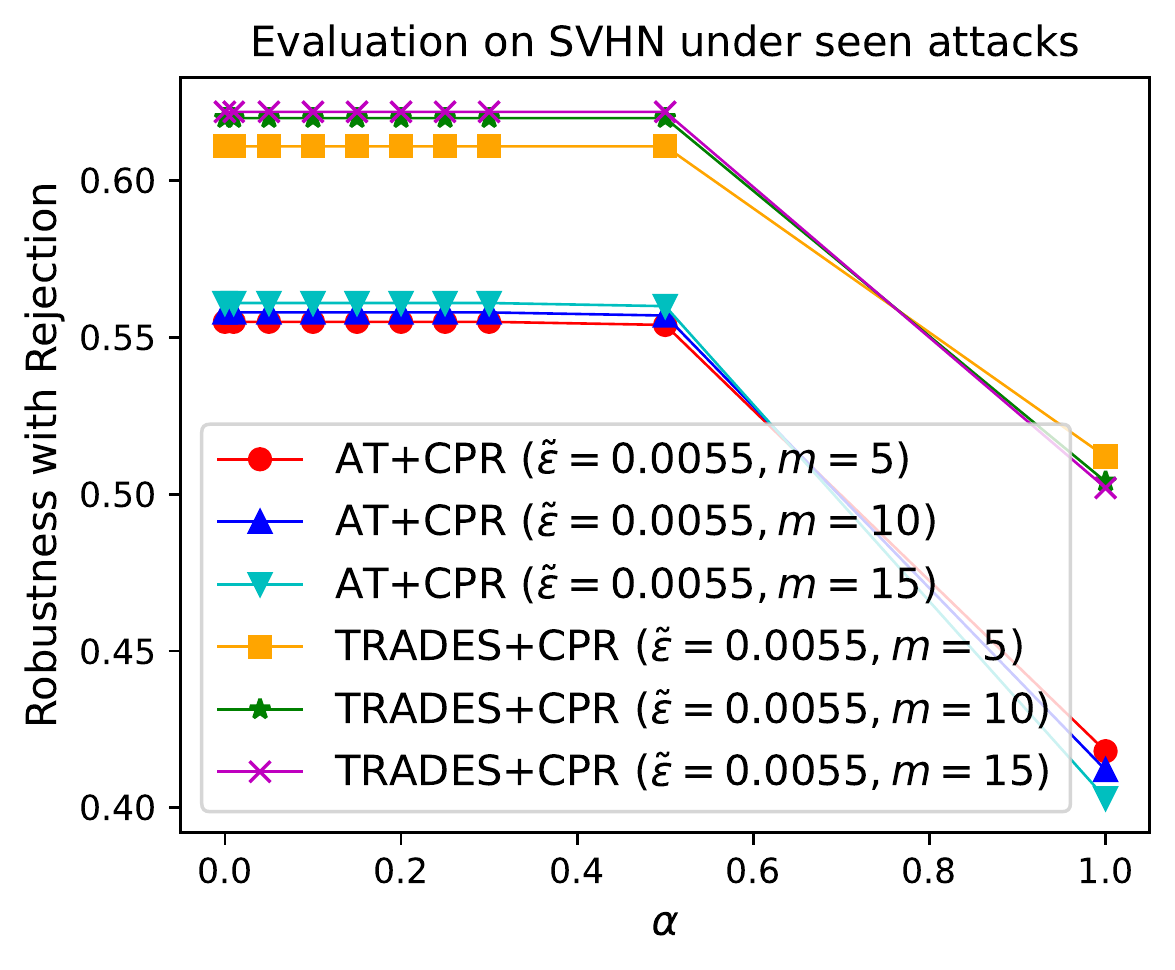}
	\includegraphics[width=0.31\linewidth]{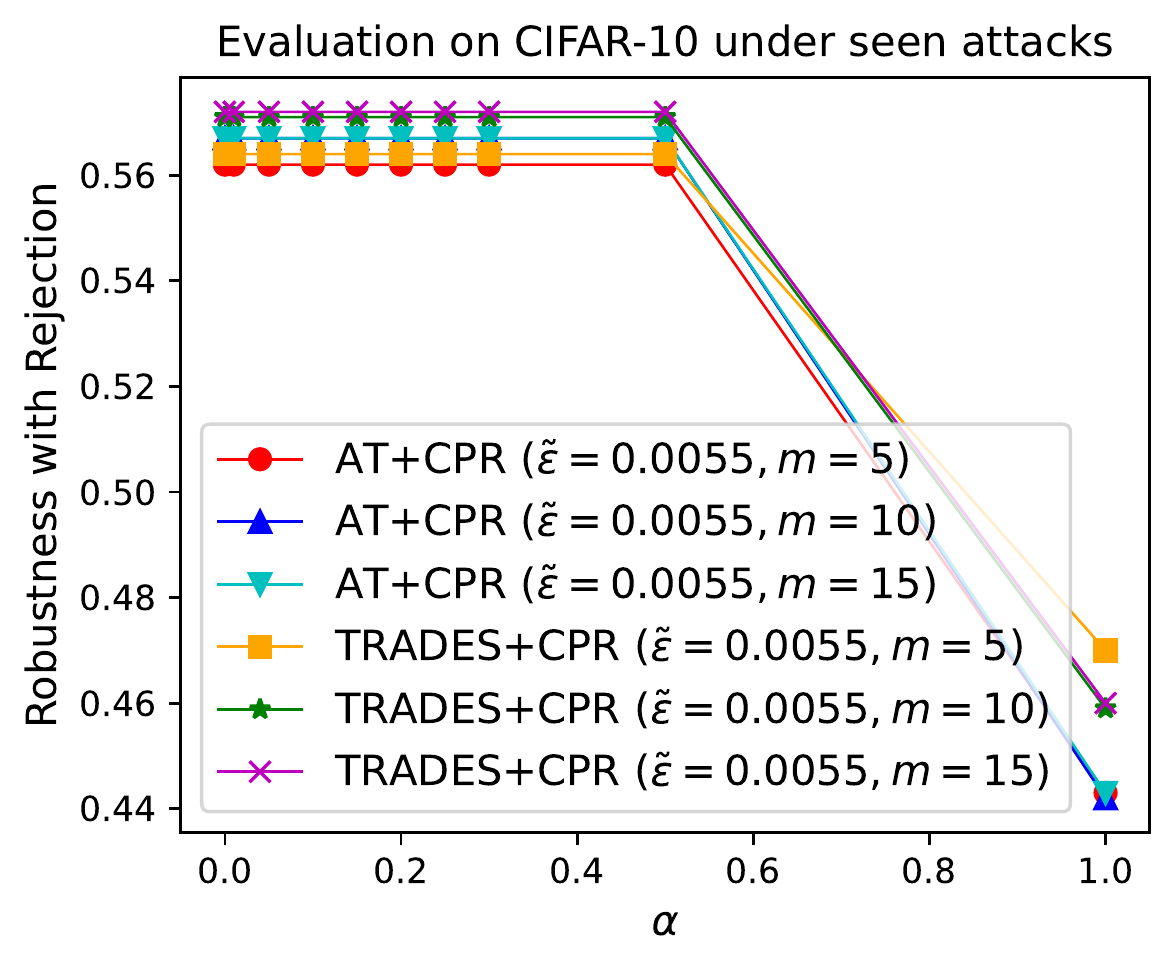}
	\includegraphics[width=0.31\linewidth]{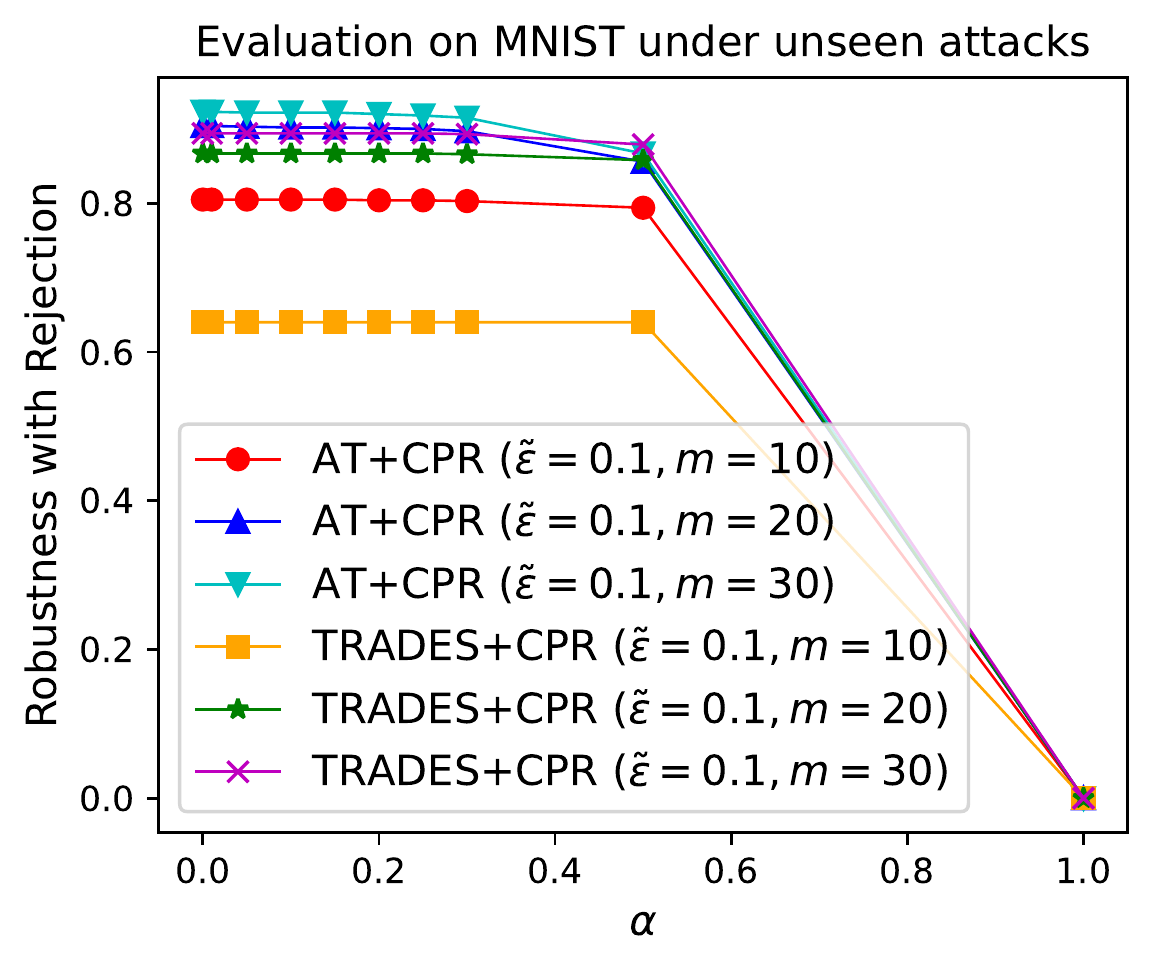}
	\includegraphics[width=0.31\linewidth]{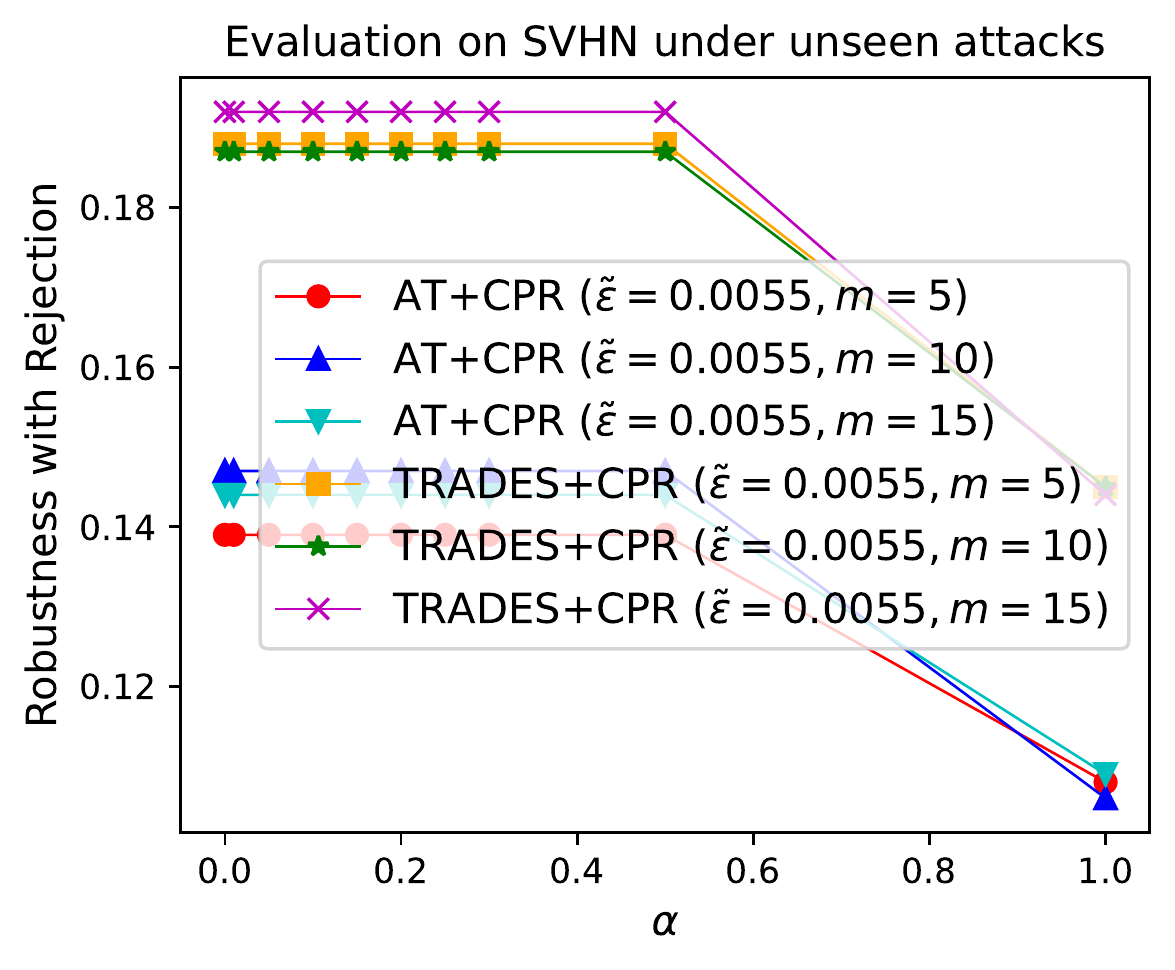}
	\includegraphics[width=0.31\linewidth]{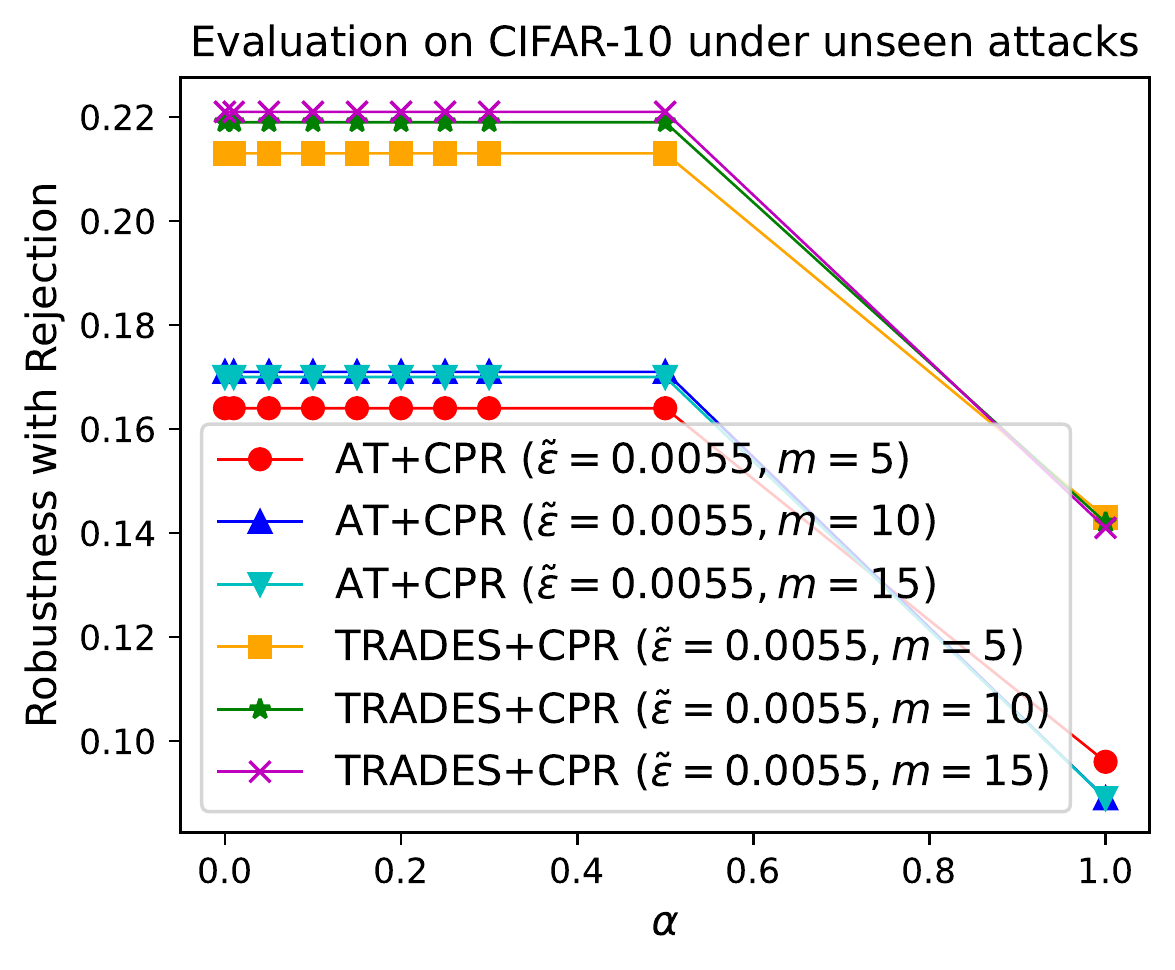}
	\caption{\small Ablation study for the proposed method CPR where we vary the hyper-parameter $m$ while fixing the hyper-parameter $\tilde{\epsilon}$.  }
	\label{fig:ablation-vary-m-results}
\end{figure*}

\subsection{Hyper-parameter Selection for CPR}
\label{sec:cpr-hyper-selection}

The proposed method CPR has three hyper-parameters: the perturbation budget $\tilde{\epsilon}$, the number of PGD steps $m$, and the PGD step size $\eta$. We don't tune $\eta$ but just set it to be a fixed value. From Appendix~\ref{sec:cpr-hyper-ablation}, we know that larger $m$ will lead to better robustness with rejection at $\alpha=0$. However, as we increase $m$, the improvement becomes minor. Thus, by considering the computational cost, we select a reasonably large $m$ based on the performance on the validation data. We also know that larger $\tilde{\epsilon}$ will lead to better robustness with rejection at $\alpha=0$, but will lead to a larger rejection rate on the clean test inputs. In practice, we select a large enough $\tilde{\epsilon}$ such that CPR achieves good robustness with rejection while having a reasonably low rejection rate on the clean inputs. We observed that a wide range of $m$ and $\tilde{\epsilon}$ can lead to good results. Thus, it is easy to select $m$ and $\tilde{\epsilon}$. In our experiments, we consider $m$ in the set $\{ 10, 20, 30 \}$ on MNIST, and $m$ in the set $\{ 5, 10, 15 \}$ on SVHN and CIFAR-10. We consider $\tilde{\epsilon}$ in the set $\{ 0.05, 0.1, 0.15, 0.2 \}$ on MNIST, and $\tilde{\epsilon}$ in the set $\{ 0.004, 0.005, 0.0055, 0.006 \}$ on SVHN and CIFAR-10. We select the best $m$ and $\tilde{\epsilon}$ based on the performance on the validation data.

\subsection{Attacking CCAT}
\label{sec:attacking-ccat}

In this section, we show that the High Confidence Misclassification Outer Attack (HCMOA) with attack step size enumeration proposed in Appendix~\ref{sec:app_adaptive_attacks} is stronger than the PGD attack with backtracking proposed in~\cite{stutz2020ccat} for evaluating robustness with detection of CCAT. We follow the setup in~\cite{stutz2020ccat} to train CCAT models and attack CCAT. For the PGD attack with backtracking, we use a base learning rate of 0.001, momentum factor of 0.9, learning rate factor of $1.1$, 1,000 iterations, and 10 random restarts.

\begin{table*}[htb]
    \centering
		\begin{tabular}{l|l|c|c}
			\toprule
			\multirow{2}{0.08\linewidth}{Dataset} & \multirow{2}{0.08\linewidth}{Attack} & \multicolumn{2}{c}{Robustness with Detection $\downarrow$} \\  \cline{3-4}
           &  & Seen Attacks & Unseen Attacks \\ \hline \hline
			\multirow{2}{0.12\linewidth}{MNIST}  
			& PGD with backtracking & 88.50 & 85.30 \\
            & HCMOA & \textbf{83.20} & \textbf{75.50} \\ \hline 
			\multirow{2}{0.12\linewidth}{{SVHN}}
			& PGD with backtracking  & 64.10 & 62.30 \\
            & HCMOA  & \textbf{45.30} & \textbf{9.10} \\ \hline 
			\multirow{2}{0.12\linewidth}{CIFAR-10}
			& PGD with backtracking  & 43.70 & 38.20 \\
            & HCMOA  & \textbf{27.70} & \textbf{8.80} \\ 
			\bottomrule
		\end{tabular}
	\caption[]{\small Attacking CCAT using different attacks. All numbers are percentages. \textbf{Bold} numbers are superior results. }
    \label{tab:attacking-ccat}
\end{table*}


\end{document}